\documentclass{article} 
\PassOptionsToPackage{dvipsnames,table}{xcolor} 
\usepackage{iclr2027_conference,times}

\usepackage{amsmath,amsfonts,bm}

\def\eqref#1{equation~\ref{#1}}
\def\Eqref#1{Equation~\ref{#1}}

\def\1{\bm{1}}

\newcommand{\test}{\mathcal{D_{\mathrm{test}}}}

\def\vone{{\bm{1}}}
\def\vmu{{\bm{\mu}}}

\def\va{{\bm{a}}}

\def\vq{{\bm{q}}}

\def\vv{{\bm{v}}}
\def\vw{{\bm{w}}}
\def\vx{{\bm{x}}}
\def\vy{{\bm{y}}}

\def\mA{{\bm{A}}}
\def\mB{{\bm{B}}}
\def\mC{{\bm{C}}}

\def\mE{{\bm{E}}}

\def\mI{{\bm{I}}}

\def\mM{{\bm{M}}}

\def\mP{{\bm{P}}}
\def\mQ{{\bm{Q}}}

\def\mU{{\bm{U}}}
\def\mV{{\bm{V}}}
\def\mW{{\bm{W}}}
\def\mX{{\bm{X}}}

\def\mZ{{\bm{Z}}}

\def\mSigma{{\bm{\Sigma}}}

\DeclareMathAlphabet{\mathsfit}{\encodingdefault}{\sfdefault}{m}{sl}
\SetMathAlphabet{\mathsfit}{bold}{\encodingdefault}{\sfdefault}{bx}{n}

\newcommand{\E}{\mathbb{E}}

\newcommand{\R}{\mathbb{R}}

\newcommand{\KL}{D_{\mathrm{KL}}}

\newcommand{\setF}{\mathcal{F}}      
\newcommand{\setR}{\mathcal{R}}      
\newcommand{\setS}{\mathcal{S}}      
\newcommand{\setY}{\mathcal{Y}}      

\newcommand{\vtau}{{\bm{\tau}}}                    
\newcommand{\taum}{\vtau_{\mathrm{m}}}             
\newcommand{\tauF}{\vtau_{\setF}}                  
\newcommand{\tauR}{\vtau_{\setR}}                  
\newcommand{\tauRstar}{\vtau_{\setR}^{\star}}      

\newcommand{\Winit}{\mW_{\mathrm{init}}}
\newcommand{\Wzero}{\mW_{0}}
\newcommand{\Wpt}{\mW_{\mathrm{PT}}}
\newcommand{\Wft}{\mW_{\mathrm{FT}}}
\newcommand{\Wrt}{\mW_{\mathrm{RT}}}
\newcommand{\Wul}{\mW_{\mathrm{UL}}}

\newcommand{\QF}{\mQ_{\setF}}
\newcommand{\QR}{\mQ_{\setR}}
\newcommand{\PF}{\mP_{\setF}}
\newcommand{\PR}{\mP_{\setR}}

\AtBeginDocument{\definecolor{bestgreen}{HTML}{DFF5E1}}
\newcommand{\best}[1]{\cellcolor{bestgreen}\textbf{#1}}                 
\newcommand{\bestlegend}[1]{{\setlength{\fboxsep}{1pt}\colorbox{bestgreen}{\textbf{#1}}}}  

\usepackage{hyperref}
\usepackage{url}
\usepackage{colortbl}
\usepackage{adjustbox}
\usepackage{titletoc}
\usepackage[dvipsnames,table]{xcolor}
\usepackage{algorithm}
\usepackage{algorithmic}
\usepackage[utf8]{inputenc} 
\usepackage[T1]{fontenc}    
\usepackage{hyperref}       
\usepackage{url}            
\usepackage{booktabs}       
\usepackage{nicefrac}       
\usepackage{microtype}      
\usepackage{xcolor}         
\usepackage{graphicx}
\usepackage{subcaption}
\usepackage{amssymb}
\usepackage{mathtools}
\usepackage{amsthm}
\newtheorem{assumption}{Assumption}[section]
\newtheorem{lemma}[assumption]{Lemma}
\newtheorem{proposition}[assumption]{Proposition}
\newtheorem{corollary}[assumption]{Corollary}
\newtheorem{remark}[assumption]{Remark}
\usepackage{enumitem}

\title{\textsc{Unmerge}: Efficient Machine Unlearning via Task Arithmetic}

\author{Haoran Tang \quad Andrew Tan \quad Rajiv Khanna \\
  Department of Computer Science \\
  Purdue University \\
  \texttt{\{thr,\,tan434,\,rajivak\}@purdue.edu}
}

\iclrfinalcopy 
\begin{document}

\maketitle
\lhead{Preprint.}
\begin{abstract}
    Approximate machine unlearning seeks to remove the influence of a forget
    set from a trained model without full retraining. Existing gradient-based
    methods require data-dependent hyperparameter search, struggle when forget
    and retain knowledge are entangled, and offer little insight into where
    unlearning actually happens inside the network. We re-cast unlearning
    through the lens of task arithmetic: if finetuning produces a \emph{merged}
    task vector $\boldsymbol{\tau}_{\text{m}}$ that combines learning on forget
    and retain sets, unlearning is the inverse operation that subtracts a
    learned forget component $\boldsymbol{\tau}_{\mathcal{F}}$ to recover the
    retain task vector $\boldsymbol{\tau}_{\mathcal{R}}$.
    The forget signal is concentrated: at every layer, forget activations lie
    in a subspace spanned by only a handful of dominant directions. We exploit
    this with a low-rank factorization $\boldsymbol{\tau}_{\mathcal{F}} =
    \mathbf{B} Q_{\mathcal{F}}^\top$, which is faithful up to a small
    tail-eigenvalue residual and bounds how far the correction can perturb
    retain. We then optimize three intuitive goals---match the merged vector
    inside the forget span, suppress leakage into the retain span, and bound
    the correction size---that provably bound forget leakage and retain damage in
    activation space.
    The resulting algorithm \textsc{Unmerge} is fast and powerful: on class-level
    unlearning with ResNet-50 on CIFAR-100 and Tiny ImageNet, it improves Tug-of-War
    by up to $\sim$24\% over a baseline of comparable runtime and by up to $\sim$18\%
    over stronger baselines that run $\sim5\times$ slower, keeps membership-inference
    exposure at the level of retraining, and shrinks the feature-distribution gap to
    the retrained model, where relabeling methods leave forget features cleanly
    separable. Further studies show that \textsc{Unmerge} also applies to ViT-S/16 and scales
    to Llama-3.2-3B. The per-layer basis geometry that drives
    the algorithm also serves as a layerwise diagnostic for when and where
    unlearning becomes structurally hard.
\end{abstract}

\section{Introduction}
\label{sec:intro}
\vspace{-5pt}
Machine unlearning, which aims to remove the influence of a subset of training data from an already-trained model, has moved from an abstract ideal to an operational requirement. Regulatory frameworks such as the EU GDPR~\citep{voigt2017eu} codify a right to be forgotten, and emerging concerns around copyright, privacy, and safety in foundation models~\citep{li2024wmdp} demand mechanisms that delete the footprint of specific data without retraining from scratch. Exact unlearning by partition-and-retrain~\citep{bourtoule2021machine} is the gold standard but is computationally prohibitive for modern deep networks, motivating a large body of \emph{approximate} methods that edit the trained weights post-hoc. We focus on \emph{concept-level} requests---removing a class, a semantic group or a behavior, of which the forget set is a sample---rather than the deletion of arbitrary individual records, which is a distinct problem with its own reference and metrics (\S\ref{sec:instance}).

The dominant approximate methods can be grouped into four families. \emph{Gradient-ascent} methods~\citep{kurmanji2023towards} reverse-train on the forget set with a retain regularizer; \emph{relabeling} methods~\citep{graves2021amnesiac} replace forget labels with random or noise targets; \emph{teacher-student} methods~\citep{chundawat2023bad} distill an ``incompetent teacher'' into the forget region; and \emph{saliency- or Fisher-based} methods~\citep{golatkar2020eternal,fan2024salun,foster2024fast} dampen weights estimated to encode the forget signal. All four ultimately drive a forget loss to reshape the weights, and trade forget completeness against retain damage through hyperparameter tuning.

Despite rapid progress, this line of work shares three limitations: \textbf{(i) Fragility}: gradient-ascent and relabeling variants operate on a knife-edge, where one extra epoch can flip a model from under-forgetting to catastrophic collapse, forcing per-task tuning; \textbf{(ii) Entanglement blindness}: scalar forget/retain losses give the optimizer no language for reasoning about \emph{where} forget and retain knowledge coexist, so when the two sets share features (the common case) retain damage is an unavoidable byproduct; \textbf{(iii) Opacity}: methods are evaluated almost entirely through scalar accuracy and privacy numbers and expose nothing about \emph{why} a layer or a direction was (un)successfully unlearned, leaving diagnosis, comparison, and trust on shaky ground.

We aim to address all three by revisiting unlearning from a new angle: task arithmetic~\citep{ilharco2023editing}. If finetuning on $\setS=\setR\cup\setF$ for retain set $\setR$ and forget set $\setF$ is read as implicit multi-task learning that produces a \emph{merged} task vector $\taum$, then unlearning $\setF$ becomes the inverse problem: identify the forget component $\tauF$ inside $\taum$ and subtract it. We call this \emph{unmerge}, and view it as the natural counterpart to recent data-free model merging~\citep{cheng2025whoever}: the activation-subspace geometry that governs cross-task interference when combining experts also governs forget-retain leakage when separating them, with the optimization direction reversed. This single conceptual shift addresses each of the three limitations above. \emph{Fragility} is replaced with a convex quadratic per layer, solved in a $k$-dimensional projected space without any gradient on data, so there is no ascent-descent collapse to guard against, and the one remaining scale $\alpha$ can be usually set from forget-set accuracy alone (App.~\ref{sec:ablation}). \emph{Entanglement} becomes a first-class object: forget and retain subspaces are explicitly constructed and their overlap is measurable per layer and per direction. \emph{Opacity} dissolves because the algorithm's internal state---bases, overlaps, per-layer task-vector magnitudes---is directly inspectable.

The implications go beyond a single algorithm: unmerging recasts unlearning from a tuning problem on top of opaque weight edits into a structured decomposition problem with a principled subspace vocabulary. That vocabulary \emph{predicts} unlearning difficulty before any update is applied, \emph{localizes} where in a network the forget signal lives (late layers are the high-leverage region), and delimits the method's scope: a forget request with no separable activation subspace---random instance subsets being the extreme case---cannot be unmerged, and \textsc{Unmerge} flags this before editing (\S\ref{sec:instance}). The diagnostics are method-agnostic, the formulation extends to concept removal and modular unlearning in larger models, and the gradient-free edit is amortizable across forget requests on the same checkpoint. Empirically, our instantiation improves the retain/forget tradeoff over gradient-based and task-vector baselines on ResNet-50 (CIFAR-100, TinyImageNet) and ViT-S (CIFAR-100), avoids the collapse mode of gradient-ascent unlearning, and matches retraining closely under membership-inference attacks on the forget set where relabeling methods leave a larger signature on the forget set. We summarize our \textbf{contributions} as follows:
\begin{itemize}[leftmargin=*, itemsep=-0.3em, topsep=-0.15em]
  \item \textbf{A new formulation.} We recast approximate unlearning as the inverse of model merging and introduce \textsc{Unmerge}: learning a per-layer forget task vector and subtracting it from $\taum$ so that the remainder behaves as a retain-only task vector in activation space (\S\ref{sec:method}).
  \item \textbf{Theoretical motivation.} Building on the activation-subspace assumption of~\citet{cheng2025whoever}, we give an activation-space analysis that connects each term of our objective to an interpretable quantity: forget leakage, retain damage, and tail-direction over-expansion, positioning unmerging as the counterpart of their merging framework and inviting more rigorous theoretical work (\S\ref{sec:theory}).
  \item \textbf{Gradient-free unlearning.} We instantiate a low-rank procedure---a convex per-layer objective solved in a $k$-dimensional projected space---whose only data access is forward passes for activation statistics. It improves the retain/forget tradeoff over gradient-based and task-vector baselines on ResNet-50 (CIFAR-100, TinyImageNet) and ViT-S (CIFAR-100), avoids the collapse failure mode of gradient-ascent methods, and recovers a model closer to retraining in output space (\S\ref{sec:exp}). Ablations show that the retain-leakage term carries this gain, and a preliminary study carries the method over to a 3B-parameter language model (App.~\ref{sec:llm}).
  \item \textbf{Interpretability tools.} We introduce per-layer, per-direction diagnostics---task-vector magnitudes, forget-retain overlap, and their depth-wise profiles---that make unlearning diagnosable layer by layer. A separability score computed \emph{before} unlearning predicts which forget requests lie beyond the reach of subspace editing; we characterize that boundary with instance-level experiments (\S\ref{sec:instance}).
\end{itemize}

\section{Related Work}
\label{sec:related}
\vspace{-5pt}
\textbf{Task arithmetic and model merging.} Task vectors~\citep{ilharco2023editing} are weight-space deltas induced by finetuning that support simple arithmetic for composing or suppressing behaviors without retraining. Subsequent \emph{merging} work refines this view: weight averaging across finetuned experts~\citep{wortsman2022model,matena2022merging}, conflict-aware pruning before averaging~\citep{yadav2023ties,yu2024language}, closed-form layer-activation objectives~\citep{jin2023dataless}, and tangent-space analyses of when task-vector linearity holds~\citep{ortiz2023task}. Most closely, \citet{cheng2025whoever} show that task-vector interference is governed by the span of layer inputs and derive a data-free merging objective. Our work builds directly on this view but \emph{reverses} it: instead of combining task vectors while minimizing cross-interference, we \emph{unmerge}---decompose an already-merged vector by learning the sub-component that lives in a forget subspace and subtracting it. The handful of papers that compose task vectors for removal (e.g., subtracting a toxicity vector~\citep{ilharco2023editing,ilharco2022patching,zhang2023composing}) assume an explicit forget task vector is available, whereas we \emph{construct and learn} one from data-free bases on the finetuned model itself. Recent \emph{task-vector unlearning} obtains that vector by finetuning on the forget set and negating it, merging several such vectors under a sign-consensus mask~\citep{kim2025negmerge}, or constraining finetuning so that merged models admit exact unlearning at scale~\citep{kuo2026exact}. \S\ref{sec:main} compares against the negated-vector and NegMerge baselines.

\textbf{Machine unlearning.} Motivated by right-to-be-forgotten regulations~\citep{voigt2017eu} and the desire to remove poisoned or copyrighted data, machine unlearning~\citep{cao2015towards,nguyen2022survey,wang2026survey} aims to produce a model indistinguishable from one retrained on $\setR=\setS\setminus\setF$. Exact methods partition training so forgetting a point invalidates one shard~\citep{bourtoule2021machine} but require training-time instrumentation. Approximate methods operate post-hoc on $\Wft$: influence- or Fisher-based scrubbing~\citep{golatkar2020eternal,golatkar2021mixed}, gradient-ascent variants such as NegGrad+ and SCRUB~\citep{kurmanji2023towards}, saliency-masked updates~\citep{fan2024salun}, synaptic dampening~\citep{foster2024fast}, boundary shifting~\citep{chen2023boundary}, amnesiac relabeling~\citep{graves2021amnesiac,tarun2023fast}, and distillation-based erasure that separates forgetting from retention through a teacher~\citep{chundawat2023bad,zhou2025delete} or adds a masked-distillation penalty against over-unlearning~\citep{ha2026blindspots}. Compared to these, \textsc{Unmerge} is \emph{data-free at optimization} (data enters only when collecting second-moment statistics), \emph{layerwise and interpretable}, and motivated by a merging-theoretic objective with activation-space bounds (Prop.~\ref{prop:bounds}); it performs no training on $\setF$ or $\setR$ at unlearning time (App.~\ref{sec:requirements}).

\textbf{Subspace projection and representation editing.} A separate thread exploits low-dimensional structure to intervene on behaviors: continual-learning subspace methods project gradients orthogonal to previous tasks~\citep{saha2021gradient,farajtabar2020orthogonal}, while concept-erasure and knowledge-editing methods intervene on activation directions or low-rank weight subspaces~\citep{kim2018interpretability,belrose2023leace,ravfogel2022linear,meng2022locating}. Closest to our edit are gradient-free class-unlearning methods that project weights away from forget activation subspaces~\citep{kodge2024deep,chen2024unsc}; with $\Winit{=}0$ our update is a soft member of this family, from which it differs by acting on the task vector, by solving an explicit forgetting-vs-leakage objective instead of a hard projection, and by its supervised bases and diagnostics (App.~\ref{sec:related-supp}). We adopt Concept Activation Vectors~\citep{kim2018interpretability} as one of our basis constructions but lift them from activation-level probing to weight-level subtraction, and unlike hard orthogonal projection we keep the natural overlap between $\QF$ and $\QR$ and softly penalize retain leakage via Term~2 of \eqref{eqn:obj} (see App.~\ref{sec:related-supp} for method-specific comparisons).
\section{Unlearning via Unmerging}
\label{sec:method}
\vspace{-5pt}
\subsection{Preliminaries}
\vspace{-5pt}
Let $\Winit\in\{\Wzero,\Wpt\}$ be the starting weights: a public pretrained checkpoint $\Wpt$, or training from scratch ($\Wzero$), for which we take $\Winit=0$ rather than the random draw so that $\taum=\Wft$ (ablated in App.~\ref{sec:winit}). Finetuning on a train set $\setS$ gives $\Wft$. The forget set $\setF\subset\setS$ is revealed after finetuning, and $\setR=\setS\setminus\setF$ is the retain set with $|\setR|>|\setF|$. An unlearning algorithm maps $\Wft$ to $\Wul$ using $\setF$ and $\setR$; the ideal target is the model $\Wrt$ retrained from $\Winit$ on $\setR$ only.

Following Task Arithmetic~\cite{ilharco2023editing}, a task vector is the weight difference induced by training, and we define
\begin{equation}
    \taum = \Wft - \Winit, \qquad \tauRstar = \Wrt - \Winit, \qquad \tauR \coloneqq \taum - \tauF,
\end{equation}
where $\taum$ is the merged task vector (since $\setS=\setR\cup\setF$) and $\tauRstar$ the ideal retrained one. We learn a parameterized forget task vector $\tauF$ such that the remainder $\tauR$ behaves like a retain-only task vector in activation space; \S\ref{sec:tauval} measures how close $\tauR$ lands to $\tauRstar$.

\textbf{Why $\Wft$ can be treated as merged.} Joint finetuning on $\setR\cup\setF$ is not additive, and we never assume $\taum=\tauF+\tauR$ for separately trained experts: $\tauR$ is \emph{defined} as the remainder, and the objective selects the $\tauF$ that makes it approximately orthogonal to the forget span. Assumption~\ref{ass:input-subspace} makes this meaningful for a jointly trained vector, since the rows of $\taum$ lie in the span of the joint input activations. ``Unmerging'' thus names the \emph{objective} rather than the training provenance, and unlike task arithmetic it needs no separately trained forget expert.

\textbf{Parameterizing forget task vector.} \citet{cheng2025whoever} show that the task vector at layer $l$ of a multi-layer neural network is spanned by $l$-th layer's inputs. This motivates us to construct a parameterized forget task vector $\tauF^l\in\R^{m\times d_l}$ at layer $l$ based on forget inputs $\mX_{\setF}^l\in\R^{N_{\setF}\times d_l}$ at that layer. Denote $\tauF^l=\mB^l{\QF^l}^{\top}$ with a learnable matrix $\mB^l\in\R^{m\times k}$ and a forget basis matrix $\QF^l\in\R^{d_l\times k}$. We can construct $\QF^l$ either by singular value decomposition of $\mX_{\setF}^l$ (unsupervised, high-energy basis) or by ridge-probe concept-activation vectors~\cite{kim2018interpretability} (supervised, discriminative basis). The forget projector is $\PF^l=\QF^l{\QF^l}^{\top}$; the retain basis $\QR^l\in\R^{d_l\times k}$ and projector $\PR^l$ are built likewise from the retain inputs $\mX_{\setR}^l\in\R^{N_{\setR}\times d_l}$.

\subsection{Unmerging}
\vspace{-5pt}
The objective decouples across layers. For each layer we learn $\tauF^l$ such that the remainder $\tauR^l=\taum^l-\tauF^l$ is approximately orthogonal to the forget span and $\tauF^l$ approximately orthogonal to the retain span, and we regularize the size of $\tauF^l$ against over-expansion. The per-layer \textsc{Unmerge} objective, and the update that applies the learned $\tauF^l=\mB^l{\QF^l}^{\top}$ with a scaling $\alpha>0$, are
\begin{gather}
\label{eqn:obj}
    \min_{\tauF^l}\,||\left(\taum^l-\tauF^l\right)\PF^l||^2_F+\lambda||\tauF^l \PR^l||^2_F+\gamma||\tauF^l||^2_F,\\
\label{eq:update}
    \Wul^l=\Wft^l-\alpha\,\tauF^l.
\end{gather}
With orthonormal bases, \eqref{eqn:obj} can be reduced to a strictly convex quadratic in $\mB^l$ alone, $\|\taum^l\QF^l-\mB^l\|_F^2+\lambda\|\mB^l\mC^l\|_F^2+\gamma\|\mB^l\|_F^2$ with $\mC^l={\QF^l}^{\top}\QR^l\in\R^{k\times k}$: $mk$ parameters per layer, no data access, and a retain term that acts only along the directions $\QF^l$ shares with $\QR^l$.

\textbf{Why low rank suffices.} Replacing $\PF^l$ and $\PR^l$ with the empirical activation second-moments $\frac{1}{N_{\setF}}{\mX_{\setF}^l}^\top\mX_{\setF}^l$ and $\frac{1}{N_{\setR}}{\mX_{\setR}^l}^\top\mX_{\setR}^l$ recovers the data-space form of the same objective, at $md_l$ instead of $mk$ parameters per layer and with the activations revisited at every step. The low-rank form is both tighter and safer. By Assumption~\ref{ass:input-subspace}, the rows of $\tauF^l$ lie approximately in the span of the forget activations, so the top-$k$ eigenbasis is faithful up to a tail residual bounded by Lemma~\ref{lem:lowrank}, and the forget eigenspectrum decays sharply in the layers we consider. The rank-$k$ form also caps how far $\tauF^l$ can perturb directions outside the forget subspace, which disproportionately carry retain information: an overly large $k$ hurts retain accuracy at the same $\alpha$ (App.~\ref{sec:ablation}).

\textbf{Implementation.} \textsc{Unmerge} runs in three phases---(1) one forward pass over $\setF\cup\setR$ builds $\QF^l$ and $\QR^l$; (2) $\mB^l$ is initialized at the projection $\mB^l_0=\taum^l\QF^l$ and moved toward the minimizer of the projected quadratic under a fixed budget of $T{=}100$ Adam steps ($<\!1$s); (3) \eqref{eq:update} is applied, optionally only to deep layers (\texttt{skip-layers}, \S\ref{sec:perlayer})---and is \emph{optimization-time data-free}: no gradient is ever taken on a training sample (pseudocode and details in App.~\ref{sec:algo-supp} and~\ref{sec:requirements}).

\textbf{Solver and step budget.} The projected quadratic has the unique minimizer
\begin{equation}
\label{eq:closed}
    {\mB^l}^{\star}=\taum^l\QF^l\bigl((1+\gamma)\mI+\lambda\,\mC^l{\mC^l}^{\top}\bigr)^{-1},
\end{equation}
a single $k\times k$ solve per layer (App.~\ref{sec:proof}). We do not run Phase~2 to convergence: the fixed budget is part of the estimator, and because Adam's normalized steps move each entry by at most $\delta=\mathrm{lr}\times T$, this motivates the approximation $\mB^l_\delta={\mB^l}^{\star}+S_\delta(\mB^l_0-{\mB^l}^{\star})$ with $S_\delta$ the entrywise soft-threshold: every coefficient reaches its regularized optimum except the few largest projection coefficients, which are shrunk only by $\delta$. Early stopping thus acts as an implicit regularizer~\citep{yao2007early}; for the low-overlap CAV bases the iterate coincides with the minimizer, for PCA the residual carries part of the edit (see App.~\ref{sec:closedform} for comparison with actual iterate). Removing terms with $\alpha$ re-swept shows that Term~2 accounts for essentially all of the gain over the plain projection (App.~\ref{sec:terms}).

\subsection{Theoretical Motivation}
\label{sec:theory}
\vspace{-5pt}
We motivate the objective with activation-space bounds, adapting recent model-merging theory~\cite{cheng2025whoever} to unlearning; the analysis is per layer and explains what each term controls.

\begin{assumption}[Input Subspace Property~\cite{cheng2025whoever}]
\label{ass:input-subspace}
For a model finetuned with learning rate $\eta$ under Lipschitz continuity, the task vector $\vtau\in\R^{m\times d}$ at a given layer satisfies $\vtau = \mC \mX + \mE$ where $\mC\in\R^{m\times N}$ is a coefficient matrix, $\mX\in\R^{N\times d}$ collects the layer inputs during training, and $\|\mE\|_F = \mathcal{O}(\eta^2)$.
\end{assumption}

Every SGD update of a linear layer is an outer product with that layer's inputs, so the span property is exact up to the drift of those inputs during training; the $\mathcal{O}(\eta^2)$ rate is a finetuning idealization that we use heuristically for from-scratch training (App.~\ref{sec:proof}). It motivates our low-rank parameterization:

\begin{lemma}[Low-Rank Approximation]
\label{lem:lowrank}
Let $\mSigma_{\setF} = \frac{1}{N_\setF}\mX_{\setF}^\top\mX_{\setF}$ be the second-moment matrix of forget inputs with eigenvalues $\nu_1\geq\cdots\geq\nu_d\geq 0$, and let $\QF\in\R^{d\times k}$ be the top-$k$ eigenvectors with projector $\PF=\QF{\QF}^\top$. Under Assumption~\ref{ass:input-subspace}, any task vector $\vtau$ whose rows lie in $\mathrm{span}(\mX_{\setF}^\top)$ admits the decomposition $\vtau = \mB{\QF}^\top + \mE$ with $\mB=\vtau\QF\in\R^{m\times k}$ and
\begin{equation}
    \|\mE\|_F^2 \leq \|\mC\|_{\mathrm{op}}^2\cdot N_\setF\sum_{i=k+1}^{d}\nu_i.
\end{equation}
\end{lemma}
The approximation error of $\tauF=\mB{\QF}^\top$ thus decays with the tail eigenvalues of the forget covariance, so rank $k$ suffices when the forget signal is low-dimensional. The three terms of our objective, in turn, control the activation-space quantities that matter for unlearning.

\begin{proposition}[Activation-Space Bounds]
\label{prop:bounds}
Let $\mSigma_{\setF}\in\R^{d\times d}$ be the second-moment matrix of forget inputs with eigenvalues $\nu_1\geq\cdots\geq\nu_d$ and top-$k$ projector $\PF$. Similarly, let $\mSigma_{\setR}$ have eigenvalues $\mu_1\geq\cdots\geq\mu_d$ with top-$k$ projector $\PR$. Let $\alpha>0$ be the scaling used in the update and denote $\tauR=\taum-\alpha\tauF$. Then:
\begin{align}
    \E_{x\sim\setF}\!\left[\|\tauR\,x\|^2\right] \leq& \nu_1\left\|(\taum-\alpha\tauF)\PF\right\|_F^2 + \nu_{k+1}\left\|\taum-\alpha\tauF\right\|_F^2, \label{eq:forget-bound}\\
    \E_{x\sim\setR}\!\left[\|\alpha\tauF\,x\|^2\right] \leq& \alpha^2\Bigl(\mu_1\left\|\tauF\PR\right\|_F^2 + \mu_{k+1}\left\|\tauF\right\|_F^2\Bigr). 
    \label{eq:retain-bound}
\end{align}
\end{proposition}
Proofs are in App.~\ref{sec:proof}. Each term of the objective controls one component of these bounds:
\begin{itemize}
    \item \textbf{Term 1} $\|(\taum-\tauF)\PF\|_F^2$ controls the dominant ($\nu_1$-weighted) component of \emph{forget leakage}~\eqref{eq:forget-bound}---the response of the unlearned model on forget inputs.
    \item \textbf{Term 2} $\|\tauF\PR\|_F^2$ controls the dominant ($\mu_1$-weighted) component of \emph{retain damage}~\eqref{eq:retain-bound}---how much the forget correction disrupts retain performance.
    \item \textbf{Term 3} $\|\tauF\|_F^2$ controls the tail ($\mu_{k+1}$-weighted) component of retain damage, providing uniform protection across all directions beyond the top-$k$ retain subspace.
\end{itemize}
Both bounds hold for any $\tauF$, hence for the budgeted iterate we apply, and Terms~2 and~3 carry the factor $\alpha^2$, so $\alpha$ trades forgetting strength against retain damage (\S\ref{sec:perlayer}).

\begin{corollary}[The minimizer]
\label{cor:gain}
Let $\sigma_1\ge\dots\ge\sigma_k\in[0,1]$ be the singular values of $\mC=\QF^{\top}\QR$, the cosines of the principal angles between the forget and retain spans, with left singular vectors $\mV$. Then ${\mB}^{\star}=\taum\QF\,\mV\,\mathrm{diag}\bigl(g_i\bigr)\mV^{\top}$ with $g_i=1/(1+\gamma+\lambda\sigma_i^2)$: the minimizer keeps forget-private directions ($\sigma_i\!\approx\!0$) at gain $1/(1+\gamma)$ and shrinks directions shared with retain ($\sigma_i\!\approx\!1$) by up to $1/(1+\gamma+\lambda)$, and the leading term of \eqref{eq:retain-bound} shrinks by the factor $g_i^2$ along each direction.
\end{corollary}
Term~2 therefore acts exactly where the two spans meet, which is the mechanism behind the ablation of App.~\ref{sec:terms}. The selected $\alpha$ exceeds the unit gain at which the first term of \eqref{eq:forget-bound} is smallest, which we attribute to forget signal in its tail term; the scope of the bounds is discussed in App.~\ref{sec:proof}.

\begin{remark}[Relation to Model Merging]
\label{rem:duality}
Our unmerging objective mirrors the model-merging objective of~\citet{cheng2025whoever}. For two tasks $\setF$ and $\setR$ with task vectors $\tauF,\tauR$, their merging objective seeks $\taum$ minimizing $\sum_i\|(\taum-\vtau_i)\vtau_i^\top\|_F^2$, i.e., bringing the merged vector close to each expert in its own subspace. Substituting $\taum=\tauF+\tauR$ yields the cross-interference terms $\|\tauR\tauF^\top\|_F^2+\|\tauF\tauR^\top\|_F^2$. By Assumption~\ref{ass:input-subspace}, $\vtau_i^\top$ spans approximately the same subspace as $\mP_i$, recovering our Terms~1 and~2. Merging minimizes interference when combining tasks; unmerging maximizes separation when removing one---the same subspace vocabulary with the direction reversed; the correspondence is an analogy at the level of the objectives, not a formal duality.
\end{remark}
\section{Experiments}
\label{sec:exp}
\vspace{-5pt}
\subsection{Setup}
\label{sec:setup}
\vspace{-5pt}
\textbf{Models, data, and protocol.} We use ResNet-50~\citep{he2016deep} on CIFAR-100~\citep{krizhevsky2009learning} and TinyImageNet~\citep{le2015tiny}, and ViT-S/16~\citep{dosovitskiy2021image} on CIFAR-100 to probe architectural generality (App.~\ref{sec:vit}). Each model is finetuned from both initializations $\Winit\!\in\!\{\Wzero,\Wpt\}$ to obtain $\Wft$, which lets us probe how task-vector magnitude affects unmerging. Forget requests are class-level: a five-class superclass on CIFAR-100 (Trees, Aquatic Mammals, Vehicles) and a ten-class semantic group on TinyImageNet (Dogs \& Cats, Arthropods); $\setF$ is the union of the chosen classes' training images and $\setR=\setS\setminus\setF$. Two instance-level protocols are introduced in \S\ref{sec:instance}; class indices, training schedules, and all configurations are in App.~\ref{sec:expdetails}.

\textbf{Baselines.} We compare with the retrained model $\Wrt$ (gold standard), with NegGrad+~\citep{kurmanji2023towards}, Random Label (RL)~\citep{graves2021amnesiac}, SalUn~\citep{fan2024salun}, the task-vector family (negated task vector (NegTV)~\citep{ilharco2023editing} and NegMerge~\citep{kim2025negmerge}), and among gradient-free subspace methods, with the activation-space projection of \citet{kodge2024deep}. Hyperparameters are tuned per (method, init) on Trees and Dogs \& Cats and transferred to the other tasks with minimal changes, which stresses transferability (App.~\ref{sec:expdetails}). Our variants are \textsc{Unmerge-PCA} and \textsc{Unmerge-CAV} in \textsc{kmeans} and \textsc{class} modes (\S\ref{sec:method}, Alg.~\ref{alg:cav-basis}).

\textbf{Metrics.} We use three evaluation pipelines, each read as a gap to retraining (the gold standard) rather than in absolute terms. \textbf{(1) Performance.} Forget, retain and held-out test accuracy, and the composite \emph{tug-of-war} score of~\citet{zhao2024makes}
\begin{equation}
\label{eqn:tow}
  \mathrm{ToW}
  = \bigl(1 - |\mathrm{Acc}_{\setF} - \mathrm{Acc}_{\setF}^{\text{RT}}|\bigr)
  \cdot \bigl(1 - |\mathrm{Acc}_{\setR} - \mathrm{Acc}_{\setR}^{\text{RT}}|\bigr)
  \cdot \bigl(1 - |\mathrm{Acc}_{\text{test}} - \mathrm{Acc}_{\text{test}}^{\text{RT}}|\bigr),
\end{equation}
which equals $1$ when $\Wrt$ is recovered exactly and penalizes over- and under-unlearning alike. \textbf{(2) Entanglement} (\S\ref{sec:entangle}). Distribution distances between forget and retain features at the penultimate layer, as a relative gap to $\Wrt$. \textbf{(3) Privacy} (\S\ref{sec:mia}). The gap in membership-inference attack accuracy on the forget set between $\Wul$ and $\Wrt$, following~\citet{zhao2024makes,fan2024salun,shokri2017membership,song2021systematic} (App.~\ref{sec:mia-supp}); a gap far above retraining signals residual memorization, where far below signals over-unlearning (\S\ref{sec:instance}). App.~\ref{sec:seeds} reports three-seed $\mu$ and $\sigma$ on Trees.

\subsection{Main Unlearning Results}
\label{sec:main}
\vspace{-5pt}
Tables~\ref{tab:cifar100-summary} (CIFAR-100, mean\,$\pm$\,std across three superclasses) and~\ref{tab:tinyimgnt} (TinyImageNet, Dogs \& Cats) summarize the main comparisons; detailed per-superclass tables and the TinyImageNet Arthropods replication are deferred to App.~\ref{sec:detail-tables}. \textsc{Unmerge} consistently recovers retain and test accuracy close to retraining while driving forget accuracy toward the retrained level, and does so \emph{without any gradient updates on the training data}. We highlight several observations:

\textbf{Cross-task consistency.} On both datasets and both initializations \textsc{Unmerge} variants lead on tug-of-war (ToW); on CIFAR-100 at least two of the three variants sit above every gradient-based baseline, on TinyImageNet/$\Wzero$ only \textsc{CAV-Class} does. The dispersion across superclasses (Tab.~\ref{tab:cifar100-summary}) is small for the leading \textsc{Unmerge} cells ($\pm0.005$--$0.006$ ToW) and large for the most disruptive baselines (Random Label forget accuracy $9.2\!\pm\!5.1\%$ on $\Wpt$). Over three seeds on Trees every ordering of the \textsc{Unmerge} variants relative to SalUn and RL is preserved, with $1.7$--$4.5\times$ smaller ToW standard deviations on $\Wzero$ (App.~\ref{sec:seeds}); NegTV and NegMerge trail the best variant by $0.12$--$0.24$ ToW on every task because the vector they subtract, finetuned on $\setF$, overlaps retain directions, which our leakage Term~2 of \eqref{eqn:obj} penalizes; the projection of \citet{kodge2024deep} is the most competitive baseline on CIFAR-100 (mean ToW $0.962$/$0.944$, within $0.02$--$0.05$ of our best variant) but trails it by $0.17$--$0.35$ on TinyImageNet, where it leaves $18\%$ forget accuracy on $\Wzero$; and on ViT-S/16 (retrained test $86.1\%$; App.~\ref{sec:vit}) \textsc{Unmerge-CAV-KMeans} leads every axis (ToW $0.9975$ vs.\ $0.9949$ for SalUn at $14\times$ lower runtime, MIA gap and entanglement $1.8\times$ and $7\times$ closer to retraining). On $\Wzero$ where $\taum$ is larger in deeper layers, supervision pays off more: \textsc{Unmerge-CAV-Class} leads on both datasets ($0.9909$ on CIFAR-100, $0.9152$ on TinyImageNet), and PCA is most exposed because its top-$k$ eigenvectors absorb shared forget/retain directions.

\textbf{Efficiency.} \textsc{Unmerge} takes $8$--$105$s on one H100: the PCA variant is about $2\times$ faster than NegGrad+, the CAV variants are on par with it, and all are $\sim$$5\times$ faster than Random Label and SalUn ($110$--$113$s on CIFAR-100, $666$--$679$s on TinyImageNet), \emph{without any retain-loop fine-tuning}. Combined with the ToW gains this is up to a $\sim$24\% relative lift over NegGrad+, the comparable-runtime gradient baseline (the task-vector and projection baselines are as fast or faster but trail by more), and $\sim$18\% over Random Label and SalUn at $5\times$ their speed. Over $95\%$ of the runtime is the forward pass and basis construction; the speedup is set by the baselines' epochs, not by model size (App.~\ref{sec:cost}).

\textbf{Scaling up for Language models.} Though not the focus of our study, we conduct preliminary experiments on TOFU with Llama-3.2-3B-Instruct (App.~\ref{sec:llm}) and find that \textsc{Unmerge} also works on a 3B-parameter LLM: without any training (a retain-deflated PCA basis and the closed-form solver replace the vision configuration), the edited model reaches a forget quality of $0.13$ at a model utility of $0.630$ (retain-only model: $0.650$), against $0.09$ and $0.633$ for RMU, the strongest default-setting baseline we ran, indicating promising potential for future extension.

\begin{table}[t]
\centering
\caption{CIFAR-100 unlearning across three superclasses (Trees, Aquatic Mammals, Vehicles). Cells report mean\,$\pm$\,1\,std across superclasses (sample std, $N{=}3$; one seed per cell, three-seed variance in App.~\ref{sec:seeds}); $\Wrt$ shows plain means. Best non-retrained in \bestlegend{bold}. \textsc{Unmerge} variants achieve best ToW performance with noticeable short runtime. Per-superclass detail tables are in App.~\ref{sec:detail-tables}.}
\vspace{-5pt}
\label{tab:cifar100-summary}
\begin{subtable}{\linewidth}
\centering
\begin{adjustbox}{max width=\linewidth}
\begin{tabular}{l|c|cccccc|ccc}
\toprule
$\Winit=\Wpt$ & $\Wrt$ & NegGrad+ & Rand.\ Label & SalUn & NegTV & NegMerge & Kodge et al. & Unmerge-PCA & \shortstack[c]{Unmerge-CAV\\KMeans} & \shortstack[c]{Unmerge-CAV\\Class} \\
\midrule
Forget Acc $\downarrow$ & 0.00 & 0.31\,$\pm$\,0.28 & 9.24\,$\pm$\,5.09 & 6.12\,$\pm$\,2.16 & 2.27\,$\pm$\,2.19 & 3.43\,$\pm$\,3.47 & 0.32\,$\pm$\,0.14 & 0.11\,$\pm$\,0.08 & 0.08\,$\pm$\,0.11 & 0.69\,$\pm$\,0.65 \\
Retain Acc $\uparrow$ & 99.98 & 96.64\,$\pm$\,1.13 & 99.94\,$\pm$\,0.00 & 99.94\,$\pm$\,0.04 & 95.10\,$\pm$\,4.08 & 95.83\,$\pm$\,3.96 & 99.39\,$\pm$\,0.49 & 99.80\,$\pm$\,0.09 & 99.63\,$\pm$\,0.20 & 99.22\,$\pm$\,0.36 \\
Test Acc $\uparrow$ & 80.52 & 73.82\,$\pm$\,0.66 & 80.31\,$\pm$\,0.29 & 80.52\,$\pm$\,0.60 & 72.81\,$\pm$\,4.02 & 73.52\,$\pm$\,4.09 & 77.62\,$\pm$\,1.35 & 79.31\,$\pm$\,0.25 & 78.33\,$\pm$\,0.74 & 76.90\,$\pm$\,0.96 \\
ToW $\uparrow$ & 1.0000 & 0.8992\,$\pm$\,0.0199 & 0.9053\,$\pm$\,0.0485 & 0.9365\,$\pm$\,0.0226 & 0.8581\,$\pm$\,0.0562 & 0.8604\,$\pm$\,0.0479 & 0.9622\,$\pm$\,0.0150 & \best{0.9851\,$\pm$\,0.0046} & \textbf{0.9739\,$\pm$\,0.0077} & 0.9498\,$\pm$\,0.0070 \\
Runtime $\downarrow$ &  & 23\,$\pm$\,0 & 110\,$\pm$\,0 & 112\,$\pm$\,0 & \best{4\,$\pm$\,0} & 11\,$\pm$\,0 & 8\,$\pm$\,0 & 8\,$\pm$\,0 & 26\,$\pm$\,0 & 22\,$\pm$\,0 \\
\bottomrule
\end{tabular}
\end{adjustbox}
\end{subtable}
\begin{subtable}{\linewidth}
\centering
\begin{adjustbox}{max width=\linewidth}
\begin{tabular}{l|c|cccccc|ccc}
\toprule
$\Winit=\Wzero$ & $\Wrt$ & NegGrad+ & Rand.\ Label & SalUn & NegTV & NegMerge & Kodge et al. & Unmerge-PCA & \shortstack[c]{Unmerge-CAV\\KMeans} & \shortstack[c]{Unmerge-CAV\\Class} \\
\midrule
Forget Acc $\downarrow$ & 0.00 & 4.53\,$\pm$\,1.13 & 9.40\,$\pm$\,6.02 & 3.92\,$\pm$\,0.73 & 7.32\,$\pm$\,6.86 & 7.84\,$\pm$\,6.77 & 0.87\,$\pm$\,0.87 & 0.51\,$\pm$\,0.34 & 0.41\,$\pm$\,0.36 & 0.43\,$\pm$\,0.61 \\
Retain Acc $\uparrow$ & 99.97 & 96.89\,$\pm$\,0.75 & 99.96\,$\pm$\,0.01 & 99.95\,$\pm$\,0.02 & 95.73\,$\pm$\,1.14 & 96.21\,$\pm$\,1.15 & 98.35\,$\pm$\,1.14 & 97.05\,$\pm$\,1.53 & 99.70\,$\pm$\,0.11 & 99.85\,$\pm$\,0.03 \\
Test Acc $\uparrow$ & 72.47 & 67.54\,$\pm$\,1.49 & 73.95\,$\pm$\,0.63 & 74.18\,$\pm$\,0.50 & 66.72\,$\pm$\,1.89 & 66.95\,$\pm$\,1.88 & 69.19\,$\pm$\,2.07 & 68.01\,$\pm$\,2.00 & 71.79\,$\pm$\,0.95 & 72.66\,$\pm$\,0.82 \\
ToW $\uparrow$ & 1.0000 & 0.8798\,$\pm$\,0.0235 & 0.8925\,$\pm$\,0.0564 & 0.9443\,$\pm$\,0.0111 & 0.8370\,$\pm$\,0.0722 & 0.8387\,$\pm$\,0.0755 & 0.9435\,$\pm$\,0.0317 & 0.9230\,$\pm$\,0.0347 & \textbf{0.9865\,$\pm$\,0.0029} & \best{0.9909\,$\pm$\,0.0062} \\
Runtime $\downarrow$ &  & 23\,$\pm$\,1 & 110\,$\pm$\,1 & 113\,$\pm$\,2 & \best{4\,$\pm$\,0} & 11\,$\pm$\,0 & 8\,$\pm$\,0 & 8\,$\pm$\,2 & 21\,$\pm$\,0 & 18\,$\pm$\,0 \\
\bottomrule
\end{tabular}
\end{adjustbox}
\end{subtable}
\vspace{-7.5pt}
\end{table}

\begin{table}[htb]
\centering
\caption{TinyImageNet unlearning results (Dogs \& Cats). Best non-retrained in \bestlegend{bold}. \textsc{Unmerge-CAV-Class} achieves the best ToW for both $\Winit$ at a runtime comparable to NegGrad+; the task-vector and projection baselines are faster but trail on ToW. See Arthropods results in Tab.~\ref{tab:tinyimgnt-arthropods}.}
\vspace{-5pt}
\label{tab:tinyimgnt}
\begin{subtable}{\linewidth}
\centering
\label{tab:tinyimgnt-wpt}
\begin{adjustbox}{max width=\linewidth}
\begin{tabular}{l|c|cccccc|ccc}
\toprule
$\Winit= \Wpt$ & $\Wrt$ & NegGrad+ & Rand.\ Label & SalUn & NegTV & NegMerge & Kodge et al. & Unmerge-PCA & \shortstack[c]{Unmerge-CAV\\KMeans} & \shortstack[c]{Unmerge-CAV\\Class} \\
\midrule
Forget Acc $\downarrow$ & 0.0 & 1.44 & 16.26 & 15.46 & 0.26 & 0.86 & 1.90 & 0.68 & 0.00 & 0.04 \\
Retain Acc $\uparrow$ & 99.98 & 96.36 & 99.98 & 99.98 & 89.70 & 91.62 & 93.30 & 99.37 & 99.96 & 99.94 \\
Test Acc $\uparrow$ & 74.46 & 67.98 & 74.65 & 74.62 & 62.57 & 64.29 & 63.91 & 72.07 & 73.41 & 73.58 \\
ToW $\uparrow$ & 1 & 0.8883 & 0.8358 & 0.8440 & 0.7885 & 0.8161 & 0.8189 & 0.9635 & \textbf{0.9893} & \best{0.9904} \\
Runtime $\downarrow$ &  & 93s & 666s & 671s & \best{18s} & 54s & 22s & 71s & 105s & 99s \\
\bottomrule
\end{tabular}
\end{adjustbox}
\end{subtable}
\begin{subtable}{\linewidth}
\centering
\label{tab:tinyimgnt-w0}

\begin{adjustbox}{max width=\linewidth}
\begin{tabular}{l|c|cccccc|ccc}
\toprule
$\Winit=\Wzero$ & $\Wrt$ & NegGrad+ & Rand.\ Label & SalUn & NegTV & NegMerge & Kodge et al. & Unmerge-PCA & \shortstack[c]{Unmerge-CAV\\KMeans} & \shortstack[c]{Unmerge-CAV\\Class} \\
\midrule
Forget Acc $\downarrow$ & 0.0 & 0.26 & 9.44 & 8.92 & 8.50 & 9.48 & 17.74 & 4.60 & 2.82 & 0.52 \\
Retain Acc $\uparrow$ & 99.98 & 85.64 & 99.98 & 99.97 & 93.29 & 94.06 & 96.39 & 83.72 & 97.87 & 97.88 \\
Test Acc $\uparrow$ & 61.78 & 48.08 & 61.29 & 61.12 & 52.18 & 52.45 & 54.46 & 46.10 & 55.31 & 55.75 \\
ToW $\uparrow$ & 1 & 0.7373 & 0.9011 & \textbf{0.9047} & 0.7718 & 0.7722 & 0.7350 & 0.6736 & 0.8898 & \best{0.9152} \\
Runtime $\downarrow$ &  & 101s & 670s & 679s & \best{18s} & 53s & 23s & 56s & 86s & 84s \\
\bottomrule
\end{tabular}
\end{adjustbox}
\end{subtable}
\vspace{-7.5pt}
\end{table}

\subsection{Where Does Unlearning Happen?}
\label{sec:perlayer}
\vspace{-5pt}
Because \textsc{Unmerge} builds per-layer bases $\QF^l,\QR^l$ and extracts $\tauF^l$, we can quantify each layer's contribution. Empirically, the useful unmerging budget concentrates in the deeper layers: skipping the first 35/40 conv layers in ResNet-50 retains the forget drop while \emph{improving} retain/test accuracy by 2--5 points. This matches the feature hierarchy of vision networks---generic edges and textures early, class-specific parts and objects late~\citep{zeiler2014visualizing,yosinski2014transferable,bau2017network}---and the depth-aware edits of recent class-unlearning methods~\citep{hatami2026damp,chang2024zeroshot}. It is a property of the request rather than of the method: a class- or concept-level forget set is separable where semantics are encoded, whereas a request defined by a low-level attribute such as a colour or texture would call for earlier layers, which the same diagnostics would indicate. Theory predicts the depth profile: early layers have larger tail eigenvalues $\sum_{i>k}\nu_i$ (Lemma~\ref{lem:lowrank}) and a larger retain-damage constant $\mu_1$ (\eqref{eq:retain-bound}) because generic edge/texture features dominate their second-moment spectrum, so any forget correction leaks heavily into retain; deep layers encode class-specific representations, $\mu_{k+1}$ decays, and the forget/retain spectra are more separable.

\textbf{Per-layer geometry.} Figs.~\ref{fig:crossfr} and~\ref{fig:taum} (App.~\ref{sec:ablation}) make this visible on Trees. \emph{Magnitude} (Fig.~\ref{fig:taum}): the cross-init gap is depth-dependent---in layer1--2 $\|\taum\|$ is slightly larger on $\Wpt$ (pretrained features must be shifted away from ImageNet), but at layer4 and the classifier it is $\sim$$1.5$--$2\times$ larger on $\Wzero$ (peak $15.4$/$15.8$ vs.\ $10.0$/$8.8$), where the deepest task-specific weights are learned from scratch. \emph{Entanglement} (Fig.~\ref{fig:crossfr}): the mean per-direction overlap $\mathrm{mean}_i\,\|\PR^l\vq_i\|$ of the forget directions with the retain subspace, all bases at rank $16$; it is the quantity that enters the separability score of \S\ref{sec:instance}. It decays with depth for every basis, from $0.8$--$1.0$ in layer1--2 to $0.4$--$0.6$ at layer4; PCA is the most entangled at every depth ($0.54$/$0.58$ at layer4, $0.46$/$0.56$ at the classifier on $\Wpt$/$\Wzero$), the discriminative CAV probes suppress the shared variance that PCA's eigenvectors absorb ($0.40$--$0.46$ and $0.08$--$0.20$), and all three bases are nearly init-invariant (block means within $0.13$, $0.05$ at layer4): what changes between regimes is the magnitude of $\taum$, not the entanglement of the bases.

\textbf{From diagnostics to hyperparameters.} The two diagnostics fix the hyperparameters used throughout---\texttt{skip-layers} at the overlap elbow, $\alpha$ inversely to the deepest-layer $\|\taum\|$, and CAV bases where entanglement is the binding constraint---and explain why \textsc{Unmerge-PCA} leads on $\Wpt$ yet trails on $\Wzero$ at the same entanglement; $\alpha$ itself can be line-searched on forget-set accuracy alone without $\Wrt$ or test data (App.~\ref{sec:ablation}).

\begin{figure}[t]
  \centering
  \begin{subfigure}{0.49\linewidth}
    \includegraphics[width=\linewidth]{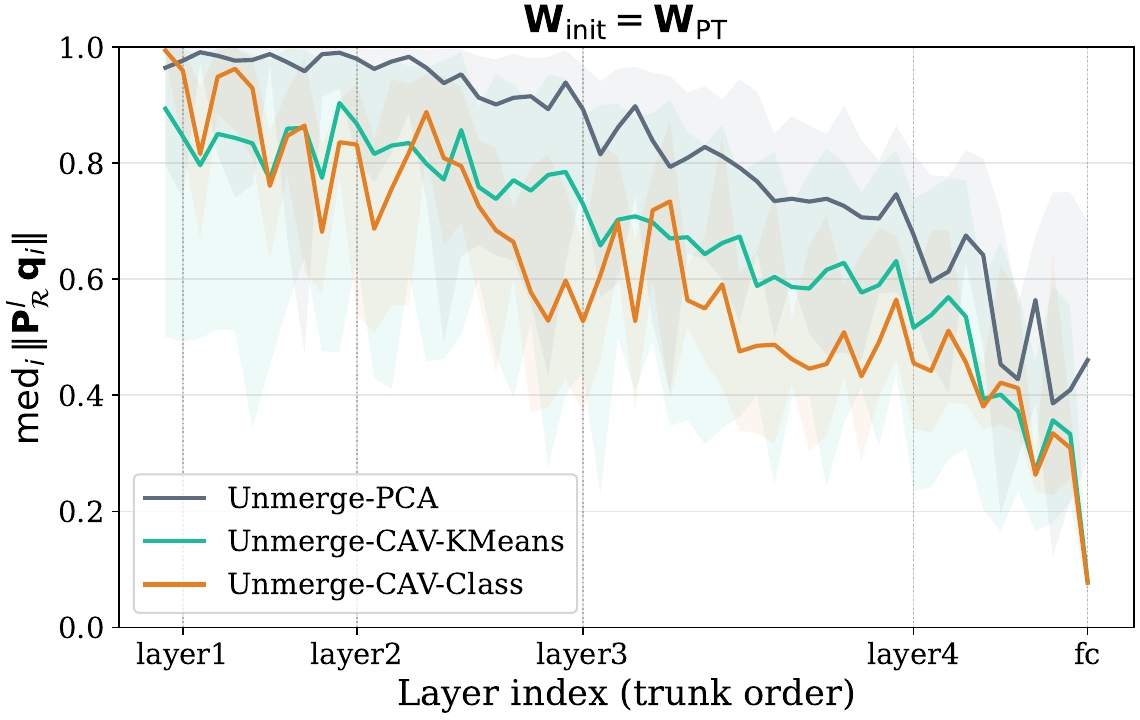}
    \caption{$\Winit=\Wpt$}
  \end{subfigure}
  \begin{subfigure}{0.49\linewidth}
    \includegraphics[width=\linewidth]{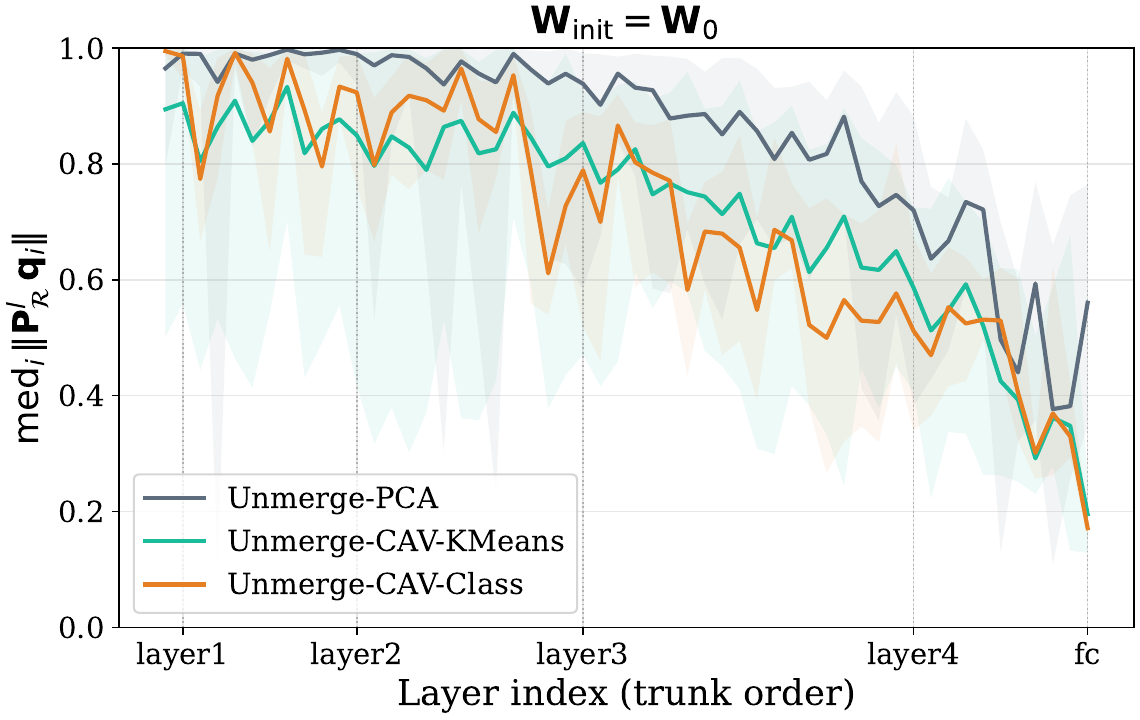}
    \caption{$\Winit=\Wzero$}
  \end{subfigure}
  \vspace{-5pt}
  \caption{Per-layer forget--retain entanglement on CIFAR-100 (Trees), ResNet-50: mean over the forget basis directions $\vq_i$ of the overlap $\|\PR^l\vq_i\|$ with the retain subspace, all bases at rank $16$, seed~1 (the band is the spread over directions). This per-direction overlap is the rank-invariant form of the basis cross-norm $\|\QF^{l\top}\QR^l\|_F$, the structural analogue of the merging cross-interference term $\|\tauF\tauR^\top\|_F$ (Asm.~\ref{ass:input-subspace}, Remark~\ref{rem:duality}); smaller means less entangled.}
  \label{fig:crossfr}
  \vspace{-12.5pt}
\end{figure}

\subsection{Feature-Level Entanglement}
\label{sec:entangle}
\vspace{-5pt}
The per-layer overlap of Fig.~\ref{fig:crossfr} exists only for basis-constructing methods. To compare all methods on equal footing we measure entanglement in feature space, as the distance between the forget and retain feature distributions at the penultimate layer. Under retraining, forget images are out-of-distribution and project onto retain class manifolds, so the two clouds mix and the distance is small; in a model that has not unlearned they stay separable. An unlearned model should therefore \emph{match the retrained distance}: Tab.~\ref{tab:entangle-feat} reports the relative gap $|d_{\text{UL}}-d_{\text{RT}}|/d_{\text{RT}}$ under four distances on a same-size retain subsample---Gaussian-RBF MMD$^2$ (single and multi-bandwidth), the Bures--Wasserstein W$_2^2$ between Gaussian fits, and Sliced-W$_2^2$ (App.~\ref{sec:expdetails}).

\begin{table}[htb]
\centering
\vspace{5pt}
\caption{Feature-level entanglement on CIFAR-100 (Trees): penultimate-layer distribution distance between forget and a same-size retain subsample, reported as relative gap to the retrained reference (lower is better; $\Wrt$ is 0 by construction). Best non-retrained per row in \bestlegend{bold}.}
\vspace{-5pt}
\label{tab:entangle-feat}
\begin{subtable}{\linewidth}
\centering
\label{tab:entangle-feat-wpt}
\begin{adjustbox}{max width=\linewidth}
\begin{tabular}{l|c|cccccc|ccc}
\toprule
$\Winit=\Wpt$ & $\Wrt$ & NegGrad+ & Rand.\ Label & SalUn & NegTV & NegMerge & Kodge et al. & Unmerge-PCA & \shortstack[c]{Unmerge-CAV\\KMeans} & \shortstack[c]{Unmerge-CAV\\Class} \\
\midrule
MMD$^2$ & 0.000 & 0.612 & 1.085 & 0.951 & 0.702 & 0.699 & 0.627 & 0.319 & \best{0.255} & 0.267 \\
MMD$^2$-multi & 0.000 & 0.638 & 2.030 & 1.917 & 0.735 & 0.731 & 0.638 & 0.309 & \best{0.217} & 0.257 \\
W$_2^2$-Bures & 0.000 & 0.313 & 0.488 & \best{0.273} & 0.522 & 0.505 & 0.548 & 0.553 & 0.499 & 0.611 \\
Sliced-W$_2^2$ & 0.000 & 0.654 & 1.251 & 0.916 & 0.736 & 0.734 & 0.722 & 0.092 & \best{0.046} & 0.125 \\
\midrule
Mean rel.\ gap $\downarrow$ & 0.000 & 0.554 & 1.213 & 1.014 & 0.674 & 0.667 & 0.634 & 0.318 & \best{0.254} & 0.315 \\
\bottomrule
\end{tabular}
\end{adjustbox}
\end{subtable}
\begin{subtable}{\linewidth}
\centering
\label{tab:entangle-feat-w0}
\begin{adjustbox}{max width=\linewidth}
\begin{tabular}{l|c|cccccc|ccc}
\toprule
$\Winit=\Wzero$ & $\Wrt$ & NegGrad+ & Rand.\ Label & SalUn & NegTV & NegMerge & Kodge et al. & Unmerge-PCA & \shortstack[c]{Unmerge-CAV\\KMeans} & \shortstack[c]{Unmerge-CAV\\Class} \\
\midrule
MMD$^2$ & 0.000 & 0.271 & 1.938 & 1.834 & 0.512 & 0.499 & 0.549 & 0.219 & 0.181 & \best{0.038} \\
MMD$^2$-multi & 0.000 & 0.340 & 3.119 & 3.015 & 0.546 & 0.531 & 0.529 & 0.252 & 0.172 & \best{0.052} \\
W$_2^2$-Bures & 0.000 & 0.078 & 1.659 & 1.661 & 0.233 & 0.201 & 0.384 & 0.438 & \best{0.009} & 0.117 \\
Sliced-W$_2^2$ & 0.000 & 0.204 & 3.265 & 3.226 & 0.403 & 0.378 & 0.535 & 0.340 & \best{0.071} & 0.181 \\
\midrule
Mean rel.\ gap $\downarrow$ & 0.000 & 0.223 & 2.495 & 2.434 & 0.423 & 0.402 & 0.499 & 0.312 & 0.108 & \best{0.097} \\
\bottomrule
\end{tabular}
\end{adjustbox}
\end{subtable}
\vspace{-10pt}
\end{table}

\textsc{Unmerge} is the closest to retraining on both initializations (Tab.~\ref{tab:entangle-feat}): the CAV variants land within $\sim$$0.10$ relative gap on $\Wzero$ ($0.25$--$0.32$ on $\Wpt$), whereas Random Label and SalUn deviate by $1.0$--$2.5\times$ the retrained reference, $3$--$25\times$ the CAV gap, and \citet{kodge2024deep} by $0.50$--$0.63$. The un-edited $\Wft$ calibrates the axis: it sits at $0.67$ on $\Wzero$ and at $0.45$--$0.70$ on TinyImageNet, and \textsc{Unmerge} moves the features from there toward retraining in all three settings ($0.06$--$0.31$; Tab.~\ref{tab:entangle-feat-tin}, App.~\ref{sec:entangle-tin}), whereas relabeling pushes them \emph{away} from retraining in every setting (on TinyImageNet/$\Wzero$ the projection of \citet{kodge2024deep} reaches the retrained distance too, but while still classifying $18\%$ of the forget set and losing $7$ test points, which is why this axis is read together with accuracy and privacy); on CIFAR-100/$\Wpt$ the pretrained features already sit near retraining for every model ($0.22$ un-edited) and the axis has no headroom. (W$_2^2$-Bures on $\Wpt$, a mean-shift-dominated distance, favors SalUn.) This exposes a \emph{label/feature dissociation}: Random Label and SalUn drive forget accuracy down but leave the forget representations cleanly separable from retain, whereas \textsc{Unmerge} lowers it while matching retraining's forget-retain feature separation. The gap measures distributional fidelity, not decodability (App.~\ref{sec:limitations}).

\subsection{Membership Privacy, Scope and Fidelity}
\label{sec:mia}
\vspace{-5pt}
\textbf{Membership privacy.} 
On CIFAR-100 the projection of \citet{kodge2024deep} has the smallest mean MIA gap ($0.017$/$0.028$ on $\Wzero$/$\Wpt$), followed on $\Wzero$ by \textsc{Unmerge-CAV-KMeans} ($0.026$) and \textsc{CAV-Class} ($0.029$) ahead of SalUn ($0.047$), Random Label ($0.062$) and NegGrad+ ($0.109$); on $\Wpt$ the other methods sit within $0.067$--$0.093$. On TinyImageNet \textsc{Unmerge} sweeps $\Wpt$ at $0.009$--$0.020$ against $0.043$--$0.261$ for every baseline (full details in App.~\ref{sec:mia-supp}). Under these attacks, \textsc{Unmerge} remains close to the retrained reference in membership-inference exposure at class level.

\textbf{Concept-level vs.\ instance-level requests.}\label{sec:instance} \textsc{Unmerge} removes what the forget samples \emph{share}, so we test two instance-level protocols on Trees, each with its own retrained reference (App.~\ref{sec:instance-supp}). On \emph{coherent} subsets (half of the images of each Trees class) the aggregate scores favor \textsc{Unmerge} (ToW $0.97$/$0.94$ vs.\ $0.94$ for a re-tuned NegGrad+ and $0.81$--$0.90$ for relabeling), but a per-class breakdown shows a concept-level edit: the two halves of a class are i.i.d., so the edit lowers both alike, whereas relabeling protects the retained half and under-forgets the other; no method reproduces retraining. On \emph{random} subsets no method does either: \textsc{Unmerge} and NegGrad+ stay near-inert, and Random Label and SalUn score only by \emph{marking} the forget set (attack accuracy $0.84$--$0.88$ vs.\ $0.63$--$0.68$ under retraining). The Phase-1 separability score $S=0.17$ is consistent with both outcomes: it tells whether a request carries a concept the edit can act on. Removing the influence of given samples and removing the concept they represent are distinct problems with different references and metrics~\citep{triantafillou2026untraining,cooper2025machine,zhao2024makes}, which coincide only when $\setF$ covers the concept, and \textsc{Unmerge} targets the latter.

\textbf{Does unmerging recover the retrained task vector?}\label{sec:tauval} Beyond accuracy proxies, we compare $\tauR=\taum-\alpha\tauF$ with $\tauRstar=\Wrt-\Winit$ directly on Trees (App.~\ref{sec:tauval-supp}). In weight space two independent retrains differ by $0.72$--$0.73$ relative error and the un-edited $\Wft$ already sits at $0.75$; \textsc{Unmerge} stays at that floor ($0.750$--$0.754$: the minimal edit does no weight-space harm), whereas Random Label ($0.82$) and SalUn ($0.91$) move \emph{away} from the retrained solution. Output space discriminates more sharply in the same direction: \textsc{Unmerge-CAV-KMeans} has the smallest forget-set divergence to $\Wrt$ in both regimes ($2.9$ vs.\ $4.9$--$5.2$ on $\Wzero$, $5.3$ vs.\ $5.9$--$6.6$ on $\Wpt$), because relabeling pushes forget inputs confidently to wrong classes while subspace subtraction lands closer to a model that never saw them.

\section{Conclusion}
\label{sec:conclusion}
\vspace{-5pt}
We recast approximate machine unlearning as the inverse of model merging: \textsc{Unmerge} learns a low-rank forget component of the merged task vector under a subspace objective whose terms bound forget leakage and retain damage in activation space, and subtracts it without any gradient update on the training data. On class-level forgetting with ResNet-50 and ViT-S/16 it improves the tug-of-war trade-off over gradient-based and task-vector baselines at a fraction of their cost, keeps membership-inference exposure at the level of retraining, and, unlike relabeling methods, moves the forget representations toward the retrained model rather than away from it. The per-layer bases and task-vector magnitudes double as a diagnostic that motivates the depth and scale choices and, computed before any edit, delimits the method's scope to concept-level requests; instance-level requests are a different problem that no method we tested solves (further limitations in App.~\ref{sec:limitations}). Promising directions include language models beyond our preliminary study (App.~\ref{sec:llm}), continual forget requests, and adaptive per-layer rank and depth selection driven by the same diagnostics.

\textbf{Acknowledgement.} RK thanks the Central Indiana Corporate Partnership AnalytiXIN Initiative and NSF Award 2543174 for their support.
\clearpage
\subsection*{AI use statement}
\label{sec:ai-use}
\vspace{-5pt}
In this work, we used a generative AI (Claude) for the following tasks with required disclosure: debugging, providing feedback on experimental design, implementing utility code such as plotting, and interpreting results. We have not used generative AI tools to generate synthetic data sets, for translation, or to clean or reformat data sets; qualitative and thematic data analysis is not applicable to this work. Additionally, we used the assistant to search and summarize related literature. We have reviewed and verified all AI-assisted work. We take responsibility for the final content of this work, including text, claims and artifacts produced with the aid of generative AI.

\subsection*{Ethics statement}
\vspace{-5pt}
This work uses only public image-classification benchmarks (CIFAR-100, TinyImageNet) and publicly available pretrained weights, and involves no human subjects or personal data. Machine unlearning serves data-protection goals, but an approximate method must not be presented as a guarantee of deletion: we therefore evaluate against retraining on accuracy, feature and membership-inference axes, and we report the forget requests for which our method does not work (\S\ref{sec:instance}). Our finding that relabeling-based unlearning can make forget samples \emph{more} identifiable is reported to inform practitioners. We are not aware of other ethical concerns.

\subsection*{Reproducibility statement}
\vspace{-5pt}
Algorithms~\ref{alg:unmerge}--\ref{alg:cav-basis} give complete pseudocode, and App.~\ref{sec:proof} contains the assumptions, proofs and the derivation of the closed-form minimizer. App.~\ref{sec:expdetails} lists the forget tasks, training schedules, baseline settings and all \textsc{Unmerge} configurations; App.~\ref{sec:ablation} documents how they are chosen; App.~\ref{sec:seeds} reports seed variance; and App.~\ref{sec:winit} identifies the pretrained checkpoint by its hash. All datasets are public, and every experiment runs on a single NVIDIA H100 GPU (runtimes in App.~\ref{sec:cost}). We will release the code, configurations and evaluation pipeline upon acceptance.


\clearpage
\bibliography{iclr2027_conference}
\bibliographystyle{iclr2027_conference}

\clearpage
\appendix
\section*{Appendix}
\vspace{-5pt}
\startcontents[sections]
\printcontents[sections]{l}{1}{\setcounter{tocdepth}{3}}

\section{Additional Related Work}
\label{sec:related-supp}
\vspace{-5pt}
We expand the comparison to subspace-projection and representation-editing methods that the main paper Related Work (\S\ref{sec:related}) treats only briefly.

\textbf{Continual learning via subspace projection.} Gradient Projection Memory~\citep{saha2021gradient} and Orthogonal Gradient Descent~\citep{farajtabar2020orthogonal} project gradient updates onto subspaces orthogonal to previously seen tasks, preventing catastrophic forgetting by hard separation. We share their intuition that forget and retain knowledge should occupy separable subspaces, but rather than enforcing separation by hard orthogonal projection, we let the bases $\QF,\QR$ retain their natural overlap and penalize retain leakage softly via Term~2 of \eqref{eqn:obj}, so that partially shared directions are explicitly traded off rather than discarded by construction. This soft trade-off is what makes \textsc{Unmerge} work in regimes where the forget and retain subspaces are genuinely entangled (e.g., $\Wpt$ where shared ImageNet-flavor variance lifts $\|\QF^{\top}\QR\|_F$ uniformly; see \S\ref{sec:perlayer}).

\textbf{Gradient-free activation-subspace class unlearning.} \citet{kodge2024deep} estimate retain and forget spaces by SVD of layer activations on a few samples, remove the shared part from the forget space, and project the weights so that the remaining class-discriminatory directions are suppressed, selecting two scaling hyperparameters on training accuracy alone. Null-space calibration~\citep{chen2024unsc} confines the unlearning update to the null space of retain activations, and the concurrent DAMP~\citep{hatami2026damp} removes forget directions relative to retain prototypes with a depth-aware scale and argues that class unlearning is often carried by the classifier head. With $\Winit{=}0$ our update reads $\Wul=\Wft(\mI-\alpha\,\QF\mM\QF^{\top})$ with the $k\times k$ filter $\mM=((1+\gamma)\mI+\lambda\mC\mC^{\top})^{-1}$ for the exact minimizer, i.e.\ a soft input-space projection of the same family; the hard projection $\PF(\mI-\PR)$ of \citet{kodge2024deep} corresponds to removing the retain-shared part of the forget directions instead of down-weighting it. Our formulation differs in acting on the task vector, which under $\Wpt$ leaves the pretrained weights intact and matters for the unsupervised basis (App.~\ref{sec:winit}); in deriving the filter from an explicit objective that trades forgetting against retain leakage; in offering supervised CAV bases; and in exposing the same quantities as diagnostics. \S\ref{sec:exp} benchmarks \citet{kodge2024deep} under the protocol of the other baselines (App.~\ref{sec:expdetails}).

\textbf{Linear concept erasure.} LEACE~\citep{belrose2023leace} and RLACE~\citep{ravfogel2022linear} give closed-form linear concept erasure with activation-distribution guarantees, and Concept Activation Vectors (CAV)~\citep{kim2018interpretability} probe concept directions with linear classifiers trained on labeled positive/negative activation pools. We use CAV as one of our basis constructions (Alg.~\ref{alg:cav-basis}) and inherit its supervised, discriminative flavor, but \emph{lift} it from representation-level probing/erasure to weight-level subtraction via task-vector decomposition: instead of projecting activations through a learned operator at inference time, we modify the weights themselves so the forget direction is no longer encoded.

\textbf{Knowledge editing.} ROME~\citep{meng2022locating} and similar locate-and-edit methods modify specific low-rank weight subspaces to update factual associations in language models. They share our weight-level intervention philosophy but require localized $(s, r, o)$ associations rather than distributional forget/retain sets, and target editing of specific facts rather than removal of broader behaviors. Our setting (class-level forgetting from distributional samples of $\setF$ and $\setR$) is complementary.

\textbf{Unifying view.} The retain projector $\PR$ in our objective plays the role of the ``preserve'' subspace from continual learning, while $\PF$ inherits the concept-direction view from CAV/LEACE. Together, the three terms of \eqref{eqn:obj} (forget removal, retain leakage, magnitude regularization) map cleanly onto the activation-space bounds of \S\ref{sec:theory}, giving a single objective that recovers and softens individual instances of these prior approaches.

\section{Method Details}
\label{sec:method-supp}
\vspace{-5pt}
\subsection{Algorithm and Implementation Details}
\label{sec:algo-supp}
\vspace{-5pt}
This subsection expands the implementation sketch in \S\ref{sec:method} and provides the complete pseudocode. Algorithm~\ref{alg:unmerge} realizes \eqref{eqn:obj} as a three-phase, \emph{optimization-time data-free} procedure: data enters only through the second-moment statistics gathered in Phase~1, and no gradient is taken with respect to a training sample.

\textbf{Phase~1: basis construction.} A forward pass through $\setF\cup\setR$ (plus a second pass over $\setR$ for the CAV variants) collects the layer inputs $\mX_{\setF}^l, \mX_{\setR}^l$ via forward hooks. The retain basis $\QR^l$ is built as the top-$k$ eigenbasis of $\frac{1}{N_{\setR}}{\mX_{\setR}^l}^\top\mX_{\setR}^l$ (PCA, Alg.~\ref{alg:pca-basis}). The forget basis $\QF^l$ is built either by the same PCA construction or by the discriminative CAV construction in Alg.~\ref{alg:cav-basis}, which learns a ridge probe per cluster against the retain pool and stacks the orthogonalized probe directions as the basis. CAV admits two cluster definitions: \emph{CAV-KMeans} partitions forget activations into $k$ unsupervised clusters via $k$-means; \emph{CAV-Class} uses the available forget class labels, producing one direction per forget class (so $k' = |\setY_{\setF}|$ regardless of $k$). The two modes share the ridge-probe and orthogonalization machinery and differ only in line~5 of Alg.~\ref{alg:cav-basis}. KMeans is purely activation-driven and applicable when class labels are unavailable or when the forget set is sub-class; Class exploits supervision when one direction per forget class is the natural granularity, and matches PCA in cost since it skips the clustering step. The projection $\taum^l \QF^l$---the minimizer of Term~1 of \eqref{eqn:obj} in isolation---and the overlap $\mC^l={\QF^l}^\top \QR^l$ are computed here; they are all Phase~2 needs.

\textbf{Phase~2: optimization in projected space.} Substituting $\tauF^l = \mB^l {\QF^l}^\top$ into \eqref{eqn:obj} reduces every term to small matrix products in the projected basis---a strictly convex quadratic in $\mB^l$ whose exact minimizer is \eqref{eq:closed} (App.~\ref{sec:proof}). With $\taum^l \QF^l$ and $\mC^l$ pre-computed once, our procedure initializes $\mB^l$ at the projection $\taum^l \QF^l$ and takes 100 Adam steps at learning rate $10^{-3}$ ($mk$ parameters per layer, never revisiting the data). This budget does not reach the exact minimizer; App.~\ref{sec:closedform} measures the gap and shows that applying ${\mB^l}^\star$ directly at the paper's $(\lambda,\gamma)$ over-shrinks $\tauF$ on high-overlap bases, i.e.\ the fixed budget functions as an early-stopping regularizer.

\textbf{Phase~3: application.} We apply $\Wul^l=\Wft^l-\alpha\,\mB^l {\QF^l}^\top$ with a scalar $\alpha$ that corrects for the rank-truncation residual $\|\mE\|_F$ of Lemma~\ref{lem:lowrank} and for imperfect rank. The optional \texttt{skip-layers} knob restricts the update to deeper layers, where the retain-damage constant $\mu_1$ of \eqref{eq:retain-bound} is smaller and the forget/retain spectra are more separable; \S\ref{sec:perlayer} reads the $\alpha$ and \texttt{skip-layers} choices directly off Figs.~\ref{fig:crossfr} and~\ref{fig:taum}.

\textbf{Implementation notes.} Five details of the released code differ from, or are not visible in, the description above. \emph{(i) Centering.} Bases are computed from the \emph{centered} covariance of the layer inputs rather than from the second moment used in \S\ref{sec:theory}; this removes the shared mean-activation direction of post-ReLU features, and the bounds of Proposition~\ref{prop:bounds} then hold for centered inputs, with the mean response $\|\tauR\,\vmu_{\setF}\|^2$ as an additional forget term. \emph{(ii) CAV pools.} In Alg.~\ref{alg:cav-basis} the forget and retain pools are mean-centered \emph{separately} before the ridge probes are fit, so each probe contrasts one cluster (or class) with the remaining activations in centered coordinates; the common forget-vs-retain mean shift is removed by construction. With $|\setY_{\setF}|{=}5$ classes the class probes are nearly linearly dependent (effective rank $4$), so the fifth orthogonalized direction of \textsc{CAV-Class} carries almost no signal. \emph{(iii) Subsampling.} PCA statistics use every sample, but the CAV probes are fit on at most \texttt{max-points} activation vectors per layer ($8192$ on CIFAR-100, $16384$ on TinyImageNet), taken from the first shuffled batches with a fixed loader seed; for a convolutional layer with $L$ spatial positions this is $\texttt{max-points}/L$ images (e.g.\ $512$ images at layer4), and CAV makes a second pass over the retain set. \emph{(iv) Biases.} For an edited layer with a bias, the code shrinks the bias task vector uniformly, $b_{\text{UL}}=b_{\text{init}}+\bigl(1-\alpha\gamma/(1+\gamma)\bigr)(b_{\text{FT}}-b_{\text{init}})$. In ResNet-50 only the classifier has a bias ($\|b_{\text{FT}}-b_{\text{init}}\|\le0.15$, a logit shift below $0.02$), so the rule is immaterial there; ViT-S/16 has a bias in every edited layer, and App.~\ref{sec:vit} reports the ViT cells with and without the rule. \emph{(v) Normalization layers.} Only convolutional and linear weights are edited. BatchNorm (LayerNorm for ViT) parameters and running statistics are kept from $\Wft$, activations are collected in evaluation mode, and no recalibration pass is run after the edit.

\begin{algorithm}[tb]
  \caption{Unmerging}
  \label{alg:unmerge}
  \begin{algorithmic}
    \STATE {\bfseries Input:} Finetuned model $\Wft$, initial weights $\Winit$, forget set $\setF$, retain set $\setR$, rank $k$, regularization $\lambda, \gamma$, scaling $\alpha$, learning rate $\eta$, steps $T$, basis type $\in\{\text{PCA},\text{CAV}\}$, CAV mode $\in\{\text{kmeans},\text{class}\}$ (if basis$=$CAV)
    \STATE {\bfseries Output:} Unlearned model $\Wul$
    \STATE
    \FOR{each target layer $l$}
      \STATE $\taum^l \leftarrow \Wft^l - \Winit^l$
      \STATE Collect layer inputs $\mX_{\setR}^l, \mX_{\setF}^l$ via forward pass
      \STATE $\QR^l \leftarrow \textsc{PcaBasis}(\mX_{\setR}^l, k)$
      \IF{basis type is PCA}
        \STATE $\QF^l \leftarrow \textsc{PcaBasis}(\mX_{\setF}^l, k)$
      \ELSE
        \STATE $\QF^l \leftarrow \textsc{CavBasis}(\mX_{\setF}^l, \mX_{\setR}^l, k, \text{mode})$
      \ENDIF
      \STATE $\mC^l \leftarrow {\QF^l}^\top \QR^l$ \COMMENT{$k \times k$ basis overlap}
    \ENDFOR
    \STATE
    \STATE $\mB^l \leftarrow \taum^l \QF^l$ for all $l$ \COMMENT{initialize at the projection, the minimizer of Term~1}
    \STATE \COMMENT{$T$ Adam steps in projected space: with $\tauF^l = \mB^l {\QF^l}^\top$, \eqref{eqn:obj} is the strictly convex quadratic below, whose exact minimizer is \eqref{eq:closed}; the fixed budget $T$ early-stops (App.~\ref{sec:closedform})}
    \FOR{$t=1$ {\bfseries to} $T$}
      \STATE $\mathcal{L} \leftarrow \sum_l \left\| \taum^l \QF^l - \mB^l \right\|_F^2 + \lambda \left\| \mB^l \mC^l \right\|_F^2 + \gamma \left\| \mB^l \right\|_F^2$
      \STATE $\mB^l \leftarrow \mathrm{Adam}\bigl(\mB^l, \nabla_{\mB^l} \mathcal{L}, \eta\bigr)$ for all $l$
    \ENDFOR
    \STATE
    \STATE \COMMENT{Recover $\tauF$ and apply}
    \FOR{each target layer $l$}
      \STATE $\tauF^l \leftarrow \mB^l {\QF^l}^\top$
      \STATE $\Wul^l \leftarrow \Wft^l - \alpha\, \tauF^l$
    \ENDFOR
    \STATE \textbf{return} $\Wul$
  \end{algorithmic}
\end{algorithm}

\begin{algorithm}[tb]
  \caption{PCA Basis Construction (\textsc{PcaBasis})}
  \label{alg:pca-basis}
  \begin{algorithmic}
    \STATE {\bfseries Input:} Layer inputs $\mX \in \R^{N \times d}$, rank $k$
    \STATE {\bfseries Output:} Orthonormal basis $Q \in \R^{d \times k}$
    \STATE
    \STATE $\vmu \leftarrow \frac{1}{N}\sum_{i=1}^{N} \vx_i$
    \STATE $\mC \leftarrow \mX^\top \mX - N\, \vmu\vmu^\top$ \COMMENT{centered covariance}
    \STATE $\vv_1, \ldots, \vv_k \leftarrow$ top-$k$ eigenvectors of $\mC$
    \STATE $Q \leftarrow [\vv_1, \ldots, \vv_k]$
    \STATE \textbf{return} $Q$
  \end{algorithmic}
\end{algorithm}

\begin{algorithm}[tb]
  \caption{CAV Basis Construction (\textsc{CavBasis})}
  \label{alg:cav-basis}
  \begin{algorithmic}
    \STATE {\bfseries Input:} Forget inputs $\mX_{\setF} \in \R^{N_\setF \times d}$ with class labels $\vy_{\setF}$, retain inputs $\mX_{\setR} \in \R^{N_\setR \times d}$, rank $k$, ridge penalty $\mu$, mode $\in\{\text{kmeans},\text{class}\}$
    \STATE {\bfseries Output:} Orthonormal basis $Q \in \R^{d \times k'}$ where $k' \leq k$ (kmeans) or $k' = |\setY_{\setF}|$ (class)
    \STATE
    \STATE Mean-center $\mX_{\setF}$ and $\mX_{\setR}$
    \IF{mode is \textsc{kmeans}}
      \STATE $\{C_1, \ldots, C_k\} \leftarrow$ $k$-means clustering on $\mX_{\setF}$
    \ELSE
      \STATE $\{C_c\}_{c\in\setY_{\setF}} \leftarrow$ partition $\mX_{\setF}$ by class label $\vy_{\setF}$ \COMMENT{one cluster per forget class}
    \ENDIF
    \STATE $Q \leftarrow []$
    \FOR{each cluster $C_j$}
      \STATE $\mX^+ \leftarrow \mX_{\setF}[C_j]$, \quad $\mX^- \leftarrow \mX_{\setR}$
      \STATE $\mZ \leftarrow [\mX^+; \mX^-]^\top$, \quad $\vy \leftarrow [+\vone; -\vone]$
      \STATE $\vw_j \leftarrow (\mZ\mZ^\top + \mu \mI)^{-1} \mZ\vy$ \COMMENT{ridge regression}
      \STATE Orthogonalize $\vw_j$ against columns of $Q$
      \IF{$\|\vw_j\| > \epsilon$}
        \STATE Append $\vw_j / \|\vw_j\|$ to $Q$
      \ENDIF
    \ENDFOR
    \STATE \textbf{return} $Q$
  \end{algorithmic}
\end{algorithm}

\subsection{Proofs and Scope of the Analysis}
\label{sec:proof}
\vspace{-5pt}
\textbf{Derivation of the closed form (\eqref{eq:closed}).} Write $\tauF=\mB\QF^\top$ with $\QF^\top \QF=\mI$ and $\QR^\top \QR=\mI$, and let $\mC=\QF^\top \QR$. Right-multiplication by a matrix with orthonormal rows preserves the Frobenius norm, so $\|(\taum-\mB\QF^\top)\QF\QF^\top\|_F=\|\taum\QF-\mB\|_F$, $\|\mB\QF^\top \QR\QR^\top\|_F=\|\mB\mC\|_F$, and $\|\mB\QF^\top\|_F=\|\mB\|_F$. \Eqref{eqn:obj} therefore equals $f(\mB)=\|\taum\QF-\mB\|_F^2+\lambda\|\mB\mC\|_F^2+\gamma\|\mB\|_F^2$, with gradient $\nabla f=-2(\taum\QF-\mB)+2\lambda\mB\mC\mC^\top+2\gamma\mB$. Setting it to zero gives $\mB\bigl((1+\gamma)\mI+\lambda\mC\mC^\top\bigr)=\taum\QF$. The matrix in parentheses is symmetric positive definite, so $f$ is strictly convex and its unique minimizer is $\mB^\star=\taum\QF\bigl((1+\gamma)\mI+\lambda\mC\mC^\top\bigr)^{-1}$. For CAV bases $\QF$ has $k'\le k$ orthonormal columns and the derivation is unchanged with $\mC\in\R^{k'\times k}$. The 100-step Adam iterate used for the reported runs does \emph{not} reach $\mB^\star$ in general; App.~\ref{sec:closedform} quantifies the distance and the effect of applying $\mB^\star$ itself.

\textbf{Proof of Lemma~\ref{lem:lowrank}.} Let $\mSigma_{\setF} = \frac{1}{N_\setF}\mX_{\setF}^\top\mX_{\setF}$ be the second-moment matrix of forget inputs with eigenvalues $\nu_1\geq\cdots\geq\nu_d\geq 0$, and let $\QF\in\R^{d\times k}$ be the top-$k$ eigenvectors with projector $\PF=\QF{\QF}^\top$. Under Assumption~\ref{ass:input-subspace}, any task vector $\vtau$ whose rows lie in $\mathrm{span}(\mX_{\setF}^\top)$ admits the decomposition $\vtau = \mB{\QF}^\top + \mE$ with $\mB=\vtau\QF\in\R^{m\times k}$ and
\begin{equation}
    \|\mE\|_F^2 \leq \|\mC\|_{\mathrm{op}}^2\cdot N_\setF\sum_{i=k+1}^{d}\nu_i.
\end{equation}

\begin{proof}
Write $\vtau=\mC\mX_{\setF}$ (neglecting $\mathcal{O}(\eta^2)$ terms). Decompose $\vtau=\underbrace{\mC\mX_{\setF} \QF}_{\mB}{\QF}^\top + \underbrace{\mC\mX_{\setF}(I-\PF)}_{\mE}$. By sub-multiplicativity, $\|\mE\|_F\leq\|\mC\|_{\mathrm{op}}\|\mX_{\setF}(I-\PF)\|_F$, and $\|\mX_{\setF}(I-\PF)\|_F^2=\mathrm{tr}\!\left((I-\PF)\mX_{\setF}^\top\mX_{\setF}(I-\PF)\right)=N_\setF\sum_{i>k}\nu_i$.
\end{proof}

\textbf{Proof of Proposition~\ref{prop:bounds}.} Let $\mSigma_{\setF}\in\R^{d\times d}$ be the second-moment matrix of forget inputs with eigenvalues $\nu_1\geq\cdots\geq\nu_d$ and top-$k$ projector $\PF$. Similarly, let $\mSigma_{\setR}$ have eigenvalues $\mu_1\geq\cdots\geq\mu_d$ with top-$k$ projector $\PR$. Let $\alpha>0$ be the scaling used in the update and denote $\tauR=\taum-\alpha\tauF$. Then:
\begin{align}
    \E_{x\sim\setF}\!\left[\|\tauR\,x\|^2\right] \leq& \nu_1\left\|(\taum-\alpha\tauF)\PF\right\|_F^2 + \nu_{k+1}\left\|\taum-\alpha\tauF\right\|_F^2, \label{eq:forget-bound-proof}\\
    \E_{x\sim\setR}\!\left[\|\alpha\tauF\,x\|^2\right] \leq& \alpha^2\Bigl(\mu_1\left\|\tauF\PR\right\|_F^2 + \mu_{k+1}\left\|\tauF\right\|_F^2\Bigr). 
    \label{eq:retain-bound-proof}
\end{align}

\begin{proof}
For any $\mA\in\R^{m\times d}$ and random vector $x$ with second-moment $\mSigma=\sum_i\sigma_iv_iv_i^\top$:
\begin{align}
    \E\!\left[\|\mA{}x\|^2\right] &= \mathrm{tr}(\mA\mSigma\mA^\top) = \sum_{i=1}^d \sigma_i\|\mA{}v_i\|^2 \\
    &= \underbrace{\sum_{i=1}^k\sigma_i\|\mA{}v_i\|^2}_{\leq\;\sigma_1\|\mA\mQ\|_F^2}+\underbrace{\sum_{i=k+1}^d\sigma_i\|\mA{}v_i\|^2}_{\leq\;\sigma_{k+1}\|\mA\|_F^2},
\end{align}
where $\mQ=[v_1,\ldots,v_k]$ and the tail bound uses $\sum_{i>k}\|\mA{}v_i\|^2\leq\|\mA\|_F^2$. Since $\mQ^\top\mQ=\mI_k$, we have $\|\mA\mQ\|_F=\|\mA\mQ\mQ^\top\|_F=\|\mA\mP\|_F$. Setting $\mA=\tauR=\taum-\alpha\tauF$ with $\mSigma_{\setF}$ gives \eqref{eq:forget-bound-proof}; setting $\mA=\alpha\tauF$ with $\mSigma_{\setR}$ and factoring $\alpha^2$ out of both norms gives \eqref{eq:retain-bound-proof}.
\end{proof}

\begin{remark}[Bases that are not eigenbases]
\label{rem:general-basis}
Proposition~\ref{prop:bounds} takes $\PF$ and $\PR$ to be top-$k$ eigenprojectors. The retain basis is always built by PCA, so \eqref{eq:retain-bound} applies to every variant. The CAV forget bases are orthonormal but are not eigenvectors of $\mSigma_{\setF}$; for any orthonormal $\QF$ with projector $\PF$, splitting $x=\PF x+(\mI-\PF)x$ and using $\|a+b\|^2\le 2\|a\|^2+2\|b\|^2$ gives
\begin{equation}
    \E_{x\sim\setF}\!\left[\|\tauR\,x\|^2\right] \leq 2\nu_1\left\|\tauR\PF\right\|_F^2 + 2\,\bigl\|(\mI-\PF)\mSigma_{\setF}(\mI-\PF)\bigr\|_{\mathrm{op}}\left\|\tauR\right\|_F^2,
\end{equation}
which reduces to \eqref{eq:forget-bound} without the factor $2$ when $\PF$ commutes with $\mSigma_{\setF}$. Term~1 therefore still controls the dominant component of forget leakage, and the tail constant $\nu_{k+1}$ is replaced by the largest forget variance left outside the span of $\QF$, which the captured-variance diagnostic of \S\ref{sec:instance} reports.
\end{remark}

\textbf{Proof of Corollary~\ref{cor:gain}.} Write the singular value decomposition $\mC=\mV\mSigma\mU^{\top}$ with $\mSigma=\mathrm{diag}(\sigma_i)$; since $\QF$ and $\QR$ are orthonormal, $\sigma_i\in[0,1]$ are the cosines of the principal angles between their spans. Then $\mC\mC^{\top}=\mV\mSigma^2\mV^{\top}$ and $\bigl((1+\gamma)\mI+\lambda\mC\mC^{\top}\bigr)^{-1}=\mV\,\mathrm{diag}(g_i)\,\mV^{\top}$, which gives the form of ${\mB}^{\star}$ from \eqref{eq:closed}. For the retain term, $\|{\mB}^{\star}\mC\|_F^2=\sum_i g_i^2\sigma_i^2\|\va_i\|^2$ with $\va_i=\taum\QF\vv_i$, against $\sum_i\sigma_i^2\|\va_i\|^2$ for the plain projection $\mB_0$. \qed

\textbf{Scope of the assumption and of the bounds.} \emph{Assumption~\ref{ass:input-subspace}.} Each SGD update of a linear (or unfolded convolutional) layer is an outer product between a backpropagated signal and that layer's inputs, so the rows of the accumulated update lie in the span of the inputs seen during training. The error term accounts for the drift of those inputs while earlier layers change, and its $\mathcal{O}(\eta^2)$ rate is a small-learning-rate idealization suited to finetuning from $\Wpt$. For from-scratch training with $\Winit{=}0$ we use the property heuristically: weight decay discounts the early updates, whose inputs differ most from the final activations, and we check the resulting low-rank structure empirically (\S\ref{sec:perlayer}). \emph{What \eqref{eq:forget-bound} measures.} Its left-hand side is the response of $\tauR$ on forget inputs, a proxy for forgetting rather than a distance to $\Wrt$: a retrained model still responds to the features that forget inputs share with retain inputs, which is the part Term~2 protects. \emph{Gain.} The first term of \eqref{eq:forget-bound} is smallest at unit gain, $\alpha/(1+\gamma+\lambda\sigma_i^2)=1$, whereas the selected $\alpha$ exceeds $1+\gamma$ (e.g.\ $\alpha/(1+\gamma)\approx1.7$ for the CAV bases on $\Wpt$): the edit over-corrects the $k$ captured directions. We attribute this to forget signal left outside the rank-$k$ span (the tail term of \eqref{eq:forget-bound}), which a unit-gain edit cannot out-vote at the classifier; the bounds motivate the three terms and are not minimized by the tuned edit. \emph{Layers.} The bounds are per layer, with bases computed on the un-edited network, and do not track how upstream edits shift downstream activations; \texttt{skip-layers} confines this interaction to the last $8$--$18$ layers, and we did not test sequential re-estimation of the bases.

\subsection{Requirements, Runtime and Cost}
\label{sec:requirements}
\label{sec:cost}
\vspace{-5pt}
\textbf{Requirements.} Tab.~\ref{tab:requirements} makes precise what each method consumes at unlearning time. No method in our comparison is data-free in the sense of touching no data; \textsc{Unmerge} is \emph{optimization-time} data-free: $\setF$ and $\setR$ are read in forward passes only (one for PCA, two over $\setR$ for CAV), to form second-moment statistics, and no gradient is ever taken with respect to a training sample. The retain statistics can be cached and reused across forget requests on the same checkpoint. Every method uses the retrained model for evaluation only.

\begin{table}[t]
\centering
\caption{Resources consumed at unlearning time (CIFAR-100 Trees, ResNet-50). $\Winit$: whether the pre-finetuning checkpoint is used ($\Winit{=}0$ is the convention in the $\Wzero$ regime, App.~\ref{sec:winit}). Runtimes: $\Wpt$ cells of Tab.~\ref{tab:cifar100}.}
\label{tab:requirements}
\vspace{-5pt}
\begin{adjustbox}{max width=\linewidth}
\begin{tabular}{l|cccccc}
\toprule
Method & Forget data & Retain data & Forget labels & Gradient updates & $\Winit$ & Runtime \\
\midrule
NegGrad+ & yes & yes & yes & 10 epochs (ascent on $\setF$, descent on $\setR$) & no & 24s \\
Random Label & yes & yes & replaced by random & 10 epochs & no & 110s \\
SalUn & yes & yes & replaced by random & 10 epochs + saliency pass & no & 111s \\
NegTV & yes & no & yes & 5 epochs on $\setF$ & no & 4s \\
NegMerge & yes & no & yes & $3\times$ 5 epochs on $\setF$ & no & 11s \\
Kodge et al. & 1 forward pass ($5$k images) & 1 forward pass ($5$k images) & no & none & no & 8s \\
\textsc{Unmerge-PCA} / \textsc{CAV-KMeans} & 1 forward pass & 1 forward pass & no & none & $\Wpt$ regime only & 8s / 26s \\
\textsc{Unmerge-CAV-Class} & 1 forward pass & 1 forward pass & yes & none & $\Wpt$ regime only & 22s \\
\bottomrule
\end{tabular}
\end{adjustbox}
\end{table}

\textbf{Runtime and cost model.} Tab.~\ref{tab:phases} times the three phases of Alg.~\ref{alg:unmerge} on the official Trees cells. The pipeline is data-bound: more than $95\%$ of the wall-clock is the single forward pass over $\setF\cup\setR$ plus basis construction, the optimization of $\{\mB^l\}$ takes under one second. Per layer, the costs are $\mathcal{O}(N d_l)$ to collect activations, $\mathcal{O}(N d_l^2)$ to accumulate the second-moment matrix, $\mathcal{O}(d_l^3)$ for its eigendecomposition (only over non-skipped layers; a randomized top-$k$ solver reduces this to $\mathcal{O}(d_l^2 k)$), and $\mathcal{O}(Tmk^2)$ for the $T$ Adam steps ($\mathcal{O}(mk^2+k^3)$ for the exact solve of \eqref{eq:closed}), so neither solver affects how the method scales. A gradient baseline that trains for $E$ epochs costs $E$ forward and backward passes over the same data, so the runtime ratio is governed by $E$ rather than by model size as long as the per-layer eigendecomposition stays cheaper than one forward pass over $N$ samples---true by orders of magnitude for $d_l$ up to $\sim\!10^4$. The measured ratios against Random Label and SalUn ($4$--$23\times$ on ResNet-50 and ViT-S/16 at $E{=}10$) vary with how much of \textsc{Unmerge}'s time goes to basis construction, which is larger for the CAV variants. One detail of the timing should be kept in mind: the baselines' training loops evaluate on the test set after every epoch to log progress, and their timers include it ($0.72$s per epoch on CIFAR-100, $7.2$s over $E{=}10$ epochs), whereas \textsc{Unmerge}'s timer stops before evaluation. This is under $7\%$ of Random Label's and SalUn's runtime but about a third of NegGrad+'s, whose epochs are short because they only traverse $\setF$ with as many retain samples ($20$ steps). Net of that logging pass NegGrad+ takes $16.7$s instead of $23.9$s on Trees ($\Wpt$), so \textsc{Unmerge-PCA} ($8.6$s) remains about $2\times$ faster than NegGrad+ and the CAV variants ($18$--$26$s) are on par with it; the $\sim$$5\times$ advantage over Random Label and SalUn is unaffected. Beyond the preliminary 3B study of App.~\ref{sec:llm}, language-model scale is scoped as future work.

\begin{table}[t]
\centering
\caption{Phase-wise runtime of \textsc{Unmerge} on CIFAR-100 (Trees), ResNet-50, single H100.}
\label{tab:phases}
\vspace{-5pt}
\begin{adjustbox}{max width=0.85\linewidth}
\begin{tabular}{l|ccc|c}
\toprule
Cell & Phase 1 (activations + bases + $\taum$) & Phase 2 (100 Adam steps) & Phase 3 (apply) & Total \\
\midrule
$\Wzero$, CAV-Class & 17.1s & 0.5s & 0.0s & 17.6s \\
$\Wpt$, PCA ($k{=}64$) & 8.2s & 0.35s & 0.0s & 8.6s \\
\bottomrule
\end{tabular}
\end{adjustbox}
\end{table}

\section{Experimental Details}
\label{sec:expdetails}
\vspace{-5pt}
\textbf{Forget tasks.} CIFAR-100 superclasses (five classes, $|\setF|{=}2500$): Trees \texttt{[47,52,56,59,96]}, Aquatic Mammals \texttt{[4,30,55,72,95]}, Vehicles-1 \texttt{[8,13,48,58,90]}. TinyImageNet semantic groups (ten classes): Dogs \& Cats \texttt{[24--33]} and Arthropods \texttt{[7,9,10,36,38,39,41,42,43,45]}. Within each task $\setF$ is the union of the chosen classes' training images and $\setR=\setS\setminus\setF$; the instance-level protocols are described in App.~\ref{sec:instance-supp}.

\textbf{Training.} All models are trained with SGD (momentum $0.9$, cosine schedule, batch size $128$). ResNet-50 on CIFAR-100: $100$ epochs at learning rate $0.1$ from $\Wzero$ and $50$ epochs at $0.01$ from $\Wpt$ (torchvision \texttt{IMAGENET1K\_V1}; hash in App.~\ref{sec:winit}), weight decay $5\!\times\!10^{-4}$; TinyImageNet uses the same schedules with weight decay $10^{-4}$; ViT-S/16 is finetuned for $10$ epochs at $0.01$. Each retrained reference $\Wrt$ uses the schedule of its $\Wft$ on $\setR$ only. All experiments run on a single NVIDIA H100.

\textbf{Baselines.} All gradient baselines run for $E{=}10$ epochs. CIFAR-100: NegGrad+ ($\alpha{=}0.5$, learning rate $10^{-4}$), Random Label (learning rate $0.01$), SalUn (saliency threshold $0.5$, learning rate $0.01$--$0.02$). TinyImageNet: NegGrad+ ($\alpha\in\{0.5,0.8\}$, learning rate $10^{-4}$), Random Label and SalUn at learning rate $10^{-3}$. The task-vector baselines subtract a vector obtained by gradient training on $\setF$: NegTV~\citep{ilharco2023editing} finetunes $\Wft$ on $\setF$ for $5$ epochs to $\mW_{\setF}$ and applies $\Wul=\Wft-\alpha(\mW_{\setF}-\Wft)$; NegMerge~\citep{kim2025negmerge} obtains $K{=}3$ such vectors at $\{0.5,1,2\}\times$ the base learning rate, keeps the entries whose sign agrees across all $K$, averages them, and subtracts. Both were given an equal-budget grid on Trees (learning rate $\{10^{-3},10^{-2}\}$, $\alpha\in[0.25,4]$, $28$ runs plus an $\alpha$ refinement), and the Trees configuration (learning rate $10^{-3}$, $5$ epochs) is transferred to the other tasks with only $\alpha$ re-swept over $\{0.25,0.35,0.5,0.75,1.0\}$; every reported optimum is interior. Their failure mode is retain collapse: the finetuned vector overlaps retain directions, so the best cells already lose $1$--$13$ points of retain accuracy or leave up to $26\%$ forget accuracy and every stronger $\alpha$ pushes retain toward chance, and the two methods are within $0.04$ ToW of each other in every cell. The activation-subspace baseline of \citet{kodge2024deep} is run as published, on every convolutional and linear layer, from the same forward-pass statistics as \textsc{Unmerge} ($5000$ images per set): its retain and forget scalings are tuned on Trees and Dogs \& Cats by ToW over $\alpha_r\in\{100,300,1000,3000\}$ and $\alpha_f\in\{3,10,30,100\}$, and $\alpha_f$ is re-swept on the transfer tasks; the grids were extended at the edges until every reported optimum is interior.

\textbf{\textsc{Unmerge} configurations.} Phase~2 always uses $T{=}100$ Adam steps at learning rate $10^{-3}$, never tuned. Tab.~\ref{tab:configs} lists the Trees configurations. They are tuned once per (variant, init) on Trees (CIFAR-100) and Dogs \& Cats (TinyImageNet) and transferred to the other tasks changing at most $\alpha$ (by $\le\!0.5$; \textsc{PCA} on $\Wpt$ transfers with $\alpha$ unchanged) and, in three $\Wzero$ cells, \texttt{skip-layers} by one step of five layers. The three metrics are defined in \S\ref{sec:setup}; attack details are in App.~\ref{sec:mia-supp}. The entanglement distances of \S\ref{sec:entangle} are computed on penultimate-layer features (the classifier input) of $\setF$ and a same-size retain subsample ($|\setR|{=}|\setF|$): the unbiased Gaussian-RBF MMD$^2$ at the median bandwidth $\sigma_{\mathrm{med}}$ and averaged over $\sigma\!\in\!\{0.5,1,2\}\!\cdot\!\sigma_{\mathrm{med}}$ (multi), the closed-form Bures--Wasserstein W$_2^2$ between Gaussian fits (mean-shift dominated when the feature dimension exceeds the sample size), and the Sliced-W$_2^2$ over $1000$ random one-dimensional projections.

\begin{table}[t]
\centering
\caption{\textsc{Unmerge} configurations on CIFAR-100 (Trees), ResNet-50.}
\label{tab:configs}
\vspace{-5pt}
\begin{adjustbox}{max width=0.8\linewidth}
\begin{tabular}{l|ccccc}
\toprule
Cell & rank $k$ & \texttt{skip-layers} & $\alpha$ & $\lambda$ & $\gamma$ \\
\midrule
$\Wpt$, PCA & 64 & 45 & 1.6 & 5 & 0.5 \\
$\Wpt$, CAV-KMeans & 16 & 35 & 2.5 & 5 & 0.5 \\
$\Wpt$, CAV-Class & 16 & 35 & 2.5 & 5 & 0.5 \\
$\Wzero$, PCA & 128 & 45 & 1.25 & 50 & 0.5 \\
$\Wzero$, CAV-KMeans & 16 & 40 & 1.0 & 50 & 0.5 \\
$\Wzero$, CAV-Class & 16 & 40 & 1.3 & 50 & 0.5 \\
\bottomrule
\end{tabular}
\end{adjustbox}
\end{table}

\section{Ablations of \textsc{Unmerge}}
\label{sec:ablations-all}
\vspace{-5pt}
\subsection{Per-Layer Diagnostics and Hyperparameters}
\label{sec:ablation}
\begin{figure}[t]
  \centering
  \begin{subfigure}{0.49\linewidth}
    \includegraphics[width=\linewidth]{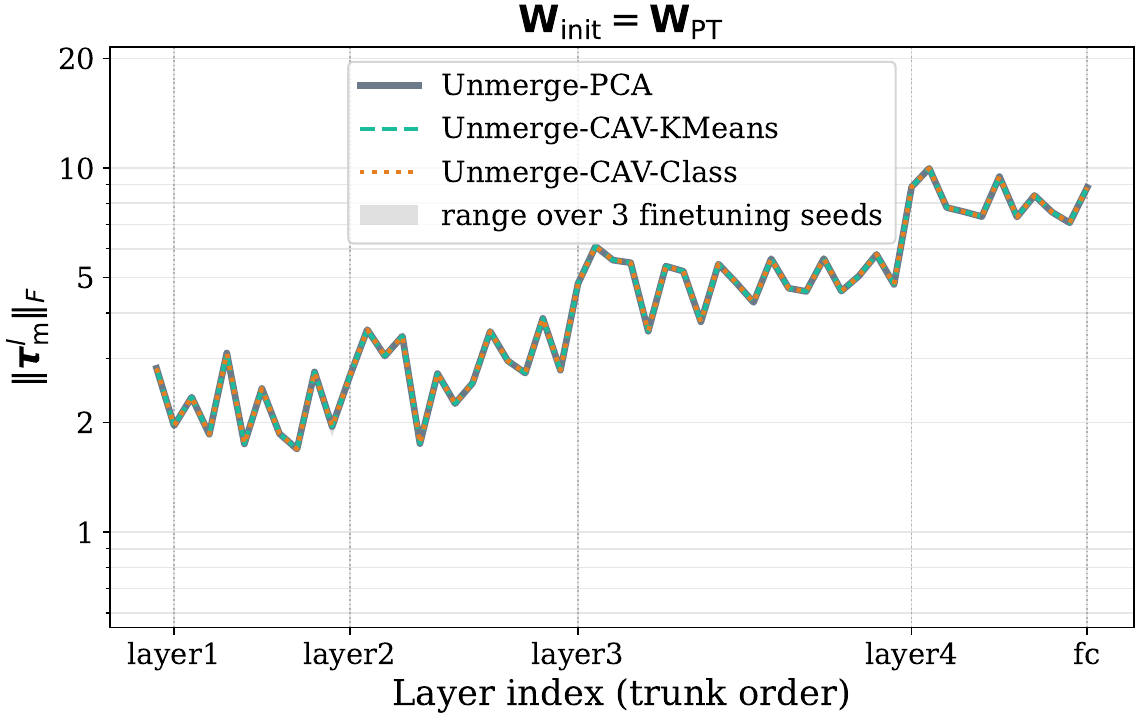}
    \caption{$\Winit=\Wpt$}
  \end{subfigure}
  \begin{subfigure}{0.49\linewidth}
    \includegraphics[width=\linewidth]{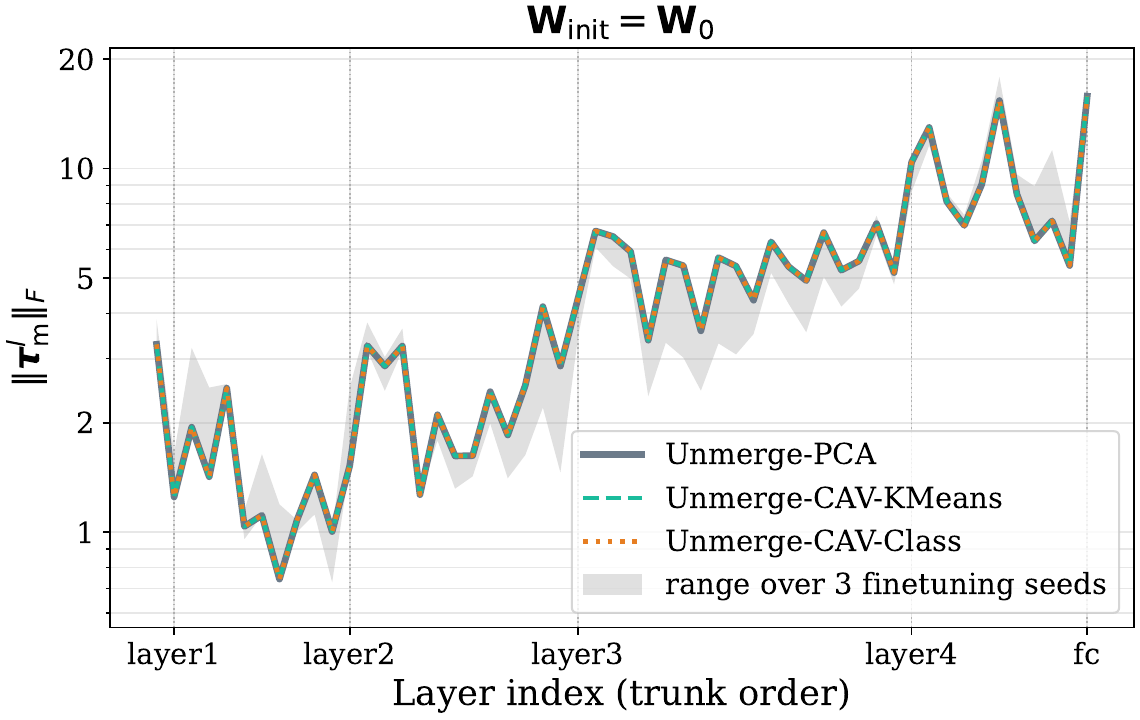}
    \caption{$\Winit=\Wzero$}
  \end{subfigure}
  \vspace{-5pt}
  \caption{Per-layer merged-task-vector magnitude $\|\taum^l\|_F$ on CIFAR-100 (Trees), ResNet-50, log-$y$. The three variant curves coincide exactly because $\taum$ is basis-independent; the shaded band is the range over three finetuning seeds. The cross-init gap is depth-dependent: shallow blocks (layer1--2) carry slightly larger magnitudes on $\Wpt$ (pretrained features must be shifted away from ImageNet), while at layer4 and the classifier $\|\taum\|$ is $\sim$$1.5$--$2\times$ larger on $\Wzero$, where the deepest task-specific weights are encoded from scratch.}
  \label{fig:taum}
\end{figure}
\vspace{-5pt}
This subsection covers the design choices that determine \textsc{Unmerge} performance---basis construction (PCA vs.\ CAV), rank $k$, the scaling $\alpha$, \texttt{skip-layers}, and the regularizers $\lambda,\gamma$---and how $\alpha$ is selected without the retrained model; the objective terms and the solver are ablated in App.~\ref{sec:terms} and~\ref{sec:closedform}. The per-layer figures of \S\ref{sec:perlayer} (Figs.~\ref{fig:crossfr},~\ref{fig:taum}) drive the recipe; the paragraphs below summarize what each knob controls and the ranges that worked across all reported runs.

\textbf{Recipe in brief.} The two diagnostics read off the hyperparameters used throughout: \texttt{skip-layers} at the overlap elbow (layer3$\to$layer4; CAV $35$/$40$, PCA $45$--$50$), $\alpha$ inversely to the deepest-layer $\|\taum\|$ ($[1.3,2.5]$ on $\Wpt$, $[0.5,1.3]$ on $\Wzero$), and CAV where entanglement is the binding constraint ($\Wzero$). They also explain why \textsc{Unmerge-PCA} wins ToW on $\Wpt$ yet loses on $\Wzero$ at the same entanglement (next paragraph).

\textbf{PCA vs.\ CAV.} PCA gives a strong, unsupervised baseline but its top-$k$ eigenvectors absorb dataset-generic variance that overlaps with retain---visible as the highest per-direction overlap at every depth in Fig.~\ref{fig:crossfr} ($0.54$--$0.58$ at layer4 and $0.46$--$0.56$ at the classifier, all bases at rank $16$). CAV's discriminative ridge probes target only forget-distinguishing directions ($0.40$--$0.46$ at layer4, $0.08$--$0.20$ at the classifier); the two CAV modes are essentially tied per direction, so the gap between them in the main tables comes from the number of directions ($|\setY_{\setF}|$ vs.\ $k$), not from entanglement. The empirical consequence is the regime-split observed in Tab.~\ref{tab:cifar100-summary}, which the two diagnostics of \S\ref{sec:perlayer} explain jointly: by \eqref{eq:retain-bound}, retain damage scales as $\alpha^2\mu_1\|\QF^\top\QR\|_F^2$ times the basis mass $\|\mB\|^2$, so entanglement and $\|\taum\|$ enter the retain budget together, and since the entanglement of each basis is nearly the same in both regimes it is the deep $\|\taum\|$ that decides. PCA wins ToW on $\Wpt$ ($0.9851$), where moderate deep $\|\taum\|$ tolerates the higher overlap at larger $\alpha$; on $\Wzero$, deep $\|\taum\|$ peaks at $\sim$$15$ and forces a smaller $\alpha$ to protect retain, which under-drives PCA ($0.9230$), while CAV-Class's less entangled basis ($0.45$ at layer4, $0.17$ at the classifier) lets $\alpha$ climb back and forget cleanly ($0.9909$). All three bases are nearly init-invariant per direction ($\le\!0.13$ change in any block between $\Wpt$ and $\Wzero$), and CAV-Class remains the safer default when the init is unknown.

\textbf{Rank $k$.} Per Lemma~\ref{lem:lowrank}, the residual $\|\mE\|_F^2$ is bounded by $N_{\setF}\sum_{i>k}\nu_i$, so accuracy saturates once $k$ covers the effective dimensionality of the per-layer forget covariance. We sweep $k\!\in\!\{8,16,32,64,128\}$; for CAV bases the elbow lies near $k\!=\!16$ on ResNet-50 and $k\!=\!16$--$32$ on ViT-S, beyond which retain accuracy degrades by $0.5$--$2\%$ at fixed $\alpha$ as the rank-$k$ form picks up more retain-aligned directions. PCA prefers larger ranks ($k\!=\!64$--$128$): its variance-ordered eigenvectors need more directions to cover the forget signal, and on $\Wzero$---where deep task vectors are large---$k\!=\!128$ captures more forget energy per unit $\alpha$ and clearly pays off. CAV-Class is invariant to $k$ by construction (one direction per forget class, $k'\!=\!|\setY_{\setF}|$), which removes this knob in the supervised regime.

\textbf{Scale $\alpha$.} $\alpha\!\in\![0.5,2.5]$ tracks the deepest-layer $\|\taum\|$ (Fig.~\ref{fig:taum}) \emph{inversely}: $\Wpt$ keeps $\alpha\!\in\![1.3,2.5]$ where deep $\|\taum\|\!\le\!10$, $\Wzero$ uses $\alpha\!\in\![0.5,1.3]$ to cap retain damage when deep $\|\taum\|\!\approx\!15$. The product $\alpha\|\taum\|$ is the operative quantity in \eqref{eq:retain-bound}.

\textbf{Skip-layers.} Restricting unmerging to deeper layers (\texttt{skip-layers}$\,=\,35$--$50$ of the $53$ convolutions of ResNet-50; the classifier is always edited: $35$ edits the last $18$ convolutions, from \texttt{layer3.3.conv2} on, $40$ the last $13$, from \texttt{layer3.5.conv1} on, $45$ the last $8$, all in layer4, and $50$ the last $3$) preserves the forget drop while \emph{improving} retain/test by $2$--$5$ points relative to all-layers updates. The cutoff is read directly off the overlap elbow in Fig.~\ref{fig:crossfr} (layer3$\to$layer4 transition) and is pushed deeper the more entangled the basis: CAV uses $35$/$40$, while PCA---whose variance-driven eigenvectors overlap retain most---works best at $45$--$50$. Shallower blocks have larger retain-damage constant $\mu_1$ and more entangled bases, so their unmerging budget costs more retain than it buys in forget.

\textbf{Regularizers $\lambda$, $\gamma$.} $\lambda\!\in\!\{5,20,50\}$ scales the retain-leakage penalty (Term~2 of \eqref{eqn:obj}), chosen per initialization regime: $\lambda\!=\!5$--$20$ on $\Wpt$ and $\lambda\!=\!50$ on $\Wzero$, where larger deep task vectors demand stronger retain protection. The penalty is most binding on PCA bases, whose overlap with retain is highest; CAV-Class runs are insensitive to $\lambda$ across an order of magnitude because the basis itself already separates forget from retain. $\gamma\!\in\!\{0.25,0.5\}$ acts as a uniform magnitude prior on $\mB^l$ (shrinking the closed-form solution of Term~1 by $1/(1{+}\gamma)$) and is most helpful on $\Wzero$ where $\tauF^l$ would otherwise over-expand along tail directions of the forget covariance. Both regularizers transfer across superclasses without retuning (\S\ref{sec:setup}), which we view as part of the unlearning quality of \textsc{Unmerge}.

\textbf{Selecting $\alpha$ without the retrained model.} All configurations in this paper, ours and the baselines', are selected by ToW, which uses $\Wrt$ and the test set, as is common practice~\citep{fan2024salun,kurmanji2023towards}. For a class-level request a practitioner has neither, but knows the target forget accuracy (zero). Because $\mB^l$ does not depend on $\alpha$, the update of \eqref{eq:update} can be line-searched after a single solve, at the cost of one forward pass over $\setF$ per candidate. We swept $\alpha$ on a uniform grid for the six Trees configurations. Forget and retain accuracy are monotone in $\alpha$, so we test the rule \emph{take the smallest $\alpha$ whose forget-set accuracy is at most $1\%$}. It selects the ToW-optimal grid point in five of the six cells and loses $0.003$ ToW in the sixth; the training-only score $\mathrm{Acc}_{\setR}(1-\mathrm{Acc}_{\setF})$ of \citet{kodge2024deep} loses at most $0.006$. On the coarser grids of the transfer superclasses the rule selects the ToW-optimal cell or loses at most $0.012$ (nine cells). On TinyImageNet/$\Wzero$ it applies only to \textsc{CAV-Class} (loss $\le\!0.007$): for \textsc{CAV-KMeans} no $\alpha$ on our grid reaches $1\%$ forget accuracy, and for \textsc{PCA} on Arthropods the rule over-shoots ($-0.15$ ToW), so the hardest regime still needs a reference. Rank, \texttt{skip-layers}, $\lambda$ and $\gamma$ are transferred from the tuning task unchanged (App.~\ref{sec:expdetails}).

\subsection{Objective Terms}
\label{sec:terms}
\vspace{-5pt}
Tab.~\ref{tab:term-ablation} removes the terms of \eqref{eqn:obj} one at a time with the reported estimator and re-sweeps $\alpha$ for every variant. With Term~1 alone the objective is minimized at the initialization, so the edit is $\alpha$ times the projection $\taum\QF$: it collapses retain for the PCA basis (ToW $0.46$--$0.48$, retain $71$--$94\%$) and trails the full objective by $0.04$--$0.35$ for the CAV bases, whose directions overlap retain less. Term~2 accounts for essentially all of the gain: with Terms~1 and~2 every cell is within $0.02$ of the full objective, and above it in one. Term~3 is a uniform shrinkage that $\alpha$ already provides: on its own it closes at most part of the gap (Terms~1+3: $0.50$--$0.96$, below Terms~1+2 in every cell), and at the reported $\lambda$ its effect is within $0.02$ everywhere. The optimization of Phase~2 is therefore meaningful precisely as the solver of the Term-2-regularized problem---the iterate ends at relative distance $0.7$--$1.9$ from its initialization (App.~\ref{sec:closedform})---and not because the initialization is already a good edit. The task-vector baselines corroborate this from outside the method: NegTV and NegMerge subtract a vector finetuned on $\setF$ with no retain penalty, and they lose retain accuracy in the same way as Term~1 alone (App.~\ref{sec:expdetails}).

\begin{table}[t]
\centering
\caption{Term ablation on CIFAR-100 (Trees) with the reported estimator ($100$ Adam steps): best ToW over an $\alpha$ sweep for each variant, with the selected $\alpha$ in parentheses. Term~1 only: $\lambda{=}\gamma{=}0$ (the edit is $\alpha\,\taum\QF$). Terms~1+3: $\lambda{=}0$, $\gamma{=}0.5$. Terms~1+2: $\lambda$ as reported, $\gamma{=}0$. All three: $\lambda$ and $\gamma$ as reported. Every optimum is interior to its grid.}
\label{tab:term-ablation}
\vspace{-5pt}
\begin{adjustbox}{max width=0.9\linewidth}
\begin{tabular}{l|cccc}
\toprule
Cell & Term 1 only & Terms 1+3 & Terms 1+2 & All three \\
\midrule
$\Wzero$, PCA & 0.4622 (0.5) & 0.4978 (0.75) & 0.9608 (1.25) & \best{0.9609 (1.25)} \\
$\Wzero$, CAV-KMeans & 0.9419 (0.3) & 0.9433 (0.4) & 0.9819 (1) & \best{0.9834 (1)} \\
$\Wzero$, CAV-Class & 0.9530 (0.5) & 0.9553 (0.6) & 0.9863 (1.5) & \best{0.9895 (1.4)} \\
$\Wpt$, PCA & 0.4802 (1) & 0.8174 (1.5) & 0.9843 (1.6) & \best{0.9868 (1.6)} \\
$\Wpt$, CAV-KMeans & 0.6162 (1.5) & 0.6897 (1.875) & \best{0.9775 (1.6)} & 0.9699 (2.4) \\
$\Wpt$, CAV-Class & 0.8089 (1.5) & 0.8600 (2.25) & 0.9294 (2) & \best{0.9459 (2.4)} \\
\bottomrule
\end{tabular}
\end{adjustbox}
\end{table}

\subsection{The Budgeted Solver and the Exact Minimizer}
\label{sec:closedform}
\vspace{-5pt}
Phase~2 of Alg.~\ref{alg:unmerge} moves $\mB^l$ from the projection $\mB^l_0=\taum^l\QF^l$ toward the exact minimizer ${\mB^l}^\star$ of \eqref{eq:closed} under a fixed budget ($T{=}100$ Adam steps, learning rate $10^{-3}$). This subsection characterizes the estimator that results, compares it with the exact minimizer, and states what depends on the budget, on the six official Trees cells. (Implementation note: the code averages each loss term over its entries rather than summing, so its exact minimizer uses $\lambda_{\text{eff}}=\lambda\,k_F/k_R$; this differs from $\lambda$ only for CAV-Class, where $k_F=|\setY_{\setF}|=5$ and $k_R=16$.)

\textbf{What the budgeted iterate computes.} Adam's normalized updates move every entry of $\mB^l$ by at most the step budget $\delta=\mathrm{lr}\times T=0.1$, several times the typical entry of the projection (rms $0.012$--$0.025$). Entries closer than $\delta$ to the minimizer therefore converge, and the others move by about $\delta$, so the iterate should be close to
\begin{equation}
\label{eq:clip}
    \mB^l_{\delta}={\mB^l}^{\star}+S_{\delta}\bigl(\mB^l_0-{\mB^l}^{\star}\bigr),
\end{equation}
with $S_\delta$ the entrywise soft-threshold: the exact minimizer plus a sparse residual of the largest projection coefficients, interpolating between the projection ($\delta{=}0$) and the minimizer ($\delta{\to}\infty$). Tab.~\ref{tab:clipmodel} tests this model, which has no free parameter. Entrywise, $\mB_\delta$ is the closest of the three candidates to the Adam iterate in every cell (relative distance $0.01$--$0.32$, against $0.05$--$0.96$ for $\mB^\star$ and $0.74$--$1.85$ for $\mB_0$), and applying $\mB_{0.1}$ in place of the iterate, at the reported hyperparameters, reproduces the reported ToW within $0.005$ in five of the six cells ($0.023$ in the sixth). The model also explains the split between bases. For the CAV bases the residual is nearly empty (at most $0.15\%$ of the entries survive the threshold) and ToW moves by at most $0.02$ over $\delta\in[0.05,0.1]$: the solver hardly matters, which is why the exact minimizer can replace the iterate there. For PCA the residual carries the edit. On $\Wzero$ the minimizer is a no-op at the reported $\lambda$ and $0.36\%$ of the coefficients do the forgetting; this is a selection over (output unit, direction) pairs---the iterate places $37\%$ of its energy in the top $5\%$ of output units, against $22\%$ for $\mB^\star$---that no ridge solution, which filters all output units identically, can express. The price is a dependence on the step budget that the CAV variants do not have (last columns of Tab.~\ref{tab:clipmodel}). All our experiments use one budget ($T{=}100$, learning rate $10^{-3}$), never tuned per cell.

\begin{table}[t]
\centering
\caption{The early-stopped Adam iterate vs.\ the closed-form model $\mB_\delta$ of \eqref{eq:clip} on the six official CIFAR-100 (Trees) cells, at the reported $(\lambda,\gamma,\alpha)$, seed~1. Rel.\ dist.: distance of the Adam iterate to the projection $\mB_0$, to the minimizer $\mB^\star$, and to $\mB_\delta$ at the best $\delta\in\{0.02,\dots,0.12\}$, relative to $\|\mB_{\text{Adam}}\|_F$ and aggregated over edited layers. Resid.: share of entries with $|\mB_0-\mB^\star|>0.1$. ToW: the reported cell, and $\mB_\delta$ applied in its place.}
\label{tab:clipmodel}
\vspace{-5pt}
\begin{adjustbox}{max width=\linewidth}
\begin{tabular}{l|ccc|c|c|ccc}
\toprule
 & \multicolumn{3}{c|}{Rel.\ dist.\ of the iterate to} & Resid. & ToW & \multicolumn{3}{c}{ToW of $\mB_\delta$} \\
Cell & $\mB_0$ & $\mB^\star$ & $\mB_\delta$ & (\%) & Adam & $\delta{=}0.1$ & $0.075$ & $0.05$ \\
\midrule
$\Wzero$, CAV-Class & 1.46 & 0.35 & 0.09 & 0.07 & 0.9851 & 0.9804 & 0.9864 & 0.9860 \\
$\Wzero$, CAV-KMeans & 1.64 & 0.49 & 0.32 & 0.15 & 0.9834 & 0.9604 & 0.9784 & 0.9795 \\
$\Wzero$, PCA ($k{=}128$) & 1.68 & 0.96 & 0.31 & 0.36 & 0.9609 & 0.9610 & 0.9088 & 0.5926 \\
$\Wpt$, CAV-Class & 0.74 & 0.07 & 0.01 & 0.01 & 0.9418 & 0.9380 & 0.9418 & 0.9429 \\
$\Wpt$, CAV-KMeans & 1.31 & 0.05 & 0.01 & 0.00 & 0.9652 & 0.9651 & 0.9654 & 0.9648 \\
$\Wpt$, PCA ($k{=}64$) & 1.85 & 0.45 & 0.16 & 0.10 & 0.9868 & 0.9867 & 0.9779 & 0.9452 \\
\bottomrule
\end{tabular}
\end{adjustbox}
\end{table}

\textbf{The exact minimizer, tuned for itself.} Hyperparameters do not transfer between the two estimators. At the $(\lambda,\gamma,\alpha)$ tuned for the budgeted iterate, ${\mB^l}^\star$ reproduces the CAV cells on $\Wpt$ (ToW $0.932$/$0.965$ vs.\ $0.942$/$0.965$; the iterate is within relative distance $0.05$--$0.07$ of it there), under-forgets the CAV cells on $\Wzero$ ($0.89$/$0.68$) and is a no-op for both PCA cells, because at that $\lambda$ the retain penalty shrinks $\tauF$ to $3$--$10\%$ of $\taum$ along PCA's heavily overlapping directions; the plain projection $\mB^l_0$ at the same $\alpha$ collapses retain instead ($1.1\%$/$14.6\%$). Tuned for itself over $\lambda\in\{0,\dots,200\}$, $\gamma\in\{0,0.5\}$ and $\alpha$ ($34$--$50$ runs per cell, all optima interior), the exact minimizer prefers $\gamma{=}0$ and a larger $\lambda$ ($10$--$100$) with a correspondingly larger $\alpha$---it scales the component of $\taum^l\QF^l$ along the $i$-th principal direction between the forget and retain spans by $1/(1+\gamma+\lambda\sigma_i^2)$, so a larger $\lambda$ shrinks only the shared directions---and is then within $0.01$ of the reported iterate or above it in five of six cells ($0.990$/$0.975$ for CAV-Class/KMeans on $\Wzero$; $0.999$/$0.996$/$0.991$ for CAV-Class/KMeans/PCA on $\Wpt$), while it fails for PCA on $\Wzero$ ($0.65$). The reported estimator run at those hyperparameters reproduces the closed form for the CAV bases ($0.999$, $0.995$, $0.991$ and $0.979$), so the gap to the reported CAV cells reflects the $\lambda$ range of the original Trees grid (up to $5$ on $\Wpt$) rather than the solver, and those cells are conservative; for PCA the two estimators remain different objects ($0.40$ and $0.0001$ at the closed form's optimum).

\textbf{Scope.} Both solvers are negligible in cost ($0.2$--$0.8$s for the iterate and $0.01$--$0.03$s for the solve, against $8$--$21$s for Phase~1), so the choice is about universality, not efficiency. We report the budgeted iterate throughout because, at one fixed budget, it is the only estimator that works in all six cells. For the CAV variants the exact minimizer is an equivalent, budget-free alternative, and \eqref{eq:clip} gives the iterate an explicit form in which the threshold $\delta$ replaces the step count; the PCA variants, in contrast, depend on the budget (last columns of Tab.~\ref{tab:clipmodel}), which is a limitation of those variants. Adopting an explicit form as the default changes the scale of $(\lambda,\alpha)$ and needs a multi-seed, multi-dataset validation that we leave to future work. All cells of this subsection are single-seed ablations on Trees.

\subsection{\texorpdfstring{$\Winit$}{W\_init} Ablation}
\label{sec:winit}
\vspace{-5pt}
\textsc{Unmerge} forms $\taum=\Wft-\Winit$, so we ask how much the pre-finetuning checkpoint matters (Tab.~\ref{tab:winit}; Trees, $\Wpt$ regime). Setting $\Winit{=}0$---our convention throughout the $\Wzero$ regime---makes $\taum$ larger (it now contains the pretrained weights), so the paper's $\alpha$ over-shoots; with $\alpha$ retuned on the same grid ($0.9$ for CAV and $1.2$ for PCA, against $2.5$ and $1.6$ at the true init) zero-init matches the CAV cells and trails only on PCA, whose variance basis otherwise absorbs the shared pretrained content that the true init subtracts. Substituting a \emph{different} pretrained checkpoint (torchvision \texttt{IMAGENET1K\_V2} in place of \texttt{V1}) is catastrophic for every variant: the subtracted vector then contains the difference between two unrelated pretrainings. The exact checkpoint is therefore not required---but a wrong one must not be used, and $\Winit{=}0$ is the safe fallback. The $\Wpt$ checkpoint is part of the method's input and we release it: torchvision ResNet-50 \texttt{IMAGENET1K\_V1} with CIFAR-adapted first convolution and classifier (SHA-256 prefix \texttt{2a54223ea63a790f}; full hash in the code release). The adapted layers are seeded, so the file, not the recipe, is the artifact.

\begin{table}[t]
\centering
\caption{$\Winit$ ablation on CIFAR-100 (Trees), $\Wpt$ regime: ToW per \textsc{Unmerge} variant (retrained reference F/R/T $0.00$/$99.98$/$80.67$).}
\label{tab:winit}
\vspace{-5pt}
\begin{adjustbox}{max width=0.8\linewidth}
\begin{tabular}{l|ccc}
\toprule
$\Winit$ & CAV-Class & CAV-KMeans & PCA \\
\midrule
True init (\texttt{V1}, paper cells) & 0.9418 & \best{0.9652} & \best{0.9868} \\
Zero, paper $\alpha$ & 0.7249 & 0.6036 & 0.8291 \\
Zero, $\alpha$ retuned & \best{0.9629} & 0.9639 & 0.9422 \\
Wrong checkpoint (\texttt{V2}) & 0.4742 & 0.0745 & 0.0628 \\
\bottomrule
\end{tabular}
\end{adjustbox}
\end{table}

\section{Additional Results}
\label{sec:results-supp}
\vspace{-5pt}
\subsection{Detailed Per-Task Results}
\label{sec:detail-tables}
\vspace{-5pt}
The summary in Tab.~\ref{tab:cifar100-summary} averages over three CIFAR-100 superclasses (Trees, Aquatic Mammals, Vehicles); we expose the per-superclass breakdown here. Tabs.~\ref{tab:cifar100}, \ref{tab:cifar100-aquatic}, and \ref{tab:cifar100-vehicle} report the three CIFAR-100 superclasses individually, and Tab.~\ref{tab:tinyimgnt-arthropods} reports the second TinyImageNet superclass (Arthropods). The hyperparameter set is \emph{tuned once per (method, init)} on Trees and Dogs \& Cats and \emph{transferred} to Aquatic Mammals, Vehicles, and Arthropods with at most a small change of $\alpha$ (App.~\ref{sec:expdetails}), so these tables also stress-test transferability. The pattern from the main paper is preserved: in every (superclass, init) cell at least one \textsc{Unmerge} variant leads every baseline on ToW, including the task-vector baselines and the projection of \citet{kodge2024deep}; on the CIFAR-100 superclasses at least two of three lead every gradient-based baseline, whereas on TinyImageNet/$\Wzero$ (Dogs \& Cats and Arthropods) only \textsc{Unmerge-CAV-Class} leads and PCA/KMeans trail SalUn. \textsc{Unmerge-CAV-Class} is consistently strongest on $\Wzero$; on $\Wpt$ the lead goes to \textsc{Unmerge-PCA} on the CIFAR-100 superclasses, to \textsc{Unmerge-CAV-KMeans} on Arthropods and to \textsc{Unmerge-CAV-Class} on Dogs \& Cats.

\begin{table}[t]
\centering
\caption{CIFAR-100 unlearning results (Trees). Rows: evaluation metric (Forget/Retain/Test accuracy, \%) and tug-of-war score (\eqref{eqn:tow}; higher is better). Columns: unlearning method. $\Wft$ and $\Wrt$ are the finetuned upper bound and retrained gold standard. Best non-retrained in \bestlegend{bold}.}
\label{tab:cifar100}
\vspace{-5pt}
\begin{subtable}{\linewidth}
\centering
\label{tab:cifar100-wpt}
\begin{adjustbox}{max width=\linewidth}
\begin{tabular}{l|c|cccccc|ccc}
\toprule
$\Winit= \Wpt$ & $\Wrt$ & NegGrad+ & Rand.\ Label & SalUn & NegTV & NegMerge & Kodge et al. & Unmerge-PCA & \shortstack[c]{Unmerge-CAV\\KMeans} & \shortstack[c]{Unmerge-CAV\\Class} \\
\midrule
Forget Acc $\downarrow$ & 0.0 & 0.36 & 12.24 & 4.32 & 1.72 & 2.60 & 0.16 & 0.08 & 0.00 & 0.12 \\
Retain Acc $\uparrow$ & 99.98 & 97.21 & 99.94 & 99.92 & 96.87 & 97.50 & 99.67 & 99.80 & 99.40 & 98.82 \\
Test Acc $\uparrow$ & 80.67 & 74.48 & 80.62 & 80.64 & 74.59 & 75.24 & 78.25 & 79.60 & 77.75 & 76.07 \\
ToW $\uparrow$ & 1 & 0.9089 & 0.8768 & 0.9560 & 0.8943 & 0.8983 & \textbf{0.9712} & \best{0.9868} & 0.9652 & 0.9418 \\
Runtime $\downarrow$ &  & 24s & 110s & 111s & \best{4s} & 11s & 8s & 8s & 26s & 22s \\
\bottomrule
\end{tabular}
\end{adjustbox}
\end{subtable}
\begin{subtable}{\linewidth}
\centering
\label{tab:cifar100-w0}
\begin{adjustbox}{max width=\linewidth}
\begin{tabular}{l|c|cccccc|ccc}
\toprule
$\Winit=\Wzero$ & $\Wrt$ & NegGrad+ & Rand.\ Label & SalUn & NegTV & NegMerge & Kodge et al. & Unmerge-PCA & \shortstack[c]{Unmerge-CAV\\KMeans} & \shortstack[c]{Unmerge-CAV\\Class} \\
\midrule
Forget Acc $\downarrow$ & 0.0 & 4.40 & 15.44 & 3.08 & 1.40 & 1.96 & 0.08 & 0.12 & 0.68 & 1.12 \\
Retain Acc $\uparrow$ & 99.97 & 96.04 & 99.96 & 99.94 & 95.60 & 96.50 & 99.01 & 98.75 & 99.83 & 99.87 \\
Test Acc $\uparrow$ & 72.92 & 66.05 & 74.00 & 74.14 & 66.19 & 66.91 & 70.05 & 70.31 & 72.07 & 72.65 \\
ToW $\uparrow$ & 1 & 0.8553 & 0.8364 & 0.9571 & 0.8795 & 0.8895 & 0.9613 & 0.9609 & \textbf{0.9834} & \best{0.9851} \\
Runtime $\downarrow$ &  & 24s & 111s & 115s & \best{4s} & 12s & 9s & 9s & 21s & 18s \\
\bottomrule
\end{tabular}
\end{adjustbox}
\end{subtable}
\end{table}

\begin{table}[t]
\centering
\caption{CIFAR-100 unlearning results (Aquatic Mammals). Best non-retrained in \bestlegend{bold}.}
\label{tab:cifar100-aquatic}
\vspace{-5pt}
\begin{subtable}{\linewidth}
\centering
\begin{adjustbox}{max width=\linewidth}
\begin{tabular}{l|c|cccccc|ccc}
\toprule
$\Winit= \Wpt$ & $\Wrt$ & NegGrad+ & Rand.\ Label & SalUn & NegTV & NegMerge & Kodge et al. & Unmerge-PCA & \shortstack[c]{Unmerge-CAV\\KMeans} & \shortstack[c]{Unmerge-CAV\\Class} \\
\midrule
Forget Acc $\downarrow$ & 0.0 & 0.56 & 3.36 & 5.52 & 4.68 & 7.24 & 0.44 & 0.20 & 0.20 & 1.40 \\
Retain Acc $\uparrow$ & 99.99 & 95.34 & 99.94 & 99.98 & 98.00 & 98.68 & 99.67 & 99.72 & 99.71 & 99.52 \\
Test Acc $\uparrow$ & 80.74 & 73.16 & 80.26 & 81.05 & 75.64 & 76.46 & 78.54 & 79.20 & 79.16 & 77.95 \\
ToW $\uparrow$ & 1 & 0.8763 & 0.9613 & 0.9417 & 0.8866 & 0.8763 & 0.9705 & \best{0.9799} & \textbf{0.9795} & 0.9540 \\
Runtime $\downarrow$ &  & 23s & 110s & 112s & \best{4s} & 11s & 8s & 9s & 27s & 22s \\
\bottomrule
\end{tabular}
\end{adjustbox}
\end{subtable}
\begin{subtable}{\linewidth}
\centering
\begin{adjustbox}{max width=\linewidth}
\begin{tabular}{l|c|cccccc|ccc}
\toprule
$\Winit=\Wzero$ & $\Wrt$ & NegGrad+ & Rand.\ Label & SalUn & NegTV & NegMerge & Kodge et al. & Unmerge-PCA & \shortstack[c]{Unmerge-CAV\\KMeans} & \shortstack[c]{Unmerge-CAV\\Class} \\
\midrule
Forget Acc $\downarrow$ & 0.0 & 5.72 & 3.40 & 4.32 & 5.72 & 6.32 & 0.72 & 0.72 & 0.56 & 0.16 \\
Retain Acc $\uparrow$ & 99.97 & 97.18 & 99.97 & 99.97 & 96.93 & 97.18 & 99.00 & 95.80 & 99.66 & 99.81 \\
Test Acc $\uparrow$ & 72.8 & 69.02 & 74.55 & 74.69 & 68.82 & 68.85 & 70.68 & 66.66 & 72.57 & 73.48 \\
ToW $\uparrow$ & 1 & 0.8819 & 0.9491 & 0.9387 & 0.8778 & 0.8747 & 0.9623 & 0.8929 & \textbf{0.9891} & \best{0.9900} \\
Runtime $\downarrow$ &  & 23s & 109s & 112s & \best{4s} & 11s & 8s & 6s & 21s & 18s \\
\bottomrule
\end{tabular}
\end{adjustbox}
\end{subtable}
\end{table}

\begin{table}[t]
\centering
\caption{CIFAR-100 unlearning results (Vehicles). Best non-retrained in \bestlegend{bold}.}
\label{tab:cifar100-vehicle}
\vspace{-5pt}
\begin{subtable}{\linewidth}
\centering
\begin{adjustbox}{max width=\linewidth}
\begin{tabular}{l|c|cccccc|ccc}
\toprule
$\Winit= \Wpt$ & $\Wrt$ & NegGrad+ & Rand.\ Label & SalUn & NegTV & NegMerge & Kodge et al. & Unmerge-PCA & \shortstack[c]{Unmerge-CAV\\KMeans} & \shortstack[c]{Unmerge-CAV\\Class} \\
\midrule
Forget Acc $\downarrow$ & 0.0 & 0.00 & 12.12 & 8.52 & 0.40 & 0.44 & 0.36 & 0.04 & 0.04 & 0.56 \\
Retain Acc $\uparrow$ & 99.97 & 97.38 & 99.94 & 99.91 & 90.43 & 91.31 & 98.82 & 99.89 & 99.78 & 99.32 \\
Test Acc $\uparrow$ & 80.15 & 73.82 & 80.05 & 79.87 & 68.21 & 68.85 & 76.07 & 79.14 & 78.09 & 76.68 \\
ToW $\uparrow$ & 1 & 0.9125 & 0.8777 & 0.9117 & 0.7934 & 0.8066 & 0.9448 & \best{0.9887} & \textbf{0.9771} & 0.9537 \\
Runtime $\downarrow$ &  & 23s & 110s & 112s & \best{4s} & 11s & 8s & 8s & 26s & 22s \\
\bottomrule
\end{tabular}
\end{adjustbox}
\end{subtable}
\begin{subtable}{\linewidth}
\centering
\begin{adjustbox}{max width=\linewidth}
\begin{tabular}{l|c|cccccc|ccc}
\toprule
$\Winit=\Wzero$ & $\Wrt$ & NegGrad+ & Rand.\ Label & SalUn & NegTV & NegMerge & Kodge et al. & Unmerge-PCA & \shortstack[c]{Unmerge-CAV\\KMeans} & \shortstack[c]{Unmerge-CAV\\Class} \\
\midrule
Forget Acc $\downarrow$ & 0.0 & 3.48 & 9.36 & 4.36 & 14.84 & 15.24 & 1.80 & 0.68 & 0.00 & 0.00 \\
Retain Acc $\uparrow$ & 99.96 & 97.46 & 99.96 & 99.93 & 94.66 & 94.94 & 97.04 & 96.59 & 99.62 & 99.86 \\
Test Acc $\uparrow$ & 71.7 & 67.56 & 73.29 & 73.70 & 65.15 & 65.09 & 66.83 & 67.07 & 70.74 & 71.85 \\
ToW $\uparrow$ & 1 & 0.9021 & 0.8920 & 0.9370 & 0.7536 & 0.7519 & 0.9069 & 0.9153 & \textbf{0.9870} & \best{0.9975} \\
Runtime $\downarrow$ &  & 23s & 110s & 111s & \best{4s} & 11s & 8s & 9s & 21s & 18s \\
\bottomrule
\end{tabular}
\end{adjustbox}
\end{subtable}
\end{table}

\begin{table}[t]
\centering
\caption{TinyImageNet unlearning results (Arthropods). Best non-retrained in \bestlegend{bold}.}
\label{tab:tinyimgnt-arthropods}
\vspace{-5pt}
\begin{subtable}{\linewidth}
\centering
\begin{adjustbox}{max width=\linewidth}
\begin{tabular}{l|c|cccccc|ccc}
\toprule
$\Winit= \Wpt$ & $\Wrt$ & NegGrad+ & Rand.\ Label & SalUn & NegTV & NegMerge & Kodge et al. & Unmerge-PCA & \shortstack[c]{Unmerge-CAV\\KMeans} & \shortstack[c]{Unmerge-CAV\\Class} \\
\midrule
Forget Acc $\downarrow$ & 0.0 & 4.42 & 15.12 & 14.70 & 1.30 & 2.04 & 23.68 & 1.84 & 0.00 & 0.00 \\
Retain Acc $\uparrow$ & 99.98 & 94.33 & 99.98 & 99.98 & 87.49 & 90.45 & 98.83 & 99.04 & 99.96 & 99.95 \\
Test Acc $\uparrow$ & 73.62 & 65.38 & 74.17 & 73.91 & 60.69 & 62.65 & 69.68 & 71.44 & 73.20 & 72.72 \\
ToW $\uparrow$ & 1 & 0.8275 & 0.8441 & 0.8505 & 0.7520 & 0.7890 & 0.7247 & 0.9512 & \best{0.9956} & \textbf{0.9907} \\
Runtime $\downarrow$ &  & 94s & 671s & 674s & \best{18s} & 53s & 22s & 71s & 104s & 100s \\
\bottomrule
\end{tabular}
\end{adjustbox}
\end{subtable}
\begin{subtable}{\linewidth}
\centering
\begin{adjustbox}{max width=\linewidth}
\begin{tabular}{l|c|cccccc|ccc}
\toprule
$\Winit=\Wzero$ & $\Wrt$ & NegGrad+ & Rand.\ Label & SalUn & NegTV & NegMerge & Kodge et al. & Unmerge-PCA & \shortstack[c]{Unmerge-CAV\\KMeans} & \shortstack[c]{Unmerge-CAV\\Class} \\
\midrule
Forget Acc $\downarrow$ & 0.0 & 0.70 & 10.02 & 9.56 & 23.22 & 25.86 & 12.42 & 4.12 & 9.30 & 0.12 \\
Retain Acc $\uparrow$ & 99.98 & 80.14 & 99.98 & 99.97 & 96.34 & 96.42 & 85.35 & 94.91 & 98.85 & 98.18 \\
Test Acc $\uparrow$ & 61.86 & 44.78 & 60.97 & 60.78 & 53.28 & 53.75 & 47.46 & 52.26 & 56.51 & 55.76 \\
ToW $\uparrow$ & 1 & 0.6600 & 0.8918 & \textbf{0.8946} & 0.6764 & 0.6570 & 0.6400 & 0.8228 & 0.8488 & \best{0.9210} \\
Runtime $\downarrow$ &  & 93s & 666s & 673s & \best{18s} & 53s & 22s & 38s & 85s & 61s \\
\bottomrule
\end{tabular}
\end{adjustbox}
\end{subtable}
\end{table}

\subsection{Seed Variance}
\label{sec:seeds}
\vspace{-5pt}
Tab.~\ref{tab:seeds} repeats the Trees experiments for three seeds with frozen configurations. The seed controls model initialization only: the data order and augmentation streams are fixed across seeds. In the $\Wzero$ regime it therefore changes the initialization of the finetuned and the retrained model; in the $\Wpt$ regime retraining from the fixed checkpoint is deterministic given the data order, so all seeds share one retrained reference and the variance reflects finetuning and method stochasticity ($\Winit$ is held fixed across seeds, including its seeded stem and classifier entries). Every ordering between the \textsc{Unmerge} variants and SalUn/Random Label in Tab.~\ref{tab:cifar100} is preserved in the means; NegGrad+ and Random Label swap on $\Wzero$ within noise. On $\Wzero$ the \textsc{Unmerge} standard deviations ($0.008$--$0.021$) are $1.7$--$4.5\times$ smaller than those of Random Label and SalUn ($0.036$--$0.037$), and \textsc{Unmerge-PCA}'s mean ($0.937$) sits above SalUn's ($0.926$) although the single-seed cells are close. On $\Wpt$ SalUn is the most stable method ($0.003$) but trails \textsc{Unmerge-PCA} and \textsc{CAV-KMeans} by $1.0$--$1.3$ points. The projection of \citet{kodge2024deep} is on par with \textsc{Unmerge-PCA} on $\Wpt$ ($0.974\!\pm\!0.005$ vs.\ $0.972\!\pm\!0.015$) and below the CAV variants on $\Wzero$ ($0.944\!\pm\!0.016$ vs.\ $0.969$--$0.976$); the task-vector baselines stay at $0.86$--$0.90$.

\begin{table}[t]
\centering
\caption{ToW on CIFAR-100 (Trees) over three seeds (seeds $1$--$3$ control the finetuning run, the from-scratch retraining on $\Wzero$, and the unlearning method; $\Winit$ and the data order are fixed), mean $\pm$ sample std. Seed~1 is the cell reported in Tab.~\ref{tab:cifar100}. Best per column in \bestlegend{bold}.}
\label{tab:seeds}
\vspace{-5pt}
\begin{adjustbox}{max width=0.8\linewidth}
\begin{tabular}{l|cc}
\toprule
Method & $\Winit=\Wpt$ & $\Winit=\Wzero$ \\
\midrule
NegGrad+ & 0.9142\,$\pm$\,0.0099 & 0.8280\,$\pm$\,0.0237 \\
Rand.\ Label & 0.8803\,$\pm$\,0.0179 & 0.8391\,$\pm$\,0.0374 \\
SalUn & 0.9587\,$\pm$\,0.0025 & 0.9263\,$\pm$\,0.0357 \\
NegTV & 0.8947\,$\pm$\,0.0054 & 0.8571\,$\pm$\,0.0220 \\
NegMerge & 0.9010\,$\pm$\,0.0097 & 0.8633\,$\pm$\,0.0237 \\
Kodge et al. & \best{0.9744\,$\pm$\,0.0046} & 0.9439\,$\pm$\,0.0156 \\
\hline
Unmerge-PCA & \textbf{0.9716\,$\pm$\,0.0146} & 0.9370\,$\pm$\,0.0208 \\
Unmerge-CAV-KMeans & 0.9688\,$\pm$\,0.0103 & \textbf{0.9688\,$\pm$\,0.0127} \\
Unmerge-CAV-Class & 0.9408\,$\pm$\,0.0102 & \best{0.9756\,$\pm$\,0.0083} \\
\bottomrule
\end{tabular}
\end{adjustbox}
\end{table}

\subsection{Unlearning with ViT}
\label{sec:vit}
\vspace{-5pt}
We evaluate \textsc{Unmerge} on ViT-S/16~\cite{dosovitskiy2021image} (ImageNet-1k weights, $224\times224$ inputs) finetuned on CIFAR-100 (test $90.2\%$) with the Trees superclass as forget set, in the $\Wpt$ regime. The retrained reference uses the same schedule and reaches $86.1\%$ test accuracy, so \eqref{eqn:tow} is well calibrated. The per-layer pipeline (Alg.~\ref{alg:unmerge}) is applied to the linear projections of the attention and MLP blocks; the \texttt{skip-layers} recipe of \S\ref{sec:perlayer} carries over to the transformer's stack of 49 linear layers and, as in ResNet-50, deep-block editing wins: the best cells skip the first 32 (PCA) and 40 (CAV-KMeans) linear layers and edit only the last blocks and the head. Tab.~\ref{tab:vit} reports all three axes of \S\ref{sec:exp}.

\textbf{Result.} \textsc{Unmerge-CAV-KMeans} leads every baseline on every axis simultaneously: ToW $0.9975$ vs.\ $0.9949$ (SalUn) and $0.9925$ (Random Label), at $34$s vs.\ $\sim$$474$s ($14\times$), with the MIA gap to retraining $1.7$--$1.8\times$ smaller ($0.061$ vs.\ $0.105$--$0.109$) and the feature-entanglement gap $6.5$--$7.4\times$ smaller ($0.101$ vs.\ $0.653$--$0.743$). \textsc{Unmerge-PCA} has the smallest MIA gap ($0.056$) and runs in $21$s; NegGrad+ under-forgets ($20\%$ forget accuracy) and damages retain. The retrained reference is strong (test $86.1\%$), and the deep-layer locality finding of \S\ref{sec:perlayer} transfers to attention-based backbones.

\textbf{Bias rule.} ViT-S/16 has a bias in every edited layer, so the bias treatment of App.~\ref{sec:algo-supp} (implementation note iv) could matter here. It does not: with biases left at their finetuned values the two cells move from $0.9975$ to $0.9977$ (\textsc{CAV-KMeans}) and from $0.9810$ to $0.9792$ (\textsc{PCA}), and the $\alpha$ profiles are unchanged to within $0.002$.

\begin{table}[t]
\centering
\caption{ViT-S/16 on CIFAR-100 (Trees), $\Winit=\Wpt$: accuracies (\%), ToW (\eqref{eqn:tow}), runtime, mean MIA gap to retraining over 9 attacks (\S\ref{sec:mia}), and mean relative feature-entanglement gap (\S\ref{sec:entangle}). Best non-retrained in \bestlegend{bold}. NegGrad+ was not evaluated on the MIA/entanglement axes.}
\label{tab:vit}
\vspace{-5pt}
\begin{adjustbox}{max width=\linewidth}
\begin{tabular}{l|c|ccc|cc}
\toprule
& $\Wrt$ & NegGrad+ & Rand.\ Label & SalUn & Unmerge-PCA & \shortstack[c]{Unmerge-CAV\\KMeans} \\
\midrule
Forget Acc $\downarrow$ & 0.00 & 20.00 & 0.12 & 0.24 & 0.92 & 0.16 \\
Retain Acc $\uparrow$ & 99.97 & 86.03 & 99.98 & 99.98 & 99.64 & 99.94 \\
Test Acc $\uparrow$ & 86.11 & 73.43 & 85.49 & 85.85 & 85.44 & 86.05 \\
ToW $\uparrow$ & 1 & 0.6012 & 0.9925 & \textbf{0.9949} & 0.9810 & \best{0.9975} \\
Runtime $\downarrow$ & & 88s & 472s & 474s & \best{21s} & 34s \\
MIA gap $\downarrow$ & 0 & -- & 0.105 & 0.109 & \best{0.056} & \textbf{0.061} \\
Entanglement gap $\downarrow$ & 0 & -- & 0.653 & 0.743 & \textbf{0.171} & \best{0.101} \\
\bottomrule
\end{tabular}
\end{adjustbox}
\end{table}

\subsection{Preliminary Study on a Language Model}
\label{sec:llm}
\vspace{-5pt}
We test whether unmerging carries over to language models on TOFU~\cite{maini2024tofu} (\texttt{forget10}: the $400$ question--answer pairs of $20$ fictitious authors; $3600$ retain pairs) with Llama-3.2-3B-Instruct~\cite{grattafiori2024llama3}, using the checkpoints and the evaluation harness of OpenUnlearning~\cite{dorna2025openunlearning}: $\Winit$ is the public instruction-tuned model, $\Wft$ the released model finetuned on all of TOFU, and $\Wrt$ the released model finetuned on the retain split, so nothing is trained. We apply the per-layer objective to the linear projections of the last $12$--$16$ of $28$ decoder layers, with activation statistics taken at question and answer tokens, and use the closed-form solver of \eqref{eq:closed} (the entries of $\taum\QF$ are about two orders of magnitude smaller here than in the vision models, so the fixed-budget iterate would simply converge to the minimizer). This is a preliminary exploration: each variant receives only light tuning of $\alpha$, the rank and the edited layers on a proxy that needs no generation (TOFU's normalized answer probability on $\setF$ and on held-out retain pairs, relative to $\Wrt$), and the baselines run at OpenUnlearning's default settings.

\textbf{Results.} With a retain-deflated PCA forget basis---the forget covariance after projecting out the top $1024$ retain directions, i.e.\ the hard orthogonal projection that our vision runs leave switched off---and $\lambda{=}0$, $k{=}64$, $\alpha{=}64$, the edit is selective (Tab.~\ref{tab:tofu}). Under the official evaluation its forget quality is $0.131$, i.e.\ the Kolmogorov--Smirnov test no longer separates the edited model from $\Wrt$ at the $5\%$ level, at a model utility of $0.630$ with the MLP projections of the last $16$ layers edited ($\Wrt$: $0.650$). RMU~\cite{li2024wmdp}, the strongest of three baselines, reaches $0.094$ and $0.633$ (GradDiff: $1.6\!\times\!10^{-27}$ and $0.618$; NPO: $1.8\!\times\!10^{-14}$ and $0.601$): \textsc{Unmerge} attains a higher forget quality at a comparable utility ($0.003$ lower) and lands closer to $\Wrt$ on the forget set, where RMU over-forgets (answer probability $0.04$ against $0.12$ for $\Wrt$ and $0.20$ for \textsc{Unmerge}). The edit needs one pass for activation statistics ($43$s) and no training, whereas the baselines train for $70$--$290$s. A training-free, closed-form edit is thus already competitive with a dedicated language-model unlearning method, which makes scaling \textsc{Unmerge} a promising extension.

\textbf{What changes from vision.} Two observations guide that extension. \emph{The basis matters more.} With the plain PCA forget basis the forget and retain bases overlap almost completely (mean per-direction overlap $0.87$--$0.91$) and the edit lowers the answer probability on $\setF$ and on $\setR$ in lockstep, the signature of the random subsets of \S\ref{sec:instance}; deflating the retain directions restores selectivity. The calibrated separability ratio reads $2.8$--$3.4$ even for the plain basis, so on this model the small absolute score ($S{=}0.03$--$0.05$) is the more informative signal. \emph{The CAV bases trail.} Token-level ridge probes, one per forget author (\textsc{CAV-Class}) or per $k$-means cluster of forget tokens (\textsc{CAV-KMeans}), are clean (overlap with the retain basis below $0.13$) but capture only $1$--$8\%$ of the forget-token variance, so retain degrades before forgetting completes (best proxy score $0.55$--$0.59$ against $0.85$ for the deflated PCA basis and $0.18$ for the un-edited model). Concept probes better matched to language, e.g.\ on subject tokens, are a natural next step.

\textbf{Scope.} The study uses one split, one model and one seed. \textsc{Unmerge} is tuned lightly on a proxy that uses $\Wrt$, whereas the baselines are untuned. The scale $\alpha{=}64$ is far above the vision range: the deflated directions carry little of $\taum$, so the edit extrapolates along forget-private input directions rather than removing a finetuning delta.

\begin{table}[t]
\centering
\caption{TOFU \texttt{forget10} with Llama-3.2-3B-Instruct under the official OpenUnlearning evaluation. Forget quality is the $p$-value of a two-sample test against $\Wrt$ (higher is better; $1$ for $\Wrt$ by definition). \textsc{Unmerge}: closed form, retain-deflated PCA forget basis, MLP projections of the last $16$ layers, lightly tuned. RMU: OpenUnlearning default settings.}
\label{tab:tofu}
\vspace{-5pt}
\begin{adjustbox}{max width=0.85\linewidth}
\begin{tabular}{l|cccc}
\toprule
Model & Forget quality $\uparrow$ & Model utility $\uparrow$ & Forget answer prob. & Forget ROUGE \\
\midrule
$\Wft$ (finetuned) & $3.6\!\times\!10^{-27}$ & 0.666 & 0.951 & 0.926 \\
$\Wrt$ (retain-only) & 1 & 0.650 & 0.124 & 0.386 \\
RMU & 0.094 & \textbf{0.633} & 0.037 & 0.225 \\
\textsc{Unmerge} & \textbf{0.131} & 0.630 & 0.200 & 0.364 \\
\bottomrule
\end{tabular}
\end{adjustbox}
\end{table}

\subsection{Feature-Level Entanglement on TinyImageNet}
\label{sec:entangle-tin}
\vspace{-5pt}
The CIFAR-100 picture in Tab.~\ref{tab:entangle-feat} is sharpened on TinyImageNet (Tab.~\ref{tab:entangle-feat-tin}), measured on the Dogs \& Cats split with the same penultimate-feature protocol and same-size retain subsample. \textsc{Unmerge-CAV-Class}/$\Wzero$ matches the retrained reference to within $\sim$\emph{6\%} relative gap (mean $0.062$, with MMD$^2$ at $0.000$ indistinguishable from retrained), while Random Label and SalUn deviate by $3.97$--$3.99\times$, roughly \emph{64$\times$} the unmerge gap. The projection of \citet{kodge2024deep} posts a smaller mean gap on $\Wzero$ ($0.036$), but the two cells are close on this axis: both match the retrained distance to within $2.4\%$ on the two MMD$^2$ variants and deviate by the same $13\%$ on W$_2^2$-Bures (in opposite directions), and the difference in the mean comes from Sliced-W$_2^2$, where \textsc{CAV-Class} lands $9\%$ below the retrained distance (forget and retain features slightly more mixed than under retraining) and the projection $0.3\%$ above. What distinguishes them is how they get there: the projection reaches the retrained feature distance while still classifying $18\%$ of the forget set correctly and losing $7$ test points (Tab.~\ref{tab:tinyimgnt}), i.e.\ by degrading features overall, which ToW and the MIA gap ($0.157$ vs.\ $0.058$) expose; the entanglement axis alone cannot. On $\Wpt$ the spread compresses but Unmerge still leads: \textsc{CAV-Class} ($0.082$) is $\sim$$6\times$ tighter than NegGrad+, $\sim$$50\times$ tighter than RL/SalUn, and slightly tighter than the projection ($0.099$). The W$_2^2$-Bures column is mean-shift dominated in the $D{\gg}N$ regime; the MMD$^2$ and Sliced-W$_2^2$ metrics, sensitive to higher moments and full distribution shape, rank the best \textsc{Unmerge} variant ahead of every gradient and task-vector baseline on both initializations.

\begin{table}[t]
\centering
\caption{Feature-level entanglement on TinyImageNet (Dogs \& Cats), same protocol as Tab.~\ref{tab:entangle-feat}.}
\label{tab:entangle-feat-tin}
\begin{subtable}{\linewidth}
\centering
\label{tab:entangle-feat-tin-wpt}
\vspace{-5pt}
\begin{adjustbox}{max width=\linewidth}
\begin{tabular}{l|c|cccccc|ccc}
\toprule
$\Winit=\Wpt$ & $\Wrt$ & NegGrad+ & Rand.\ Label & SalUn & NegTV & NegMerge & Kodge et al. & Unmerge-PCA & \shortstack[c]{Unmerge-CAV\\KMeans} & \shortstack[c]{Unmerge-CAV\\Class} \\
\midrule
MMD$^2$ & 0.000 & 0.514 & 4.531 & 4.368 & 0.508 & 0.542 & 0.103 & 0.178 & 0.090 & \best{0.058} \\
MMD$^2$-multi & 0.000 & 0.545 & 5.373 & 5.179 & 0.562 & 0.579 & 0.188 & 0.076 & \best{0.070} & 0.074 \\
W$_2^2$-Bures & 0.000 & 0.419 & 2.174 & 2.084 & 0.477 & 0.453 & \best{0.031} & 0.756 & 0.176 & 0.064 \\
Sliced-W$_2^2$ & 0.000 & 0.545 & 4.668 & 4.458 & 0.562 & 0.564 & \best{0.074} & 0.221 & 0.083 & 0.132 \\
\midrule
Mean rel.\ gap $\downarrow$ & 0.000 & 0.506 & 4.187 & 4.022 & 0.527 & 0.534 & 0.099 & 0.308 & 0.105 & \best{0.082} \\
\bottomrule
\end{tabular}
\end{adjustbox}
\end{subtable}
\begin{subtable}{\linewidth}
\centering
\label{tab:entangle-feat-tin-w0}
\begin{adjustbox}{max width=\linewidth}
\begin{tabular}{l|c|cccccc|ccc}
\toprule
$\Winit=\Wzero$ & $\Wrt$ & NegGrad+ & Rand.\ Label & SalUn & NegTV & NegMerge & Kodge et al. & Unmerge-PCA & \shortstack[c]{Unmerge-CAV\\KMeans} & \shortstack[c]{Unmerge-CAV\\Class} \\
\midrule
MMD$^2$ & 0.000 & 0.439 & 4.187 & 4.186 & 0.352 & 0.342 & 0.012 & 0.098 & 0.268 & \best{0.000} \\
MMD$^2$-multi & 0.000 & 0.469 & 4.968 & 4.978 & 0.371 & 0.359 & \best{0.003} & 0.209 & 0.246 & 0.024 \\
W$_2^2$-Bures & 0.000 & 0.331 & 2.302 & 2.317 & 0.136 & 0.123 & 0.125 & 0.490 & \best{0.095} & 0.129 \\
Sliced-W$_2^2$ & 0.000 & 0.474 & 4.441 & 4.469 & 0.315 & 0.301 & \best{0.003} & 0.433 & 0.194 & 0.093 \\
\midrule
Mean rel.\ gap $\downarrow$ & 0.000 & 0.428 & 3.974 & 3.988 & 0.293 & 0.281 & \best{0.036} & 0.308 & 0.201 & 0.062 \\
\bottomrule
\end{tabular}
\end{adjustbox}
\end{subtable}
\end{table}

\subsection{Membership Inference: Detailed Results}
\label{sec:mia-supp}
\vspace{-5pt}
We provide per-attack MIA tables for both datasets, expanding the summary in \S\ref{sec:mia}. Both protocols are identical (RUM/SalUn-style shadow/target split, nine attack signals, gap-to-retrained metric). Three details matter when reading the numbers. The attack accuracy on $\setF$ averages the attacker's true-positive rate on retain samples and its true-negative rate on forget samples, so damage to retain accuracy also registers as a privacy gap. The SVC attackers are trained on an unbalanced shadow set ($\sim$$83\%$ members) without class weighting, and the correctness signal appears twice (threshold and SVC give the same attacker). Finally, the class-level tables report the mean over signals of the \emph{absolute} gap, whereas the instance-level table (Tab.~\ref{tab:instance}) reports the \emph{signed} gap of the signal-averaged attack accuracy, in which per-signal errors of opposite sign can cancel; App.~\ref{sec:instance-supp} gives both aggregates for the coherent protocol.

\textbf{CIFAR-100 (Trees).} Tab.~\ref{tab:mia-cifar100} reports per-attack gaps. On $\Wzero$ the projection of \citet{kodge2024deep} has the smallest mean gap ($0.017$), followed by \textsc{Unmerge-CAV-KMeans} ($0.026$) and \textsc{Unmerge-CAV-Class} ($0.029$), $1.6$--$1.8\times$ tighter than SalUn ($0.047$), $2.1$--$4.2\times$ tighter than Random Label ($0.062$) and NegGrad+ ($0.109$), and $2.8$--$3.5\times$ tighter than NegTV and NegMerge ($0.082$--$0.092$); the SVC-probs and SVC-entropy attacks (full softmax-vector attackers, the strongest in our set) are where the gap closes most, with the CAV variants cutting both by $1.5$--$2.3\times$ over the saliency baselines. On $\Wpt$ the projection leads again ($0.028$) and the spread among the other methods is tight: NegGrad+ $0.067$, \textsc{Unmerge-CAV-Class} $0.070$, NegMerge $0.072$, NegTV $0.073$, \textsc{Unmerge-CAV-KMeans} $0.075$, \textsc{Unmerge-PCA} $0.078$, SalUn $0.082$, Random Label $0.093$. NegGrad+'s small edge over \textsc{Unmerge} comes at a $2.8$-point retain and $6.2$-point test drop (Tab.~\ref{tab:cifar100}), so the joint utility$\,+\,$privacy ranking still favors \textsc{Unmerge}.

\begin{table}[tb]
\centering
\caption{MIA gap to the retrained reference on CIFAR-100 (Trees): $|\mathrm{att}_{\setF}^{\mathrm{unl}}-\mathrm{att}_{\setF}^{\mathrm{RT}}|$ per attack signal (lower is better; $\Wrt$ is 0 by construction).}
\label{tab:mia-cifar100}
\begin{subtable}{\linewidth}
\centering
\label{tab:mia-cifar100-wpt}
\vspace{-5pt}
\begin{adjustbox}{max width=\linewidth}
\begin{tabular}{l|c|cccccc|ccc}
\toprule
$\Winit=\Wpt$ & $\Wrt$ & NegGrad+ & Rand.\ Label & SalUn & NegTV & NegMerge & Kodge et al. & Unmerge-PCA & \shortstack[c]{Unmerge-CAV\\KMeans} & \shortstack[c]{Unmerge-CAV\\Class} \\
\midrule
Threshold-Conf. & 0.000 & 0.034 & 0.017 & \best{0.011} & 0.036 & 0.035 & 0.015 & 0.024 & 0.028 & 0.033 \\
Threshold-Ent. & 0.000 & 0.033 & 0.171 & 0.173 & \best{0.027} & 0.048 & 0.104 & 0.160 & 0.154 & 0.144 \\
Threshold-M-Ent. & 0.000 & 0.035 & 0.021 & \best{0.007} & 0.036 & 0.035 & 0.018 & 0.024 & 0.027 & 0.040 \\
Threshold-Corr. & 0.000 & 0.016 & 0.061 & 0.022 & 0.025 & 0.025 & 0.002 & \best{0.001} & 0.003 & 0.007 \\
SVC-Conf. & 0.000 & 0.009 & \best{0.002} & 0.003 & 0.020 & 0.021 & 0.004 & 0.003 & \best{0.002} & 0.006 \\
SVC-Ent. & 0.000 & 0.211 & 0.209 & 0.207 & 0.261 & 0.242 & \best{0.065} & 0.197 & 0.194 & 0.154 \\
SVC-M-Ent. & 0.000 & 0.049 & \best{0.001} & 0.002 & 0.024 & 0.027 & 0.003 & 0.002 & 0.003 & 0.006 \\
SVC-Corr. & 0.000 & 0.015 & 0.061 & 0.022 & 0.023 & 0.025 & 0.002 & \best{0.001} & 0.003 & 0.007 \\
SVC-Probs. & 0.000 & 0.203 & 0.291 & 0.290 & 0.201 & 0.190 & \best{0.038} & 0.288 & 0.258 & 0.233 \\
\midrule
Mean$|$gap$|$ $\downarrow$ & 0.000 & 0.067 & 0.093 & 0.082 & 0.073 & 0.072 & \best{0.028} & 0.078 & 0.075 & 0.070 \\
\bottomrule
\end{tabular}
\end{adjustbox}
\end{subtable}
\begin{subtable}{\linewidth}
\centering
\label{tab:mia-cifar100-w0}
\begin{adjustbox}{max width=\linewidth}
\begin{tabular}{l|c|cccccc|ccc}
\toprule
$\Winit=\Wzero$ & $\Wrt$ & NegGrad+ & Rand.\ Label & SalUn & NegTV & NegMerge & Kodge et al. & Unmerge-PCA & \shortstack[c]{Unmerge-CAV\\KMeans} & \shortstack[c]{Unmerge-CAV\\Class} \\
\midrule
Threshold-Conf. & 0.000 & 0.037 & 0.010 & 0.003 & 0.022 & 0.026 & 0.007 & 0.008 & \best{0.001} & \best{0.001} \\
Threshold-Ent. & 0.000 & 0.103 & 0.063 & 0.065 & 0.072 & 0.057 & \best{0.017} & 0.040 & 0.046 & 0.054 \\
Threshold-M-Ent. & 0.000 & 0.040 & 0.007 & 0.002 & 0.025 & 0.032 & 0.004 & 0.005 & 0.006 & \best{0.001} \\
Threshold-Corr. & 0.000 & 0.042 & 0.077 & 0.016 & 0.029 & 0.027 & 0.005 & 0.006 & \best{0.004} & 0.006 \\
SVC-Conf. & 0.000 & 0.041 & \best{0.000} & \best{0.000} & 0.019 & 0.021 & 0.003 & 0.003 & 0.003 & 0.002 \\
SVC-Ent. & 0.000 & 0.281 & 0.143 & 0.139 & 0.295 & 0.231 & \best{0.017} & 0.354 & 0.084 & 0.093 \\
SVC-M-Ent. & 0.000 & 0.079 & \best{0.000} & 0.001 & 0.023 & 0.025 & 0.004 & 0.005 & 0.002 & 0.001 \\
SVC-Corr. & 0.000 & 0.041 & 0.077 & 0.016 & 0.030 & 0.027 & 0.005 & 0.006 & \best{0.004} & 0.006 \\
SVC-Probs. & 0.000 & 0.314 & 0.184 & 0.183 & 0.309 & 0.294 & 0.094 & 0.316 & \best{0.080} & 0.094 \\
\midrule
Mean$|$gap$|$ $\downarrow$ & 0.000 & 0.109 & 0.062 & 0.047 & 0.092 & 0.082 & \best{0.017} & 0.083 & 0.026 & 0.029 \\
\bottomrule
\end{tabular}
\end{adjustbox}
\end{subtable}
\end{table}

\textbf{TinyImageNet (Dogs \& Cats).} On $\Wpt$, all three \textsc{Unmerge} variants sweep the table (Tab.~\ref{tab:mia-tin}): \textsc{Unmerge-PCA} ($0.009$) and \textsc{Unmerge-CAV-Class}/\textsc{Unmerge-CAV-KMeans} ($0.018$/$0.020$) are $2$--$5\times$ tighter than the saliency baselines (Random Label/SalUn $\sim$$0.043$), $9$--$20\times$ tighter than NegGrad+ ($0.177$), and $13$--$29\times$ tighter than the task-vector baselines and the projection of \citet{kodge2024deep} ($0.254$--$0.261$). On $\Wzero$, Random Label ($0.034$) and SalUn ($0.040$) appear to lead the CAV variants ($0.058$). This is an artifact of \emph{under-unlearning}: the baseline configurations used on this split (Random Label lr$=$$10^{-3}$, SalUn $\tau{=}0.5$/lr$=$$10^{-3}$) leave about $9\%$ forget accuracy (Tab.~\ref{tab:tinyimgnt}), so their predictions on $\setF$ have barely moved from the finetuned model and the MIA gap closes by inaction rather than by erasure. The entanglement axis exposes it: the same checkpoints deviate from the retrained reference by $\sim$$4\times$ its value, while \textsc{Unmerge-CAV-Class}/$\Wzero$ matches retraining to within $6\%$ (Tab.~\ref{tab:entangle-feat-tin}). Neither axis suffices alone: the MIA gap can be closed by leaving the model unchanged, and the feature distance can be matched by degrading the features overall, as the projection does on this split (App.~\ref{sec:entangle-tin}). Read together with accuracy, the two axes single out the methods that forget without either failure.

\begin{table}[t]
\centering
\caption{MIA gap to the retrained reference on TinyImageNet (Dogs \& Cats), same protocol as Tab.~\ref{tab:mia-cifar100}.}
\label{tab:mia-tin}
\vspace{-5pt}
\begin{subtable}{\linewidth}
\centering
\label{tab:mia-tin-wpt}
\begin{adjustbox}{max width=\linewidth}
\begin{tabular}{l|c|cccccc|ccc}
\toprule
$\Winit=\Wpt$ & $\Wrt$ & NegGrad+ & Rand.\ Label & SalUn & NegTV & NegMerge & Kodge et al. & Unmerge-PCA & \shortstack[c]{Unmerge-CAV\\KMeans} & \shortstack[c]{Unmerge-CAV\\Class} \\
\midrule
Threshold-Conf. & 0.000 & 0.044 & \best{0.007} & 0.009 & 0.051 & 0.044 & 0.046 & 0.016 & 0.008 & 0.012 \\
Threshold-Ent. & 0.000 & 0.108 & 0.012 & 0.011 & 0.113 & 0.095 & 0.076 & \best{0.005} & 0.012 & 0.006 \\
Threshold-M-Ent. & 0.000 & 0.048 & \best{0.006} & 0.011 & 0.051 & 0.048 & 0.048 & 0.018 & 0.008 & 0.009 \\
Threshold-Corr. & 0.000 & 0.025 & 0.081 & 0.077 & 0.052 & 0.046 & 0.042 & 0.007 & \best{0.000} & \best{0.000} \\
SVC-Conf. & 0.000 & 0.028 & \best{0.001} & \best{0.001} & 0.500 & 0.500 & 0.500 & 0.002 & \best{0.001} & \best{0.001} \\
SVC-Ent. & 0.000 & 0.431 & 0.066 & 0.066 & 0.431 & 0.431 & 0.413 & \best{0.021} & 0.055 & 0.043 \\
SVC-M-Ent. & 0.000 & 0.048 & \best{0.000} & \best{0.000} & 0.287 & 0.278 & 0.301 & 0.002 & \best{0.000} & 0.001 \\
SVC-Corr. & 0.000 & 0.500 & 0.081 & 0.077 & 0.500 & 0.500 & 0.500 & 0.007 & \best{0.000} & \best{0.000} \\
SVC-Probs. & 0.000 & 0.360 & 0.137 & 0.137 & 0.363 & 0.362 & 0.362 & \best{0.007} & 0.099 & 0.092 \\
\midrule
Mean$|$gap$|$ $\downarrow$ & 0.000 & 0.177 & 0.044 & 0.043 & 0.261 & 0.256 & 0.254 & \best{0.009} & 0.020 & 0.018 \\
\bottomrule
\end{tabular}
\end{adjustbox}
\end{subtable}
\begin{subtable}{\linewidth}
\centering
\label{tab:mia-tin-w0}
\begin{adjustbox}{max width=\linewidth}
\begin{tabular}{l|c|cccccc|ccc}
\toprule
$\Winit=\Wzero$ & $\Wrt$ & NegGrad+ & Rand.\ Label & SalUn & NegTV & NegMerge & Kodge et al. & Unmerge-PCA & \shortstack[c]{Unmerge-CAV\\KMeans} & \shortstack[c]{Unmerge-CAV\\Class} \\
\midrule
Threshold-Conf. & 0.000 & 0.064 & \best{0.005} & 0.009 & 0.064 & 0.063 & 0.063 & 0.047 & 0.028 & 0.025 \\
Threshold-Ent. & 0.000 & 0.171 & 0.018 & \best{0.017} & 0.092 & 0.089 & 0.061 & 0.107 & 0.039 & 0.040 \\
Threshold-M-Ent. & 0.000 & 0.067 & \best{0.003} & 0.008 & 0.065 & 0.064 & 0.065 & 0.052 & 0.029 & 0.025 \\
Threshold-Corr. & 0.000 & 0.072 & 0.047 & 0.045 & 0.076 & 0.076 & 0.107 & 0.104 & 0.025 & \best{0.013} \\
SVC-Conf. & 0.000 & 0.500 & \best{0.001} & \best{0.001} & 0.111 & 0.117 & 0.170 & 0.500 & 0.028 & 0.006 \\
SVC-Ent. & 0.000 & 0.435 & \best{0.083} & 0.137 & 0.435 & 0.435 & 0.304 & 0.435 & 0.128 & 0.171 \\
SVC-M-Ent. & 0.000 & 0.137 & \best{0.001} & 0.002 & 0.214 & 0.200 & 0.191 & 0.311 & 0.052 & 0.010 \\
SVC-Corr. & 0.000 & 0.500 & 0.047 & 0.045 & 0.500 & 0.500 & 0.106 & 0.500 & 0.024 & \best{0.011} \\
SVC-Probs. & 0.000 & 0.397 & \best{0.101} & \best{0.101} & 0.383 & 0.380 & 0.349 & 0.397 & 0.172 & 0.220 \\
\midrule
Mean$|$gap$|$ $\downarrow$ & 0.000 & 0.260 & \best{0.034} & 0.040 & 0.215 & 0.214 & 0.157 & 0.272 & 0.058 & 0.058 \\
\bottomrule
\end{tabular}
\end{adjustbox}
\end{subtable}
\end{table}

\subsection{Instance-Level Forgetting: Details}
\label{sec:instance-supp}
\vspace{-5pt}
\textbf{Protocols and references.} Coherent: a fixed random half of the $2500$ Trees training images forms $\setF$ ($|\setF|{=}1250$; the draw is not stratified, giving $237$--$262$ images per class) and the other half stays in $\setR$. Random: $2500$ training images drawn uniformly from all 100 classes. Each protocol and initialization has its own retrained reference (coherent: $\Wzero$ F/R/T $70.64$/$99.96$/$76.27$, $\Wpt$ $74.24$/$99.98$/$84.93$; random: $\Wzero$ $74.92$/$99.97$/$74.73$, $\Wpt$ $84.56$/$99.98$/$84.02$). Baselines use the class-level configurations of \S\ref{sec:setup}, except NegGrad+ on the coherent protocol: its class-level setting over-forgets there (gradient ascent on half a class erases the class; forget accuracy $26\%$/$11\%$, ToW $0.4970$/$0.3251$), so we tune it on a 40-run grid over its forget/retain weight $\{0.3,\dots,0.99\}$, learning rate $\{10^{-5},\dots,10^{-4}\}$ and epochs $\{5,\dots,12\}$. \textsc{Unmerge} uses, on the random protocol, the protocol grid of App.~\ref{sec:ablation} (\texttt{skip-layers} $\{35,\dots,50\}$, $\lambda\in\{5,20,50\}$, $\gamma\in\{0.25,0.5\}$, PCA $k\in\{64,128\}$, CAV $k{=}16$, fine $\alpha$), and on the coherent protocol its class-level configuration with only $\alpha$ swept (five values per basis). All cells report the best final ToW of their grid, and every reported optimum is interior in the swept quantity.

\begin{table}[t]
\centering
\caption{Instance-level forgetting on CIFAR-100 (Trees), ResNet-50. \emph{Coherent}: half of the training images of each Trees class ($|\setF|{=}1250$); \emph{Random}: $2500$ images drawn uniformly from all 100 classes. Accuracies in \%. $\mathrm{att}_{\setF}$: mean attack accuracy on $\setF$ over the 9 MIA signals of \S\ref{sec:mia}, with the signed gap to $\Wrt$ in parentheses (positive = forget samples more identifiable than under retraining). Ent.: mean relative entanglement gap of Tab.~\ref{tab:entangle-feat}. Best non-retrained ToW and entanglement per protocol in \bestlegend{bold}. $^{\dagger}$On the coherent protocol NegGrad+ is tuned on a 40-run grid (forget/retain weight, learning rate, epochs), because its class-level setting over-forgets there (ToW $0.4970$/$0.3251$); Random Label and SalUn keep their class-level settings, and the \textsc{Unmerge} cells sweep only $\alpha$ (five values per basis). The identical forget/retain accuracies of \textsc{Unmerge-PCA} on $\Wzero$ and \textsc{Unmerge-CAV-KMeans} on $\Wpt$ are genuine (901 of 1250 forget images correct in both; test accuracies and losses differ).}
\label{tab:instance}
\vspace{-5pt}
\begin{subtable}{\linewidth}
\centering
\begin{adjustbox}{max width=\linewidth}
\begin{tabular}{l|ccccc|ccccc}
\toprule
$\Winit=\Wzero$ & \multicolumn{5}{c|}{Coherent ($|\setF|{=}1250$)} & \multicolumn{5}{c}{Random ($|\setF|{=}2500$)} \\
 & Forget & Retain & ToW $\uparrow$ & $\mathrm{att}_{\setF}$ (gap) & Ent.\ $\downarrow$ & Forget & Retain & ToW $\uparrow$ & $\mathrm{att}_{\setF}$ (gap) & Ent.\ $\downarrow$ \\
\midrule
$\Wrt$ & 70.64 & 99.96 & 1 & 0.700 & 0 & 74.92 & 99.97 & 1 & 0.681 & 0 \\
NegGrad+$^{\dagger}$ & 71.60 & 98.62 & 0.9368 & 0.624 ($-$0.076) & 0.46 & 99.92 & 99.97 & 0.7302 & 0.502 ($-$0.180) & 0.62 \\
Rand.\ Label & 88.72 & 99.96 & 0.8137 & 0.839 ($+$0.140) & 0.41 & 69.80 & 99.96 & 0.9405 & 0.872 ($+$0.191) & 21.9 \\
SalUn & 87.44 & 99.93 & 0.8243 & 0.825 ($+$0.125) & 0.29 & 72.20 & 99.89 & \best{0.9595} & 0.839 ($+$0.158) & 15.5 \\
Unmerge-PCA & 72.08 & 99.15 & \best{0.9665} & 0.712 ($+$0.013) & \best{0.15} & 99.64 & 99.77 & 0.7509 & 0.502 ($-$0.180) & 0.67 \\
Unmerge-CAV-KMeans & 73.60 & 99.25 & 0.9534 & 0.691 ($-$0.009) & 0.27 & 97.32 & 98.19 & 0.7515 & 0.506 ($-$0.176) & \best{0.60} \\
\bottomrule
\end{tabular}
\end{adjustbox}
\end{subtable}
\begin{subtable}{\linewidth}
\centering
\begin{adjustbox}{max width=\linewidth}
\begin{tabular}{l|ccccc|ccccc}
\toprule
$\Winit=\Wpt$ & \multicolumn{5}{c|}{Coherent ($|\setF|{=}1250$)} & \multicolumn{5}{c}{Random ($|\setF|{=}2500$)} \\
 & Forget & Retain & ToW $\uparrow$ & $\mathrm{att}_{\setF}$ (gap) & Ent.\ $\downarrow$ & Forget & Retain & ToW $\uparrow$ & $\mathrm{att}_{\setF}$ (gap) & Ent.\ $\downarrow$ \\
\midrule
$\Wrt$ & 74.24 & 99.98 & 1 & 0.686 & 0 & 84.56 & 99.98 & 1 & 0.630 & 0 \\
NegGrad+$^{\dagger}$ & 74.96 & 99.08 & \best{0.9427} & 0.620 ($-$0.066) & 0.48 & 100.0 & 99.97 & 0.8406 & 0.500 ($-$0.130) & 0.55 \\
Rand.\ Label & 85.84 & 99.92 & 0.8639 & 0.852 ($+$0.165) & 0.55 & 69.04 & 99.93 & 0.8328 & 0.879 ($+$0.249) & 81.7 \\
SalUn & 83.12 & 99.92 & 0.8970 & 0.841 ($+$0.155) & 0.41 & 73.68 & 99.88 & \best{0.8788} & 0.850 ($+$0.221) & 50.1 \\
Unmerge-PCA & 67.12 & 99.10 & 0.9002 & 0.642 ($-$0.044) & 0.16 & 87.40 & 90.89 & 0.7804 & 0.519 ($-$0.111) & \best{0.50} \\
Unmerge-CAV-KMeans & 72.08 & 99.15 & 0.9417 & 0.677 ($-$0.010) & \best{0.11} & 98.64 & 99.09 & 0.8423 & 0.505 ($-$0.125) & 0.54 \\
\bottomrule
\end{tabular}
\end{adjustbox}
\end{subtable}
\end{table}

\textbf{Coherent protocol, per group: the aggregate hides a concept-level edit.} The forget half and the retained half of a Trees class are i.i.d.\ draws, so no activation subspace separates them. Tab.~\ref{tab:coherent-breakdown} splits the retain and test accuracies of Tab.~\ref{tab:instance} accordingly. \textsc{Unmerge} matches the retrained accuracy on the forget half by lowering the \emph{whole class}: the retained half falls to the same level ($67$--$73\%$, against $\ge\!99.9\%$ under retraining) and the Trees test accuracy falls $18$--$26$ points below the retrained model's, while the other $95$ classes stay within $2.2$ test points of retraining. The tuned NegGrad+ does the same. The aggregate retain accuracy moves by at most $1.3$ points because the retained half is $2.6\%$ of $\setR$, which is why ToW ranks these cells first. Random Label and SalUn are the mirror image: they keep the retained half ($\ge\!99\%$) and most of the Trees test accuracy, and leave the forget half $9$--$18$ points above the retrained level while marking it (attack accuracy above retraining, Tab.~\ref{tab:instance}). No method reproduces retraining on this protocol; \textsc{Unmerge} treats the request as a partial removal of the concept, which is the behavior its construction implies and the reason we scope it to concept-level requests (\S\ref{sec:instance}).

\begin{table}[t]
\centering
\caption{Coherent protocol, per-group accuracy (\%) of the cells of Tab.~\ref{tab:instance}. Forget half: the $1250$ forget images; retained half: the $1250$ Trees training images that stay in $\setR$; other: the $47{,}500$ retain images of the other $95$ classes; Trees test / other test: the $500$ / $9500$ test images. Retraining keeps the retained half at $\ge\!99.9\%$ while the forget half drops to test level. \textsc{Unmerge} and the tuned NegGrad+ ($^{\dagger}$) lower both halves alike and the Trees test accuracy with them; Random Label and SalUn keep the retained half and under-forget the forget half.}
\label{tab:coherent-breakdown}
\vspace{-5pt}
\begin{adjustbox}{max width=\linewidth}
\begin{tabular}{l|ccc|cc|ccc|cc}
\toprule
 & \multicolumn{5}{c|}{$\Winit=\Wzero$} & \multicolumn{5}{c}{$\Winit=\Wpt$} \\
 & \shortstack[c]{Forget\\half} & \shortstack[c]{Retained\\half} & \shortstack[c]{Other\\retain} & \shortstack[c]{Trees\\test} & \shortstack[c]{Other\\test} & \shortstack[c]{Forget\\half} & \shortstack[c]{Retained\\half} & \shortstack[c]{Other\\retain} & \shortstack[c]{Trees\\test} & \shortstack[c]{Other\\test} \\
\midrule
$\Wrt$ & 70.64 & 99.92 & 99.96 & 66.40 & 76.79 & 74.24 & 100.00 & 99.97 & 74.80 & 85.46 \\
$\Wft$ (un-edited) & 99.92 & 100.00 & 99.96 & 73.80 & 77.69 & 99.92 & 100.00 & 99.97 & 74.00 & 85.27 \\
NegGrad+$^{\dagger}$ & 71.60 & 71.12 & 99.35 & 44.20 & 73.61 & 74.96 & 75.44 & 99.70 & 47.20 & 82.51 \\
Rand.\ Label & 88.72 & 99.92 & 99.96 & 67.40 & 77.44 & 85.84 & 99.76 & 99.93 & 67.80 & 83.49 \\
SalUn & 87.44 & 99.36 & 99.95 & 68.40 & 77.63 & 83.12 & 99.04 & 99.94 & 69.20 & 84.19 \\
Unmerge-PCA & 72.08 & 72.40 & 99.85 & 45.80 & 76.65 & 67.12 & 67.12 & 99.94 & 48.40 & 84.53 \\
Unmerge-CAV-KMeans & 73.60 & 73.36 & 99.93 & 44.20 & 76.85 & 72.08 & 73.20 & 99.83 & 56.80 & 83.32 \\
\bottomrule
\end{tabular}
\end{adjustbox}
\end{table}

\textbf{Random protocol: the lockstep ceiling.} The retrained model itself barely forgets a random subset (forget accuracy $75$/$85\%$, close to its test accuracy), so the untouched $\Wft$ already scores ToW $0.73$/$0.84$, and a 60-configuration sweep confirms that the ceiling is not under-tuning: forget and retain accuracy move together at every $\alpha$, and matching the retrained forget accuracy on $\Wpt$ costs $13$ points of retain. Random Label and SalUn reach ToW up to $0.96$ only by \emph{marking} the forget samples: attack accuracy on $\setF$ rises to $0.84$--$0.88$, above the retrained model's $0.63$--$0.68$, and their forget features become separable from retain features (relative entanglement gap $16$--$82$ against $\le\!0.7$ for \textsc{Unmerge}). The cells in Tab.~\ref{tab:instance} are the ToW maxima over genuine edits (no-op edits, whose forget accuracy stays at the un-edited level, excluded; ties within $0.005$ broken by entanglement, then signed MIA gap).

\textbf{Instance-level and concept-level forgetting are different problems.} \citet{triantafillou2026untraining} distinguish \emph{untraining}---reversing the effect of having trained on the given forget samples, whose reference is the model retrained on $\setS\setminus\setF$---from \emph{unlearning} the distribution or concept that the samples represent, and note that the two coincide when $\setF$ contains every training instance of the concept, which is our class-level protocol. Otherwise they pull in opposite directions: while the concept stays represented in $\setR$, the untraining target still predicts the forget samples correctly and the right edit is close to none, whereas removing the concept calls for a large one; \citet{cooper2025machine} describe the same mismatch between removing observed information and suppressing what a model has generalized. Our random protocol is the extreme case: its forget samples are i.i.d.\ draws from the training distribution, so the retrained forget accuracy matches test accuracy and the un-edited $\Wft$ already scores ToW $0.73$/$0.84$, in line with \citet{zhao2024makes}, who find that for forget sets that are not memorized the original model is already close to retraining. What remains is per-sample memorization, on which an edit built to generalize over what the forget samples share has nothing to act---which the separability score reports beforehand.

\textbf{Separability diagnostics.} The separability score is $S=\mathrm{fvar}\cdot(1-\mathrm{overlap})$, the fraction of forget variance captured by $\QF$ times one minus the mean per-direction overlap of $\QF$ with $\QR$; both factors are available after Phase~1, before any edit. Tab.~\ref{tab:sdiag} reports them. Within a dataset $S$ separates the regimes cleanly ($2.4$--$2.5\times$ between class-level and random CIFAR-100 in both initializations, through different factors: $\Wzero$ loses the private subspace through overlap, $0.61\!\to\!0.83$, and $\Wpt$ through captured variance, $0.63\!\to\!0.28$). Across datasets it does not: TinyImageNet class-level runs measure $S\approx0.07$--$0.10$, the CIFAR-100 ``hard'' range, yet unlearn well, because from-scratch TinyImageNet features are less structured and $S$ is basis-dependent. The ratio $S_{\text{task}}/S_{\text{rand}}$ against a same-size random subset of the same dataset removes both dependencies at the cost of one extra forward pass.

\begin{table}[t]
\centering
\caption{Phase-1 separability diagnostics over edited layers: fraction of forget variance captured by $\QF$ (fvar), mean per-direction overlap between $\QF$ and $\QR$, and $S=\mathrm{fvar}\cdot(1-\mathrm{overlap})$. CIFAR-100: CAV-KMeans basis ($k{=}16$); the coherent row uses the PCA basis of its winning cell. TinyImageNet: bases of the official $\Wzero$ cells (the diagnostic was not logged for the $\Wpt$ TinyImageNet runs). Last column: $S$ relative to the same-dataset random-subset value.}
\label{tab:sdiag}
\vspace{-5pt}
\begin{adjustbox}{max width=0.9\linewidth}
\begin{tabular}{l|l|ccc|c}
\toprule
Dataset & Forget request & fvar & overlap & $S$ & $S/S_{\text{rand}}$ \\
\midrule
CIFAR-100 ($\Wzero$) & Trees, class-level & 0.438 & 0.611 & 0.170 & 2.4 \\
& Trees, coherent half-classes & -- & -- & 0.17 & 2.4 \\
& Random subset ($2500$) & 0.421 & 0.832 & 0.071 & 1 \\
\midrule
CIFAR-100 ($\Wpt$) & Trees, class-level & 0.629 & 0.729 & 0.170 & 2.5 \\
& Random subset ($2500$) & 0.281 & 0.755 & 0.069 & 1 \\
\midrule
TinyImageNet ($\Wzero$) & Dogs \& Cats, CAV-KMeans & 0.252 & 0.628 & 0.094 & -- \\
& Arthropods, CAV-Class & 0.135--0.143 & 0.347--0.460 & 0.073--0.093 & -- \\
& Arthropods, CAV-KMeans & 0.259 & 0.611 & 0.101 & -- \\
& Arthropods, PCA & 0.526 & 0.864 & 0.072 & -- \\
\bottomrule
\end{tabular}
\end{adjustbox}
\end{table}

\subsection{Decomposition Validation: Details}
\label{sec:tauval-supp}
\vspace{-5pt}
Accuracy-based scores are proxies; the quantity the method acts on, $\tauR=\taum-\alpha\tauF$, can be compared with $\tauRstar=\Wrt-\Winit$ directly. For every checkpoint on Trees we compute, over all weight tensors with at least two dimensions ($53$--$54$ layers, not only the edited ones), the relative Frobenius error $\|\tauR-\tauRstar\|_F/\|\tauRstar\|_F$ averaged over layers and the size-weighted cosine between $\tauR$ and $\tauRstar$ (for baselines $\tauR=\Wul-\Winit$), and the predictive divergence $\KL(p_{\text{UL}}\|p_{\text{RT}})$ on the forget, retain, and test splits (Tab.~\ref{tab:tauval}). Two references anchor the weight-space columns: the un-edited $\Wft$, and the retrain-to-retrain noise floor obtained by retraining twice more from the same $\Winit$ with different data orders. On $\Wzero$ every checkpoint, including independent retrains, has cosine $\approx\!0.005$ and relative error $1.5$--$1.7$ to the reference: from-scratch runs converge to unrelated basins, so only the output-space columns are meaningful there. On $\Wpt$ the floor is $0.718$--$0.731$ and $\Wft$ sits at $0.751$; all \textsc{Unmerge} variants stay at the un-edited level ($0.750$--$0.754$: the minimal edit does no weight-space harm), whereas Random Label ($0.823$) and SalUn ($0.908$) move \emph{away} from the retrained solution. Output space discriminates more sharply in the same direction: \textsc{Unmerge-CAV-KMeans} has the smallest forget-set divergence in both regimes ($2.89$ on $\Wzero$ vs.\ $4.9$--$5.2$ for the baselines; $5.29$ vs.\ $5.9$--$6.6$ on $\Wpt$), because relabeling pushes forget inputs confidently to wrong classes---far from a model that never saw them---while subspace subtraction lands closer (cf.\ \S\ref{sec:entangle}). \emph{Caveats.} (a) \textsc{CAV-Class} at $\alpha{=}2.5$ on $\Wpt$ matches the retrained forget accuracy to within $0.12$ points but has the largest forget-set divergence of any cell ($7.15$): an aggressive $\alpha$ over-shoots distributionally while matching accuracy, so $\alpha$ should be selected on a distributional criterion when one is available. (b) \textsc{Unmerge-PCA} on $\Wzero$ has large divergence on every split ($3.87$ on retain despite $98.75\%$ retain accuracy): the argmax survives but calibration does not, consistent with its entanglement outlier in Tab.~\ref{tab:entangle-feat} and with the regime split of \S\ref{sec:main}.

\begin{table}[t]
\centering
\caption{Decomposition validation on CIFAR-100 (Trees), ResNet-50. Weight space: layer-averaged relative Frobenius error and size-weighted cosine of $\tauR$ vs.\ $\tauRstar$ over all weight matrices. Output space: $\KL(p_{\text{UL}}\|p_{\text{RT}})$ per split (lower is closer to retraining). Floor: two independent retrains from the same $\Winit$. Best non-reference per KL column in \bestlegend{bold}; the $\Wzero$ weight-space columns are uninformative (see text).}
\label{tab:tauval}
\vspace{-5pt}
\begin{adjustbox}{max width=\linewidth}
\begin{tabular}{l|cc|ccc||cc|ccc}
\toprule
& \multicolumn{5}{c||}{$\Winit=\Wpt$} & \multicolumn{5}{c}{$\Winit=\Wzero$} \\
Method & rel.\ err & cos & KL$_{\setF}$ & KL$_{\setR}$ & KL$_{\text{test}}$ & rel.\ err & cos & KL$_{\setF}$ & KL$_{\setR}$ & KL$_{\text{test}}$ \\
\midrule
Retrain floor (2 runs) & 0.718--0.731 & 0.684--0.688 & -- & -- & -- & 1.539--1.669 & 0.005 & -- & -- & -- \\
$\Wft$ (no unlearning) & 0.751 & 0.691 & -- & -- & -- & 1.579 & 0.006 & -- & -- & -- \\
\midrule
NegGrad+ & 0.751 & 0.691 & 5.89 & 0.397 & 1.151 & 1.579 & 0.006 & 4.89 & 0.500 & 1.293 \\
Rand.\ Label & 0.823 & 0.682 & 6.62 & \best{0.014} & 0.704 & 1.515 & 0.006 & 5.20 & \best{0.011} & 0.755 \\
SalUn & 0.908 & 0.684 & 6.62 & 0.040 & \best{0.698} & 1.547 & 0.006 & 5.21 & 0.024 & 0.736 \\
Unmerge-PCA & 0.750 & 0.688 & 6.03 & 0.172 & 0.864 & 1.572 & 0.006 & 4.99 & 3.873 & 3.471 \\
Unmerge-CAV-KMeans & 0.754 & 0.683 & \best{5.29} & 0.107 & 0.825 & 1.577 & 0.006 & \best{2.89} & 0.043 & 0.717 \\
Unmerge-CAV-Class & 0.754 & 0.685 & 7.15 & 0.192 & 1.049 & 1.577 & 0.006 & 3.31 & 0.029 & \best{0.708} \\
\bottomrule
\end{tabular}
\end{adjustbox}
\end{table}

\section{Limitations and Future Work}
\label{sec:limitations}
\vspace{-5pt}
\emph{Scope.} \textsc{Unmerge} is a training-free edit for concept-level forget requests, i.e.\ requests that occupy a separable activation subspace, such as whole classes and semantic groups (\S\ref{sec:instance}). On a coherent instance subset it lowers the forget and retained halves of a class alike (Tab.~\ref{tab:coherent-breakdown}). Random instance subsets have none---they pose a per-sample de-memorization problem rather than a concept-removal one~\citep{triantafillou2026untraining} (App.~\ref{sec:instance-supp})---and no subspace edit can move their forget accuracy without moving retain; gradient-on-retain methods can, at $4$--$23\times$ the cost and at the price of a membership signature on the forget set. The Phase-1 separability ratio flags the regime beforehand but needs a same-size random-subset control for calibration (App.~\ref{sec:instance-supp}). \emph{Requirements.} The method is optimization-time data-free, not data-free: it needs forward passes over $\setF\cup\setR$, \textsc{CAV-Class} needs forget labels (App.~\ref{sec:requirements}), and the $\Wpt$ regime needs the \emph{same} pretraining checkpoint---substituting a different one is catastrophic, while $\Winit{=}0$ is a safe fallback (App.~\ref{sec:winit}). \emph{Fidelity and solver.} Unmerging does no weight-space harm and has the smallest forget-set divergence to $\Wrt$ among the methods of Tab.~\ref{tab:tauval}, but does not reconstruct $\tauRstar$ in weight space beyond the un-edited model (\S\ref{sec:tauval}). At class level a logit mask on the un-edited model already matches retraining on ToW and the MIA gap, and a linear probe on frozen features still recovers the forget classes from every approximate method we tested, ours included; the feature-entanglement axis separates \textsc{Unmerge} from such a mask on $\Wzero$ and on TinyImageNet, but not on CIFAR-100/$\Wpt$, where the un-edited features already sit near retraining. We therefore claim fidelity to retraining in outputs and feature distribution, not that class information becomes undecodable. The reported cells use an early-stopped Adam iterate rather than the exact minimizer, so their $\lambda$ is nominal and the PCA variants depend on the step budget, which we fix throughout (App.~\ref{sec:closedform}). \emph{Breadth and future work.} The instance-level, task-vector-validation and seed-variance studies use one CIFAR-100 superclass, and the cost model is validated on ResNet-50 and ViT-S/16; at language-model scale we only have the preliminary study of App.~\ref{sec:llm}. The framework itself is general: the CAV construction accepts arbitrary positive/negative pools, so concept labels (toxicity, bias, copyrighted style) lift \textsc{Unmerge} to language and vision-language models without changing the objective, and the layerwise structure fits LoRA-merged and mixture-of-experts checkpoints; explicit solvers (App.~\ref{sec:closedform}), continual forget requests, and adaptive per-layer rank and skip selection are further directions.

\end{document}